\documentclass[twoside,11pt]{article}

\usepackage{fullpage}

\usepackage{natbib}
\usepackage{amssymb}

\usepackage{times}
\usepackage{mathtools}
\usepackage{amsmath}
\usepackage{eucal}
\usepackage{mathrsfs}
\usepackage[usenames,dvipsnames]{color}

\usepackage{algorithm}
\usepackage{algorithmic}

\usepackage{hyperref}

\newtheorem{theorem}{Theorem}
\newtheorem{definition}[theorem]{\bf{Definition}}
\newtheorem{proposition}[theorem]{\bf{Proposition}}
\newtheorem{lemma}[theorem]{\bf{Lemma}}

\newtheorem{procedure}[theorem]{\bf{Procedure}}

\newcommand{\Prob}{\mathbb{P}}

\newcommand{\vect}[1]{\mathbf{#1}}
\newcommand{\set}[1]{\mathbb{#1}}
\newcommand{\sett}[1]{\mathcal{#1}}
\newcommand{\matr}[1]{\mathbf{#1}}
\newcommand{\vectsymb}[1]{\boldsymbol{#1}}

\newcommand{\tup}[1]{\mathscr{#1}}

\newcommand{\Y}{{\mathcal{Y}}}

\newcommand{\Score}{\vectsymb{\theta}}
\newcommand{\score}{{\theta}}
\newcommand{\DP}[2]{{#1}^{\top}{#2}}
\newcommand{\MARG}{\mathrm{MARG}}
\newcommand{\LOCAL}{\mathrm{LOCAL}}

\DeclareMathOperator*{\argmax}{argmax}

\DeclareMathOperator{\conv}{conv}

\newcommand{\nascomment}[1]{\textcolor{blue}{\bf [Noah: #1]}}

\begin{document}

\title{\bf Alternating Directions Dual Decomposition}

\author{Andr\'{e} F. T. Martins$^\ast$$^\dagger$\\
       $^\ast$Priberam Labs\\
       Lisboa, Portugal\\
       \url{andre.t.martins@gmail.com}
       \and
       M\'{a}rio A. T. Figueiredo$^\dagger$ \\
       $^\dagger$Instituto de Telecomunica\c{c}\~{o}es\\
       Instituto Superior T\'{e}cnico\\
       Lisboa, Portugal\\
       \url{mtf@lx.it.pt}
       \and
       Pedro M. Q. Aguiar$^\ddagger$\\
       $^\ddagger$Instituto de Sistemas e Rob\'{o}tica\\
       Instituto Superior T\'{e}cnico\\
       Lisboa, Portugal\\
       \url{aguiar@isr.ist.utl.pt}
       \and
       Noah A. Smith$^\natural$ \qquad
       Eric P. Xing$^\natural$ \\
       $^\natural$School of Computer Science\\
       Carnegie Mellon University\\
       Pittsburgh, PA 15213, USA\\
       \url{{nasmith, epxing}@cs.cmu.edu}}      
       
\date{}

\maketitle

\begin{abstract}
We propose AD$^3$, a new algorithm for approximate {\it maximum a
  posteriori} (MAP)
inference on factor graphs based on the alternating
directions method of multipliers. 
Like dual decomposition algorithms,
AD$^3$  
uses worker nodes to iteratively solve local subproblems and a
controller node to combine these local solutions into a global update.
The key characteristic of AD$^3$ is that 
each local subproblem has a quadratic regularizer, leading to a faster
consensus than subgradient-based dual decomposition, 
both theoretically and in practice. 
We provide closed-form solutions for these AD$^3$ subproblems 
for binary pairwise factors and factors imposing first-order logic 
constraints. 
For arbitrary factors (large or combinatorial), 
we introduce an active set method which requires
only an oracle for computing a local MAP configuration, 
making AD$^3$ applicable to a wide range of problems.
Experiments on synthetic and real-world problems show that AD$^3$
compares favorably with the state-of-the-art.
\end{abstract}


\section{Introduction}

Graphical models enable compact
representations of probability distributions, being widely used in
computer vision, natural language processing (NLP),
and computational biology \citep{Pearl1988,Lauritzen1996,Koller2009}.
When using these models, a central problem is that of inferring the
most probable (a.k.a. \emph{maximum a posteriori} -- MAP) configuration.
Unfortunately, exact MAP inference is an intractable problem 
for many graphical models of interest in 
applications, such as those involving non-local features and/or
structural constraints.
This fact has motivated a significant research effort on \emph{approximate} MAP inference.


A class of methods that proved effective for approximate inference 
is based on linear programming relaxations of the MAP problem (LP-MAP) \citep{Schlesinger1976}.
Several message-passing algorithms have been proposed to address the resulting LP problems,
taking advantage of the underlying graph structure \citep{Wainwright2005b,Kolmogorov2006,Werner2007,Globerson2008,Ravikumar2010}.
In the same line, \citet{Komodakis2007} proposed a method
based on the classical \emph{dual decomposition} \citep{Dantzig1960,Everett1963,Shor1985},
which breaks the original problem into a set of smaller subproblems; this set of subproblems,
together with the constraints that they should agree on the shared variables, yields a
constrained optimization problem, the Lagrange dual of which is solved
using a \emph{projected subgradient} method. Initially proposed for computer vision
\citep{Komodakis2007,Komodakis2009}, this technique has also been shown effective in
several NLP tasks that involve combinatorial constraints,
such as parsing and machine translation \citep{Koo2010EMNLP,Rush2010,Auli2011,Rush2011,Chang2011,DeNero2011}.
The major drawback is that the subgradient algorithm is very slow to
achieve consensus in problems with a large number of overlapping components,
requiring $O(1/\epsilon^2)$ iterations for an $\epsilon$-accurate solution.
\citet{Jojic2010} addressed this drawback by introducing an accelerated method that
enjoys a faster convergence rate $O(1/\epsilon)$; however, that method is sensitive
to previous knowledge of $\epsilon$ and the prescription of a temperature parameter,
which may yield numerical instabilities.


In this paper, we ally the simplicity of \emph{dual decomposition} (DD) with
the effectiveness of \emph{augmented Lagrangian methods}, which have a
long-standing history in optimization \citep{Hestenes1969,Powell1969}.
In particular, we employ the \emph{alternating directions method of multipliers}\footnote{Also
known as ``method of alternating directions'' and closely related to 
 ``Douglas-Rachford splitting."} (ADMM), a method introduced in the optimization
community in the 1970s \citep{Glowinski75,Gabay1976} which
has seen a recent surge of interest in machine learning and signal
processing \citep{Afonso2010,Mota2010,Goldfarb2010,Boyd2011}.
The result is a new LP-MAP inference algorithm called AD$^3$ (\emph{alternating directions dual decomposition}),
which enjoys $O(1/\epsilon)$ convergence rate. We show that AD$^3$ is able to solve the LP-MAP
problem more efficiently than other methods. Like the projected subgradient algorithm of \citet{Komodakis2007},
AD$^3$ is  an iterative ``consensus algorithm,'' alternating between a \emph{broadcast} operation,
where subproblems are distributed across local workers, and a \emph{gather} operation, where the local
solutions are assembled by a controller. The key difference is that AD$^3$ also broadcasts
\emph{the current global solution}, allowing the local workers to regularize their subproblems toward
that solution. This has the effect of speeding up consensus, without sacrificing the modularity of DD
algorithms.

Our main contributions are:
\begin{itemize}
\item We derive AD$^3$ and establish its convergence properties,
blending classical and newer results about ADMM
\citep{Glowinski1989,Eckstein1992,Wang2012ICML}.
We show that the algorithm has the same form as
the DD method of \citet{Komodakis2007}, with the local MAP
subproblems replaced by quadratic programs.
\item We show that these AD$^3$ subproblems can be solved
exactly and efficiently in many cases of interest, including
Ising models and a wide range of hard factors representing arbitrary
constraints in first-order logic.
Up to a logarithmic term, the asymptotic cost is the same as
that of passing messages or doing local MAP inference.
\item For factors lacking a closed-form solution of the AD$^3$
subproblems, we introduce a new \emph{active set method}.
Remarkably, our method requires
only a black box that returns local MAP configurations for each factor 
(the same requirement of the projected subgradient algorithm).
This paves the way for using AD$^3$ with large dense or structured
factors, based on off-the-shelf combinatorial algorithms
(\emph{e.g.}, Viterbi, Chu-Liu-Edmonds, Ford-Fulkerson).
\item We show that AD$^3$ can be wrapped into a branch-and-bound
procedure to retrieve the \emph{exact} MAP configuration.
\end{itemize}
%

AD$^3$ was originally introduced by \citet{Martins2010bNIPS,Martins2011ICML} (then
called  DD-ADMM). In addition to a considerably more detailed presentation, this paper
contains  contributions that substantially extend
that preliminary work in several directions: the $O(1/\epsilon)$ rate of
convergence, the active set method for general factors, and the
branch-and-bound procedure for exact MAP inference. It also reports
a wider set of experiments and the release of open-source code (available at
\url{http://www.ark.cs.cmu.edu/AD3}), which may be useful to other researchers in the field.

This paper is organized as follows. We start by providing background material:
MAP inference in graphical models and its LP-MAP relaxation (Section~\ref{sec:background});
the projected subgradient algorithm for the DD of \citet{Komodakis2007} (Section~\ref{sec:dualdecomposition}).
In Section~\ref{sec:ad3},  we derive AD$^3$ and analyze its convergence.
The AD$^3$ local subproblems are addressed in Section~\ref{sec:ad3_subproblems},
where closed-form solutions are derived for Ising models and
several structural constraint factors. In Section~\ref{sec:ad3_generalfactors},
we introduce an active set method to solve the AD$^3$ subproblems for
arbitrary factors. Experiments with synthetic models, as well as in protein design
and dependency parsing (Section~\ref{sec:experiments}) testify for the success of
our approach. Finally, related work in discussed in Section~\ref{sec:relatedwork},
and Section~\ref{sec:conclusions} concludes the paper.

\section{Background}\label{sec:background}
\subsection{Factor Graphs}

Let $\vectsymb{Y} := (Y_1,\ldots,Y_M)$ be a tuple of random variables,
where each $Y_i$ takes values in a finite set $\Y_i$ and $\vectsymb{Y}$
takes values in the corresponding Cartesian product set $\Y := \prod_{i=1}^M \Y_i$.
Graphical models (\emph{e.g.}, Bayesian networks, Markov random fields)
are structural representations of joint probability distributions, which
conveniently model  statistical (in)dependencies among the variables.
In this paper, we focus on \emph{factor graphs},
introduced by \citet{Tanner1981} and \citet{Kschischang2001}.

\begin{definition}[Factor graph.]
A factor graph is a bipartite graph $\tup{G} := (\sett{V}, \sett{F}, \sett{E})$,
comprised of:
\begin{itemize}
\item a set of \emph{variable nodes}
$\sett{V} := \{1,\ldots,M\}$, corresponding to the variables $(Y_1,\ldots,Y_M)$;
\item a set of \emph{factor nodes} $\sett{F}$ (disjoint from $\sett{V}$);
\item a set of edges $\sett{E} \subseteq \sett{V} \times \sett{F}$ linking variable nodes to factor nodes.
\end{itemize}
\end{definition}
For notational convenience, we use Latin letters ($i,j,...$) and Greek letters ($\alpha,\beta,...$)
to refer to variable and factor nodes, respectively. We denote by
$\sett{N}(\cdot )$ the \emph{neighborhood set} of its node argument, whose cardinality is
called the \emph{degree} of the node. Formally, $\sett{N}(i) := \{\alpha \in \sett{F} \,\,|\,\, (i,\alpha) \in \sett{E}\}$,
for variable nodes, and $\sett{N}(\alpha) := \{i \in \sett{V} \,\,|\,\, (i,\alpha) \in \sett{E}\}$ for factor nodes.
We use the short notation $\vectsymb{Y}_{\!\!\alpha}$ to denote the tuple of variables linked to factor $\alpha$,
\emph{i.e.}, $\vectsymb{Y}_{\!\!\alpha} := (Y_i)_{i \in\sett{N}(\alpha)}$, and lowercase to denote instantiations
of random variables, \emph{e.g.}, $Y_i = y_i$ and $\vectsymb{Y}_{\!\!\alpha} = \vectsymb{y}_{\alpha}$.

The joint probability distribution of $(Y_1,\ldots,Y_M)$ is said to factor according to the factor graph
$\tup{G}=(\sett{V}, \sett{F}, \sett{E})$ if it can be written as
\begin{equation}\label{eq:probfactgraph01}
\Prob(Y_1=y_1,\ldots,Y_M=y_M) \,\,\propto\,\, \exp\left(
\sum_{i \in \sett{V}} \score_i(y_i) + \sum_{\alpha \in \sett{F}} \score_{\alpha}(\vectsymb{y}_{\alpha}) \right),
\end{equation}
where $\score_i(\cdot)$ and $\score_{\alpha}(\cdot)$ are called, respectively,
the \emph{unary} and \emph{higher-order} log-potential functions.
To accommodate hard constraints, we allow these functions to
take values in $\set{R} \cup \{-\infty\}$. Fig.~\ref{fig:constrainedgraphmodels} shows examples
of factor graphs with hard constraint factors (to be studied in detail in
Section~\ref{sec:ddadmm_hardconstraintfactors}).
For notational convenience, we represent potential functions as
vectors, $\Score_i := (\score_i(y_i))_{y_i \in \Y_i}$
and $\Score_{\alpha} := (\score_{\alpha}(\vectsymb{y}_{\alpha}))_{\vectsymb{y}_{\alpha} \in \Y_{\alpha}}$,
where $\Y_{\alpha} = \prod_{i\in \sett{N}(\alpha) } \Y_{i}$.
We denote by $\Score$ the vector of dimension
$\sum_{i\in \sett{V}} |\Y_i| + \sum_{\alpha\in \sett{F}} |\Y_{\alpha}|$
obtained by stacking all these vectors.

\begin{figure}
\begin{center}
\includegraphics[width=0.9\textwidth]{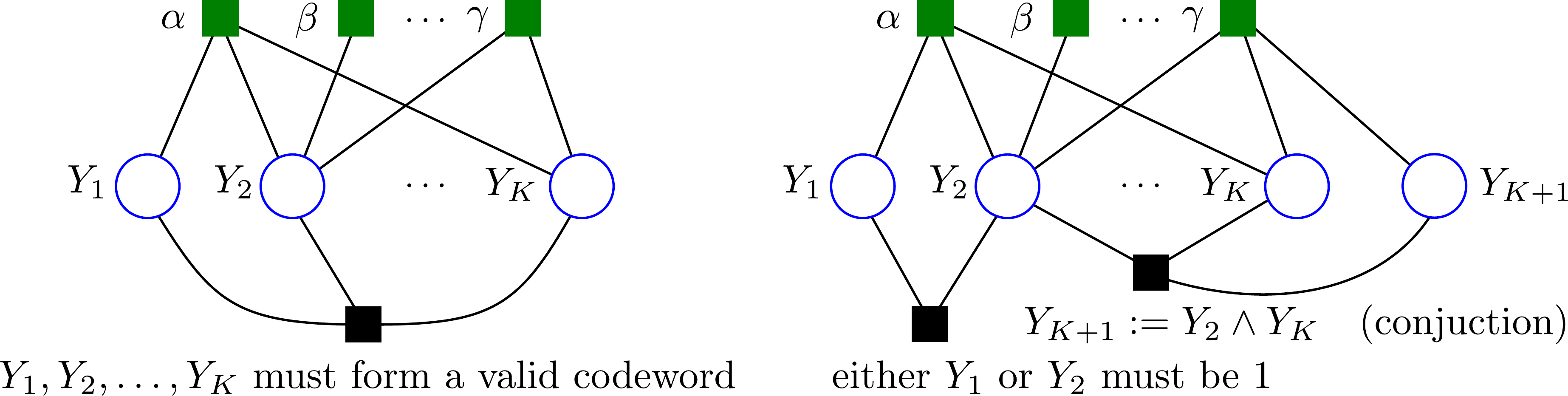}
\caption{Constrained factor graphs, with soft factors shown as \emph{green} squares
above the variable nodes (circles) and hard constraint factors as \emph{black} squares
below the variable nodes. Left: a global factor that constrains the set of admissible
outputs to a given codebook. Right: examples of declarative constraints; one of them is a factor
connecting existing variables to an extra variable, allows scores depending on a logical functions
of the former.}
\label{fig:constrainedgraphmodels}
\end{center}
\end{figure}


\subsection{MAP Inference}

Given a factor graph $\tup{G}$, along with all the log-potential functions---which fully specify the joint
distribution of $Y_1,\ldots,Y_M$---we are interested in finding an assignment with maximal probability
(the so-called MAP assignment/configuration):
\begin{eqnarray}\label{eq:inferencetasks_map}
{\widehat{\vectsymb{y}}} &\in& \arg\max_{y_1,\ldots,y_M} \Prob(Y_1=y_1,\ldots,Y_M=y_M)\nonumber\\
&=& \arg\max_{y_1,\ldots,y_M} \sum_{\alpha \in \sett{F}} \score_{\alpha}(\vectsymb{y}_{\alpha}) + \sum_{i \in \sett{V}} \score_{i}(y_i).
\end{eqnarray}
This combinatorial problem, which is NP-hard in general \citep{Koller2009},
can be transformed into a linear program
by using the concept of \emph{marginal polytope} \citep{Wainwright2008},
as next described.
Let $\Delta^K$ denote the $K$-dimensional probability simplex,
\begin{equation}
\Delta^K := \{ \vectsymb{u} \in \set{R}^K \,\,|\,\,\vectsymb{u} \ge \vect{0}, \,\, \DP{\vect{1}}{\vectsymb{u}} = 1\} = \conv \{\vectsymb{e}_1,\ldots, \vectsymb{e}_K\},
\end{equation}
where $\vectsymb{e}_j$ has all entries equal to zero except for the $j$th one, which is equal to one.
We introduce ``local'' probability
distributions over the variables and factors, $\vectsymb{p}_i \in \Delta^{|\Y_i|}$
and $\vectsymb{q}_{\alpha} \in \Delta^{|\Y_{\alpha}|}$.
We stack these distributions into vectors $\vectsymb{p}$ and $\vectsymb{q}$,
with dimensions
$P := \sum_{i\in \sett{V}} |\Y_i|$ and
$Q := \sum_{\alpha\in \sett{F}} |\Y_{\alpha}|$,
respectively, and denote
by $(\vectsymb{p},\vectsymb{q}) \in \set{R}^{P+Q}$
their concatenation.
We say that $(\vectsymb{p},\vectsymb{q})$
is \emph{globally consistent} if 
the local distributions 
$\{\vectsymb{p}_i\}$ and $\{\vectsymb{q}_{\alpha}\}$
are the 
marginals 
of some global joint distribution.
The set of all globally consistent vectors $(\vectsymb{p},\vectsymb{q})$
is called the \emph{marginal polytope} of $\tup{G}$,
denoted as $\MARG(\tup{G})$.
There is a one-to-one mapping between the vertices
of $\MARG(\tup{G})$ and the
set of possible configurations $\sett{Y}$: each
$\vectsymb{y} \in \sett{Y}$ corresponds
to a vertex $(\vectsymb{p},\vectsymb{q})$ of $\MARG(\tup{G})$
consisting of ``degenerate'' distributions
$\vectsymb{p}_i = \vectsymb{e}_{y_i}$
and $\vectsymb{q}_{\alpha} = \vectsymb{e}_{\vectsymb{y}_{\alpha}}$,
for each $i \in \sett{V}$ and $\alpha \in \sett{F}$.
%
The MAP problem in \eqref{eq:inferencetasks_map} is then
equivalent to the following linear program (LP):
\begin{equation}\label{eq:map}
\boxed{%
\begin{array}{lll}
\textbf{MAP:} \qquad & \mathrm{maximize} & \displaystyle \sum_{\alpha \in \sett{F}} \DP{\Score_{\alpha}}{\vectsymb{q}_{\alpha}} + \sum_{i \in \sett{V}} \DP{\Score_{i}}{\vectsymb{p}_{i}}\\
& \text{with respect to} & (\vectsymb{p},\vectsymb{q}) \in \MARG(\tup{G}).
\end{array}}
\end{equation}
The equivalence between \eqref{eq:inferencetasks_map}
and \eqref{eq:map} stems from the fact that any linear program on
a bounded convex constraint polytope  attains a solution at a vertex
of that polytope (see, \emph{e.g.}, \citet{Rockafellar1970}).
Unfortunately, the linear program \eqref{eq:map} is not easier to solve than
\eqref{eq:inferencetasks_map}: for a graph $\tup{G}$ with cycles (which induce
\emph{global consistency constraints} that are hard to specify concisely),
the number of linear inequalities that characterize $\MARG(\tup{G})$
may grow superpolynomially with the size of $\tup{G}$. As a consequence,
approximations of $\MARG(\tup{G})$ have been actively investigated in recent years.

\subsection{LP-MAP Inference}\label{sec:lpmap}

Tractable approximations of $\MARG(\tup{G})$ can be built
by using weaker constraints that all realizable marginals must satisfy
to ensure \emph{local consistency}:
\begin{enumerate}
\item {\bf Non-negativity:} all local marginals $p_i(y_i)$ and $q_{\alpha}(\vectsymb{y}_{\alpha})$ must be non-negative.
\item {\bf Normalization:} local marginals must be properly normalized, \emph{i.e.},
$\sum_{y_i \in \Y_i} p_i(y_i) = 1$, for all  $i \in \sett{V}$,
and $\sum_{\vectsymb{y}_{\alpha} \in \Y_{\alpha}} q_{\alpha}(\vectsymb{y}_{\alpha}) = 1$, for all $\alpha \in \sett{F}$.
\item {\bf Marginalization:} a variable participating in a factor must have a marginal which is consistent
with the factor marginals, \emph{i.e.},  $p_i(y_i) = \sum_{\vectsymb{y}_{\alpha} \sim y_i} q_{\alpha}(\vectsymb{y}_{\alpha})$,
for all $(i,\alpha) \in \sett{E}$ and $y_i \in \Y_i$.%
\footnote{We use the short notation $\vectsymb{y}_{\alpha} \sim y_i$ to denote that
the assignments $\vectsymb{y}_{\alpha}$ and $y_i$ are consistent. A sum or
maximization with a subscript $\vectsymb{y}_{\alpha} \sim y_i$ means that the sum/maximization is over all the
$\vectsymb{y}_{\alpha}$ which are consistent with $y_i$.} %
\end{enumerate}
Note that some of these constraints are redundant: the non-negativity of
$q_{\alpha}(\vectsymb{y}_{\alpha})$ and the marginalization constraint
imply the non-negativity of $p_i(y_i)$; the normalization of
$\vectsymb{q}_{\alpha}$ and the marginalization constraints imply
$\sum_{y_i \in \Y_i} p_i(y_i) = 1$. For convenience, we express those constraints
in vector notation.
For each $(i,\alpha) \in \sett{E}$, we define a (consistency) matrix
$\matr{M}_{i\alpha} \in \set{R}^{|\sett{Y}_i| \times |\sett{Y}_{\alpha}|}$
as follows:
\begin{equation}
\matr{M}_{i\alpha}(y_i,\vectsymb{y}_{\alpha}) = \left\{
\begin{array}{ll}
1, & \text{if $\vectsymb{y}_{\alpha} \sim y_i$}\\
0, & \text{otherwise.}
\end{array}
\right.
\end{equation}
The local marginalization constraints can be expressed as $\matr{M}_{i\alpha}\vectsymb{q}_{\alpha} = \vectsymb{p}_i.$
Combining all the local consistency constraints, and dropping redundant ones, yields
the so-called \emph{local polytope}:
\begin{equation}\label{eq:localpolytope}
\LOCAL(\tup{G}) = \left\{
(\vectsymb{p},\vectsymb{q}) \in \set{R}^{P+Q} \,\, \Bigg| \,\,
\begin{array}{ll}
\vectsymb{q}_{\alpha} \in \Delta^{|\Y_{\alpha}|},
& \forall \alpha \in \sett{F}\\
\matr{M}_{i\alpha}\vectsymb{q}_{\alpha} = \vectsymb{p}_i, &
\forall (i,\alpha) \in \sett{E} \\
\end{array}
\right\}.
\end{equation}
The number of constraints that define $\LOCAL(\tup{G})$ grows
\emph{linearly} with $P+Q$,
rather than superpolynomially. 
The elements of $\LOCAL(\tup{G})$ are called \emph{pseudo-marginals}.
Since any true marginal vector must satisfy the constraints above, $\LOCAL(\tup{G})$ is an \emph{outer approximation}, \emph{i.e.},
\begin{equation}
\MARG(\tup{G}) \subseteq \LOCAL(\tup{G}),
\end{equation}
as illustrated in Fig.~\ref{fig:localpolytope}.
\begin{figure}[t]
\begin{center}
\includegraphics[scale=0.4]{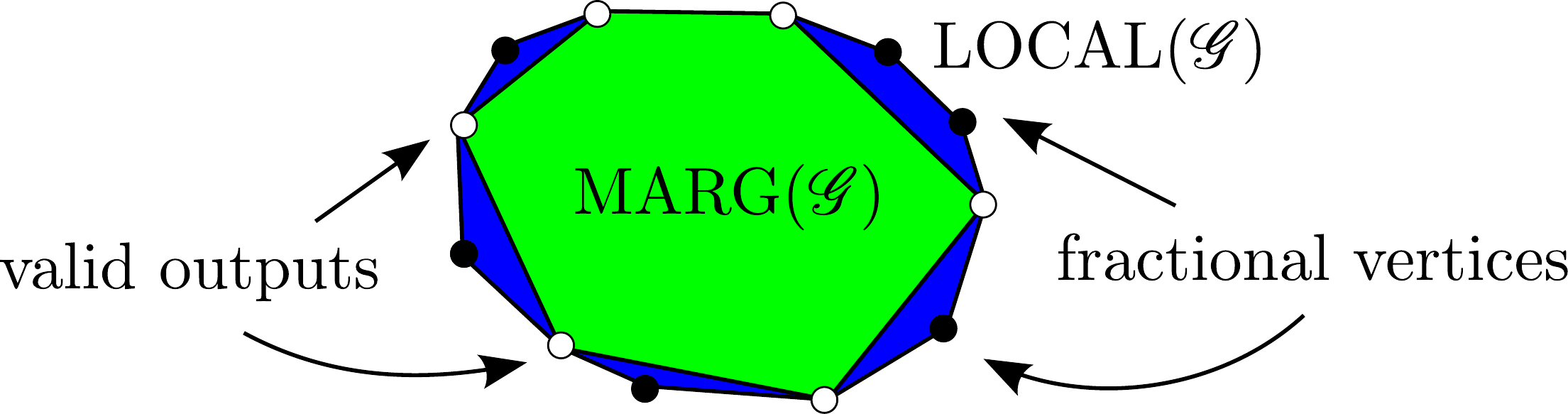}%
\caption{Marginal polytope (in green) and its outer aproximation,
the local polytope (in blue).
Each element of the marginal polytope
corresponds to a joint distribution of $Y_1,\ldots,Y_M$,
and each vertex corresponds to a configuration $\vectsymb{y} \in \sett{Y}$,
having coordinates in $\{0,1\}$.
The local polytope may have additional fractional vertices,
with coordinates in $[0,1]$.}
\label{fig:localpolytope}
\end{center}
\end{figure}
This approximation is tight for tree-structured graphs
and other special cases, but not in general (even for  small
graphs with cycles) \citep{Wainwright2008}. The \emph{LP-MAP problem}
results from replacing $\MARG(\tup{G})$ by $\LOCAL(\tup{G})$ in \eqref{eq:map} \citep{Schlesinger1976}:
\begin{equation}\label{eq:lpmap}
\boxed{%
\begin{array}{lll}
\textbf{LP-MAP:} \qquad & \mathrm{maximize} & \displaystyle \sum_{\alpha \in \sett{F}} \DP{\Score_{\alpha}}{\vectsymb{q}_{\alpha}} + \sum_{i \in \sett{V}} \DP{\Score_{i}}{\vectsymb{p}_{i}}\\
& \text{with respect to} & (\vectsymb{p},\vectsymb{q}) \in \LOCAL(\tup{G}).
\end{array}}
\end{equation}
It can be shown that the points of $\LOCAL(\tup{G})$ that are integer are
the vertices of $\MARG(\tup{G})$; therefore, the solution of LP-MAP provides
an upper bound on the optimal objective of MAP.

\subsection{LP-MAP Inference Algorithms}
While any off-the-shelf LP solver can be used for solving \eqref{eq:lpmap},
specialized algorithms have been designed to exploit the graph structure,
achieving superior performance on several benchmarks \citep{Yanover2006}.
Most of these specialized algorithms belong to two classes:
block (dual) coordinate descent, which take the form of \emph{message-passing} algorithms, and projected
subgradient algorithms, based on \emph{dual decomposition}.

Block coordinate descent methods address the dual of \eqref{eq:lpmap}
by alternately optimizing over blocks of coordinates. Different choices of
blocks lead to different algorithms: \emph{max-sum diffusion}
\citep{Kovalevsky1975,Werner2007}; \emph{max-product sequential tree-reweighted belief propagation} (TRW-S)
\citep{Wainwright2005b,Kolmogorov2006};
\emph{max-product linear programming} (MPLP) \citep{Globerson2008}.
These algorithms work by passing messages (that require computing
\emph{max-marginals}) between factors and variables.
Under certain conditions, if the relaxation is tight, one may obtain optimality certificates.
If the relaxation is not tight, it is sometimes possible to reduce the problem size
or use a cutting-plane method to move toward the exact MAP \citep{Sontag2008}.
A disadvantage of block coordinate descent algorithms is that they
may stop at suboptimal solutions, since the objective is non-smooth
(\citealt[Section~6.3.4]{Bertsekas1999}).

Projected subgradient based on the dual decomposition is a classical
technique \citep{Dantzig1960,Everett1963,Shor1985},  which was first proposed
in the context of graphical models by \citet{Komodakis2007} and \citet{Johnson2007}.
Due to its strong relationship with the approach pursued in this paper, that
method is described in detail in the next section.

\section{Dual Decomposition with the Projected Subgradient Algorithm}\label{sec:dualdecomposition}




The projected subgradient method for dual decomposition can be seen as an
iterative ``consensus algorithm,'' alternating between a \emph{broadcast} operation,
where local subproblems are distributed across \emph{worker nodes},
and a \emph{gather} operation, where the local solutions are assembled by a
\emph{master node} to update the current global solution.
At each round, the worker nodes perform local MAP inference independently;
hence, the algorithm is completely modular and trivially parallelizable.

For each edge $(i,\alpha)\in \sett{E}$, define a potential function
$\Score_{i\alpha} := (\score_{i\alpha}(y_i))_{y_i \in \Y_i}$ that satisfies $\sum_{\alpha \in \sett{N}(i)}\Score_{i\alpha} = \Score_{i}$;
a trivial choice is  $\Score_{i\alpha} = |\sett{N}(i)|^{-1} \Score_{i}$.
The objective \eqref{eq:lpmap} can be rewritten as
\begin{eqnarray}
\sum_{\alpha \in \sett{F}}
\DP{\Score_{\alpha}}{\vectsymb{q}_{\alpha}} +
\sum_{i \in \sett{V}} \DP{\Score_{i}}{\vectsymb{p}_{i}}
&=&
\sum_{\alpha \in \sett{F}}
\DP{\left({\Score}_{\alpha} + \sum_{i\in\sett{N}(\alpha)} \matr{M}_{i\alpha}^{\top}\Score_{i\alpha}\right)}{\vectsymb{q}_{\alpha}},
\end{eqnarray}
because $\vectsymb{p}_i = \matr{M}_{i\alpha}\vectsymb{q}_{\alpha}$.
The LP-MAP problem \eqref{eq:lpmap} is then equivalent to the following
\emph{primal formulation}, which we call LP-MAP-P:
\begin{equation}
\boxed{%
\begin{aligned}\label{eq:lpmap-reformulation}
\textbf{LP-MAP-P:}\qquad \mathrm{maximize} \quad
&
\sum_{\alpha \in \sett{F}}
\DP{\left({\Score}_{\alpha} + \sum_{i\in\sett{N}(\alpha)} \matr{M}_{i\alpha}^{\top}\Score_{i\alpha}\right)}{\vectsymb{q}_{\alpha}}
\\
\text{with respect to}  \quad
&
\vectsymb{q}_\alpha \in  \Delta^{|\Y_{\alpha}|}, \forall \alpha \in
\sett{F}, \\ &  \vectsymb{p} \in \set{R}^P,  \\
\text{subject to}  \quad &
 \matr{M}_{i\alpha}\vectsymb{q}_{\alpha} = \vectsymb{p}_i,\,\, \forall (i,\alpha) \in \sett{E}.
\end{aligned}}
\end{equation}

Note that, although the $\vectsymb{p}$-variables do not appear in the
objective of \eqref{eq:lpmap-reformulation},
they play a fundamental role through the constraints in the last line,
which are necessary to
ensure that the marginals encoded in the $\vectsymb{q}$-variables
are consistent on their overlaps.
Indeed, it is this set of constraints that complicate the optimization problem,
which would otherwise be separable into independent subproblems, one per factor.
Introducing a Lagrange multiplier ${\lambda}_{i\alpha}(y_i)$ for each of these equality constraints leads
to the Lagrangian function
\begin{eqnarray}\label{eq:ddsubgrad_lagrangian}
L(\vectsymb{q},\vectsymb{p},\vectsymb{\lambda}) &=&
\sum_{\alpha \in \sett{F}}
\DP{\left({\Score}_{\alpha} + \sum_{i\in\sett{N}(\alpha)}\matr{M}_{i\alpha}^{\top}(\Score_{i\alpha} + \vectsymb{\lambda}_{i\alpha})\right)}{\vectsymb{q}_{\alpha}}
- \sum_{(i,\alpha) \in \sett{E}} \DP{\vectsymb{\lambda}_{i\alpha}}{\vectsymb{p}_i},
\end{eqnarray}
the maximization of which w.r.t. $\vectsymb{q}$ and $\vectsymb{p}$ will yield the (Lagrangian) dual objective.
Since the $\vectsymb{p}$-variables are unconstrained, we have
\begin{eqnarray}
\max_{\vectsymb{q},\vectsymb{p}} L(\vectsymb{q},\vectsymb{p},\vectsymb{\lambda}) =
\left\{
\begin{array}{ll}
\displaystyle  g(\vectsymb{\lambda}) & \text{if $\vectsymb{\lambda} \in \Lambda$,}\\
+\infty & \text{otherwise,}
\end{array}
\right.
\end{eqnarray}
where $g(\vectsymb{\lambda}) = \sum_{\alpha \in \sett{F}} g_{\alpha}(\vectsymb{\lambda})$ is the \emph{dual objective function},
\begin{equation}\label{eq:dualdecomplambdaset}
\Lambda := \left\{\vectsymb{\lambda} \,\, \bigg| \,\, \sum_{\alpha \in \sett{N}(i)} \vectsymb{\lambda}_{i\alpha} = \vect{0},
\,\, \forall i \in \sett{V} \right\},
\end{equation}
and each $g_{\alpha}(\vectsymb{\lambda})$ is the solution of a local problem
(called the \emph{$\alpha$-subproblem}):
\begin{eqnarray}
g_{\alpha}(\vectsymb{\lambda}) &:=&
\max_{\vectsymb{q}_{\alpha} \in \Delta^{|\Y_{\alpha}|}}
\DP{\left({\Score}_{\alpha} + \sum_{i\in\sett{N}(\alpha)}\matr{M}_{i\alpha}^{\top}(\Score_{i\alpha} + \vectsymb{\lambda}_{i\alpha})\right)}{\vectsymb{q}_{\alpha}}
\label{eq:dualdecompslave01}\\
&=&
\max_{\vectsymb{y}_{\alpha}\in \Y_{\alpha}}
\left(\score_{\alpha}(\vectsymb{y}_{\alpha}) +
\sum_{i \in \sett{N}(\alpha)} (\score_{i\alpha}(y_i) + \lambda_{i\alpha}(y_i))\right);
\label{eq:dualdecompslave02}
\end{eqnarray}
the last equality is justified by the fact that maximizing a linear objective over the
probability simplex gives the largest component of the score vector.
Finally, the dual problem is
\begin{equation}\label{eq:dualdecomp}
\boxed{%
\begin{array}{lll}
\textbf{LP-MAP-D:} \qquad & \mathrm{minimize} & \displaystyle g(\vectsymb{\lambda})\\
& \text{with respect to} & \vectsymb{\lambda} \in \Lambda.
\end{array}}
\end{equation}

Problem \eqref{eq:dualdecomp} is often referred to as the \emph{master} or \emph{controller},
and each $\alpha$-subproblem \eqref{eq:dualdecompslave01}--\eqref{eq:dualdecompslave02} as a \emph{slave} or \emph{worker}.
Note that each of these $\alpha$-subproblems is itself a local MAP problem,
involving only factor $\alpha$.
As a consequence, the solution $\widehat{\vectsymb{q}}_{\alpha}$
of \eqref{eq:dualdecompslave01}
is an indicator vector corresponding to a particular configuration
$\widehat{\vectsymb{y}}_{\alpha}$ (the solution of \eqref{eq:dualdecompslave02} ), that is, $\widehat{\vectsymb{q}}_{\alpha} = \vectsymb{e}_{\widehat{\vectsymb{y}}_{\alpha}}$.

The dual problem \eqref{eq:dualdecomp} can be solved with a \emph{projected subgradient algorithm}.%
\footnote{The same algorithm can be derived by applying Lagrangian relaxation to the original MAP.
A slightly different formulation is presented by \citet{Sontag2011OPT} which yields
a subgradient algorithm with no projection.} %
By Danskin's rule \citep[p.~717]{Bertsekas1999}, a subgradient of $g_{\alpha}$
is readily given by
\begin{equation}
\frac{\partial g_{\alpha}(\vectsymb{\lambda})}{\partial \vectsymb{\lambda}_{i\alpha}} = \matr{M}_{i\alpha}\widehat{\vectsymb{q}}_{\alpha} := \widehat{\vectsymb{q}}_{i\alpha},
\quad \forall (i,\alpha) \in \sett{E};
\end{equation}
and the projection onto $\Lambda$ amounts to a centering operation.
The $\alpha$-subproblems \eqref{eq:dualdecompslave01}--\eqref{eq:dualdecompslave02}
can be handled in parallel and then have their solutions gathered for
computing this projection and update the Lagrange variables.
Putting these pieces together yields Algorithm~\ref{alg:ddsubgrad},
which assumes a black-box procedure {\sc ComputeMap} that returns
a local MAP assignment (as an indicator vector), given log-potentials as input.
At each iteration, the algorithm broadcasts the current Lagrange multipliers
to all the factors.
Each factor adjusts its internal unary log-potentials (line~\ref{algline:ddsubgrad-unary}) and
invokes the {\sc ComputeMap} procedure (line~\ref{algline:ddsubgrad-localmap}).%
\footnote{Note that, if the factor log-potentials $\Score_{\alpha}$
have special structure (\emph{e.g.}, if the factor
is itself combinatorial, such as a sequence or a tree model), then this structure is preserved since only the
internal unary log-potentials are changed.
Therefore, if evaluating $\text{\sc ComputeMap}(\Score_{\alpha})$
is tractable, so is evaluating
$\text{\sc ComputeMap}({\Score}_{\alpha} + \sum_{i\in \sett{N}(\alpha)} \matr{M}_{i\alpha}^{\top}{\vectsymb{\xi}}_{i\alpha})$.} %
The solutions achieved by each factor are then gathered
and averaged (line~\ref{algline:ddsubgrad-average}),
and the Lagrange multipliers are updated with step size $\eta_t$ (line~\ref{algline:ddsubgrad-lambdaupd}).
The two following propositions establish the convergence
properties of Algorithm~\ref{alg:ddsubgrad}.

\begin{algorithm}[t]
   \caption{Dual Decomposition with Projected Subgradient \citep{Komodakis2007} \label{alg:ddsubgrad}}
\begin{algorithmic}[1]
\STATE {\bfseries input:} graph $\tup{G}$, parameters $\Score$, maximum number of iterations $T$,
stepsizes $(\eta_t)_{t = 1}^T$
   \STATE for each $(i,\alpha)\in \sett{E}$,
   choose $\Score_{i\alpha}$ such that
   $\sum_{\alpha \in \sett{N}(i)}\Score_{i\alpha} = \Score_{i}$
   \STATE initialize $\vectsymb{\lambda} = \vect{0}$
	\FOR{$t=1$ {\bfseries to} $T$}
	\FOR{ {\bfseries each} factor $\alpha \in \sett{F}$}
	\STATE set unary log-potentials ${\vectsymb{\xi}}_{i\alpha} := {\Score}_{i\alpha} + \vectsymb{\lambda}_{i\alpha}$, for $i\in \sett{N}(\alpha)$
	\label{algline:ddsubgrad-unary}
	\STATE set $\widehat{\vectsymb{q}}_{\alpha} :=
	\text{\sc ComputeMap}({\Score}_{\alpha} + \sum_{i\in
          \sett{N}(\alpha)}\matr{M}_{i\alpha}^{\top}\vectsymb{\xi}_{i\alpha})$ \label{algline:ddsubgrad-localmap}
	\STATE set $\widehat{\vectsymb{q}}_{i\alpha} :=
	\matr{M}_{i\alpha}\widehat{\vectsymb{q}}_{\alpha}$, for $i\in\sett{N}(\alpha)$
	\ENDFOR
	\STATE compute average ${\vectsymb{p}}_i := |\sett{N}(i)|^{-1}\sum_{\alpha \in \sett{N}(i)} \widehat{\vectsymb{q}}_{i\alpha}$ for each $i \in \sett{V}$\label{algline:ddsubgrad-average}
	\STATE update $\vectsymb{\lambda}_{i\alpha} := \vectsymb{\lambda}_{i\alpha} -
	\eta_t \left( \widehat{\vectsymb{q}}_{i\alpha} - {\vectsymb{p}}_i \right)$ for each $(i,\alpha) \in \sett{E}$\label{algline:ddsubgrad-lambdaupd}
	\ENDFOR	
   \STATE \textbf{output:} dual variable $\vectsymb{\lambda}$ and upper bound $g(\vectsymb{\lambda})$
\end{algorithmic}
\end{algorithm}


\begin{proposition}[Convergence rate]\label{prop:dd-convergence}
Let the non-negative step size sequence $(\eta_t)_{t\in\set{N}}$ be diminishing and non\-sum\-ma\-ble:
$\lim \eta_t = 0$ and $\sum_{t=1}^\infty \eta_t = \infty$.
Then, Algorithm~\ref{alg:ddsubgrad}
converges to the solution of LP-MAP-D \eqref{eq:dualdecomp}.
Furthermore,
$T = O(1/\epsilon^2)$ iterations of
Algorithm~\ref{alg:ddsubgrad}
are sufficient to achieve a
dual objective
which differs by $\epsilon$
from the optimal value.
\end{proposition}
\begin{proof}
This is a property of projected subgradient algorithms (see, \emph{e.g.}, \citealt{Bertsekas1999}).
\end{proof}

\begin{proposition}[Certificate of optimality]\label{prop:dd-certificate}
If, at some iteration of Algorithm~\ref{alg:ddsubgrad},
all the local subproblems are in agreement
(\emph{i.e.}, if  $\widehat{\vectsymb{q}}_{i\alpha} = \vectsymb{p}_i$ after line
\ref{algline:ddsubgrad-average}, for all $i\in\sett{V}$),
then:
\begin{itemize}
\item $\vectsymb{\lambda}$ is
a solution of LP-MAP-D \eqref{eq:dualdecomp};
\item $\vectsymb{p}$ is binary-valued and a solution of both LP-MAP-P and MAP.
\end{itemize}
\end{proposition}
\begin{proof}
If all local subproblems are in agreement,
then a vacuous update will occur in line~\ref{algline:ddsubgrad-lambdaupd},
and no further changes will occur. Since the algorithm is guaranteed to converge, the current
$\vectsymb{\lambda}$ is optimal.
Also, if all local subproblems are in agreement,
the averaging in line
\ref{algline:ddsubgrad-average}
necessarily yields a binary vector $\vectsymb{p}$.
Since any binary solution of LP-MAP is
also a solution of MAP, the result follows.
\end{proof}

\bigskip

Propositions~\ref{prop:dd-convergence}--\ref{prop:dd-certificate} imply that, if the parameters
$\Score$ are such that LP-MAP is a tight relaxation, then Algorithm~\ref{alg:ddsubgrad} yields
the exact MAP configuration along with a certificate of optimality.
According to Proposition~\ref{prop:dd-convergence}, even if the relaxation is not tight,
Algorithm~\ref{alg:ddsubgrad} still converges to a solution of LP-MAP.
In some problem domains, the LP-MAP is often tight \citep{Koo2010,Rush2010};
for problems with a relaxation gap, techniques to tighten the relaxation have been developed
\citep{Rush2011}. However, in large graphs with many overlapping factors,
it has been observed that Algorithm~\ref{alg:ddsubgrad} can converge quite slowly in practice
\citep{Martins2011bEMNLP}. This is not surprising, given that it attempts to reach a consensus
among all overlapping components; the larger this number, the harder it is to achieve consensus.
In the next section, we propose a new LP-MAP inference algorithm that
is more suitable for this class of problems.


\section{Alternating Directions Dual Decomposition (AD$^3$)}\label{sec:ad3}


\subsection{Addressing the Weaknesses of Dual Decomposition with Projected Subgradient}
\label{sec:weak}
The main weaknesses of Algorithm~\ref{alg:ddsubgrad} reside in the following two aspects.
\begin{enumerate}
\item The dual objective function $g(\vectsymb{\lambda})$ is \emph{non-smooth},
this being why ``subgradients'' are used instead of ``gradients.''
It is well-known that non-smooth optimization lacks some of the good properties
of its smooth counterpart. Namely, there is no guarantee of monotonic improvement in
the objective (see \citealt[p.~611]{Bertsekas1999}). Ensuring convergence requires
using a diminishing step size sequence, which leads to slow convergence rates.
In fact, as stated in Proposition~\ref{prop:dd-convergence}, $O(1/\epsilon^2)$ iterations
are required to guarantee $\epsilon$-accuracy.
\item A close look at Algorithm~\ref{alg:ddsubgrad} reveals that the consensus is
promoted solely by the Lagrange multipliers (line~\ref{algline:ddsubgrad-unary}).
In a economic interpretation,  these represent ``price adjustments''
that lead to a reallocation of resources. However, no ``memory'' exists about past
allocations or adjustments, so the workers never know how far they are from consensus.
It is thus conceivable that a
smart use of these quantities could accelerate convergence.
\end{enumerate}

To obviate the first of these problems, \citet{Johnson2007} proposed smoothing the objective
function through an ``entropic'' perturbation (controlled by a ``temperature" parameter),
which boils down to replacing the $\max$ in \eqref{eq:dualdecomp} by a ``soft-max.''
As a result, all the local subproblems become marginal (rather than MAP) computations.
A related and asymptotically faster method was proposed later by
\citet{Jojic2010}, who address the resulting smooth optimization problem with an
\emph{accelerated gradient method} \citep{Nesterov1983,Nesterov2005}.
That approach guarantees an $\epsilon$-accurate solution after
$O(1/\epsilon)$ iterations, an improvement over the $O(1/\epsilon^2)$ bound of
Algorithm~\ref{alg:ddsubgrad}. However, those methods have some drawbacks.
First, they need to operate at near-zero temperatures; \emph{e.g.}, the $O(1/\epsilon)$
iteration bound of \citet{Jojic2010} requires setting the temperature to $O(\epsilon)$,
which scales the potentials by $O(1/\epsilon)$ and may lead to numerical instabilities
in some problems.
Second, the solution of the local subproblems are always \emph{dense}; although
some marginal values may be low, they are never exactly zero. This contrasts with
the projected subgradient algorithm, for which the solutions of
the local subproblems are MAP configurations. These configurations can be
cached across iterations, leading to important speedups \citep{Koo2010EMNLP}.

We will show that AD$^3$ also yields a
$O(1/\epsilon)$ iteration bound without suffering from the two drawbacks above.
Unlike Algorithm~\ref{alg:ddsubgrad}, AD$^3$ broadcasts \emph{the current global solution}
to the workers, thus allowing them to regularize their subproblems toward that solution.
This promotes a faster consensus, without sacrificing the modularity of dual decomposition.
Another advantage of AD$^3$ is that it keeps track of primal and dual residuals, allowing
monitoring the LP-MAP optimization process and stopping when a desired accuracy level is attained.

\subsection{Augmented Lagrangians and the Alternating Directions Method of Multipliers}\label{sec:augmentedlagrangian}


Let us start with a brief overview of augmented Lagrangian methods.
Consider the following general convex optimization problem with equality constraints:
\begin{equation}\label{eq:convexproblemequalityconstraints}
\begin{array}{ll}
\mathrm{maximize} & f_1(\vectsymb{q}) + f_2(\vectsymb{p}) \\
\text{with respect to} & \vectsymb{q} \in \sett{Q}, \vectsymb{p} \in \sett{P}\\
\text{subject to} & \matr{A}\vectsymb{q} + \matr{B}\vectsymb{p} = \vectsymb{c},
\end{array}
\end{equation}
where $\sett{Q} \subseteq \set{R}^P$ and $\sett{P} \subseteq \set{R}^Q$
are convex sets and
$f_1:\sett{Q} \rightarrow \set{R}\cup \{-\infty\}$ and $f_2:\sett{P} \rightarrow \set{R}\cup \{-\infty\}$ are concave functions.
Note that the LP-MAP problem stated in \eqref{eq:lpmap-reformulation} has this form.
For any $\eta \ge 0$, consider the problem
\begin{equation}\label{eq:convexproblemequalityconstraintschanges}
\begin{array}{ll}
\mathrm{maximize} & f_1(\vectsymb{q}) + f_2(\vectsymb{p}) - \frac{\eta}{2} \| \matr{A}\vectsymb{q} + \matr{B}\vectsymb{p} - \vectsymb{c} \|^2\\
\text{with respect to} & \vectsymb{q} \in \sett{Q}, \vectsymb{p} \in \sett{P}\\
\text{subject to} & \matr{A}\vectsymb{q} + \matr{B}\vectsymb{p} = \vectsymb{c},
\end{array}
\end{equation}
which differs from \eqref{eq:convexproblemequalityconstraints} in the extra term penalizing violations
of the equality constraints; since this term vanishes at feasibility, the two problems have the same solution.
The reason to consider \eqref{eq:convexproblemequalityconstraintschanges} is that its objective is
$\eta$-strongly concave, even if $f_1 + f_2$ is not.
The Lagrangian of problem \eqref{eq:convexproblemequalityconstraintschanges},
\begin{eqnarray}\label{eq:augmented-lagrangian}
L_{\eta}(\vectsymb{q},\vectsymb{p},\vectsymb{\lambda}) = f_1(\vectsymb{q}) + f_2(\vectsymb{p}) +
\DP{\vectsymb{\lambda}}{(\matr{A}\vectsymb{q} + \matr{B}\vectsymb{p} - \vectsymb{c})} - \frac{\eta}{2}
\| \matr{A}\vectsymb{q} + \matr{B}\vectsymb{p} - \vectsymb{c} \|^2,
\end{eqnarray}
is the \emph{$\eta$-augmented Lagrangian} of problem \eqref{eq:convexproblemequalityconstraints}.
The so-called \emph{augmented Lagrangian} (AL) methods \citep[Section~4.2]{Bertsekas1999} address problem
\eqref{eq:convexproblemequalityconstraints} by seeking a saddle point of $L_{\eta_t}$, for some sequence
$(\eta_t)_{t \in \set{N}}$. The earliest instance is the \emph{method of multipliers}  (MM)
\citep{Hestenes1969,Powell1969}, which alternates between a joint update of $\vectsymb{q}$ and $\vectsymb{p}$ through
\begin{eqnarray}\label{eq:almm_jointupdate}
(\vectsymb{q}^{t+1},\vectsymb{p}^{t+1}) &:=& \arg\max_{\vectsymb{q},\vectsymb{p}} \{ L_{\eta_t}(\vectsymb{q},\vectsymb{p},\vectsymb{\lambda}^{t}) \,\,|\,\, \vectsymb{q} \in \sett{Q}, \vectsymb{p} \in \sett{P}\}
\end{eqnarray}
and a gradient update of the Lagrange multiplier,
\begin{eqnarray}\label{eq:almm_lambda_update}
\vectsymb{\lambda}^{t+1} &:=& \vectsymb{\lambda}^{t} - \eta_t (\matr{A}\vectsymb{q}^{t+1} + \matr{B}\vectsymb{p}^{t+1} - \vectsymb{c}).
\end{eqnarray}
Under some conditions, this method is convergent, and even superlinear, if
the sequence $(\eta_t)_{t \in \set{N}}$ is increasing (\citealt[Section~4.2]{Bertsekas1999}).
A shortcoming of the MM is that problem \eqref{eq:almm_jointupdate} may be difficult,
since the penalty term of the augmented Lagrangian couples the variables $\vectsymb{p}$ and $\vectsymb{q}$.
The \emph{alternating directions method of multipliers} (ADMM) avoids this shortcoming, by
replacing the joint optimization \eqref{eq:almm_jointupdate}
by a single block Gauss-Seidel-type step:
\begin{equation}\label{eq:admm_u_update}
\vectsymb{q}^{t+1} := \arg\max_{\vectsymb{q} \in \sett{Q}} L_{\eta_t}(\vectsymb{q},\vectsymb{p}^{t},\vectsymb{\lambda}^{t})
= \arg\max_{\vectsymb{q} \in \sett{Q}} f_1(\vectsymb{q}) + \DP{(\matr{A}^{\top}\vectsymb{\lambda}^t)}{\vectsymb{q}} -
\frac{\eta_t}{2} \| \matr{A}\vectsymb{q} + \matr{B}\vectsymb{p}^t - \vectsymb{c} \|^2,
\end{equation}
\begin{equation}\label{eq:admm_v_update}
\vectsymb{p}^{t+1} := \arg\max_{\vectsymb{p} \in \sett{P}} L_{\eta_t}(\vectsymb{q}^{t+1},\vectsymb{p},\vectsymb{\lambda}^{t})
= \arg\max_{\vectsymb{p} \in \sett{P}} f_2(\vectsymb{p}) + \DP{(\matr{B}^{\top}\vectsymb{\lambda}^t)}{\vectsymb{p}} -
\frac{\eta_t}{2} \| \matr{A}\vectsymb{q}^{t+1} + \matr{B}\vectsymb{p} - \vectsymb{c} \|^2.
\end{equation}
In general, problems \eqref{eq:admm_u_update}--\eqref{eq:admm_v_update} are
simpler than the joint maximization in \eqref{eq:almm_jointupdate}.
ADMM was proposed by \citet{Glowinski75} and \citet{Gabay1976} and
is related to other optimization methods, such as Douglas-Rachford splitting
\citep{Eckstein1992} and proximal point methods (see \citealt{Boyd2011} for an
historical overview).

\subsection{Derivation of AD$^3$}\label{sec:ad3-derivation}

The LP-MAP-P problem \eqref{eq:lpmap-reformulation} can be cast into the form
\eqref{eq:convexproblemequalityconstraints} by proceeding as follows:
\begin{itemize}
\item let $\sett{Q}$ in \eqref{eq:convexproblemequalityconstraints} be
the Cartesian product of simplices, $\sett{Q} := \prod_{\alpha \in \sett{F}} \Delta^{|\Y_{\alpha}|}$,
and $\sett{P} := \set{R}^P$;
\item let $f_1(\vectsymb{q}) := \sum_{\alpha \in \sett{F}}
\DP{\left({\Score}_{\alpha} + \sum_{i\in\sett{N}(\alpha)}
  \matr{M}_{i\alpha}^{\top}\Score_{i\alpha}\right)}{\vectsymb{q}_{\alpha}}$  and $f_2 :\equiv 0$;
\item let $\matr{A}$ in \eqref{eq:convexproblemequalityconstraints}
be a $R\times Q$ block-diagonal matrix, where
$R = \sum_{(i,\alpha)\in \sett{E}} |\Y_i|$,
with one block per factor, which is a vertical concatenation
of the matrices $\{\matr{M}_{i\alpha}\}_{i \in \sett{N}(\alpha)}$;
\item let $-\matr{B}$ be a $R\times P$ matrix of grid-structured blocks, where
the block in the $(i,\alpha)$th row and
the $i$th column is a negative identity matrix
of size $|\Y_i|$, and all the other blocks are
zero;
\item let $\vectsymb{c} := 0$.
\end{itemize}
%
The $\eta$-augmented Lagrangian associated with \eqref{eq:lpmap-reformulation} is
\begin{equation}\label{eq:ddsubgrad_augmlagrangian}
L_{\eta}(\vectsymb{q},\vectsymb{p},\vectsymb{\lambda}) =
\sum_{\alpha \in \sett{F}}
\DP{({\Score}_{\alpha} + \matr{M}_{i\alpha}^{\top}(\Score_{i\alpha} + \vectsymb{\lambda}_{i\alpha}))}{\vectsymb{q}_{\alpha}}
- \sum_{(i,\alpha) \in \sett{E}} \DP{\vectsymb{\lambda}_{i\alpha}}{\vectsymb{p}_i}
-\frac{\eta}{2} \sum_{(i,\alpha) \in \sett{E}}
\|\matr{M}_{i\alpha}\vectsymb{q}_{\alpha}-\vectsymb{p}_i\|^2.
\end{equation}
This is the standard Lagrangian in \eqref{eq:ddsubgrad_lagrangian} plus the Euclidean
penalty term. The ADMM updates are
\begin{align}
\textbf{Broadcast:} \quad & \vectsymb{q}^{(t)} := \arg\max_{\vectsymb{q}\in \sett{Q}}
L_{\eta_{t}}(\vectsymb{q}, \vectsymb{p}^{(t-1)}, \vectsymb{\lambda}^{(t-1)}),\label{eq:ddadmm_mu} \\
\textbf{Gather:} \quad & \vectsymb{p}^{(t)} := \arg\max_{\vectsymb{p}\in \set{R}^P} L_{\eta_{t}}(\vectsymb{q}^{(t)}, \vectsymb{p}, \vectsymb{\lambda}^{(t-1)}),\label{eq:ddadmm_zeta}\\
\textbf{Multiplier update:} \quad & \vectsymb{\lambda}_{i\alpha}^{(t)} := \vectsymb{\lambda}_{i\alpha}^{(t-1)}
- \eta_t \left( \matr{M}_{i\alpha}\vectsymb{q}_{\alpha}^{(t)} - \vectsymb{p}_i^{(t)} \right), \forall (i,\alpha) \in \sett{E}. \label{eq:ddadmm_lambda}
\end{align}
We next analyze the broadcast and gather steps, and prove a proposition about the multiplier update.

\paragraph{Broadcast step.}
The maximization \eqref{eq:ddadmm_mu} can be carried out in parallel at the factors,
as in Algorithm~\ref{alg:ddsubgrad}. The only difference is that, instead of a local MAP computation,
each soft-factor worker now needs to solve a \emph{quadratic program} of the form:
\begin{equation}\label{eq:ddadmm_mu2}
\boxed{%
\max_{\vectsymb{q}_{\alpha} \in \Delta^{|\Y_{\alpha}|}} \DP{\left({\Score}_{\alpha} + \sum_{i\in\sett{N}(\alpha)}\matr{M}_{i\alpha}^{\top}(\Score_{i\alpha} + \vectsymb{\lambda}_{i\alpha})\right)}{\vectsymb{q}_{\alpha}}
-\frac{\eta}{2} \sum_{i\in \sett{N}(\alpha)}
\|\matr{M}_{i\alpha}\vectsymb{q}_{\alpha} - \vectsymb{p}_i\|^2.}
\end{equation}
The subproblem \eqref{eq:ddadmm_mu2} differs from the linear
subproblem \eqref{eq:dualdecompslave01}--\eqref{eq:dualdecompslave02}
in the projected subgradient algorithm by including an Euclidean penalty term,
which penalizes deviations from the global consensus. Sections~\ref{sec:ad3_subproblems} and
\ref{sec:ad3_generalfactors} propose procedures to solve the local subproblems \eqref{eq:ddadmm_mu2}.

\paragraph{Gather step.}
The solution of \eqref{eq:ddadmm_zeta} has a closed form.
Indeed, this problem is separable into independent optimizations, one for each $i\in \sett{V}$;
defining $\vectsymb{q}_{i\alpha} := \matr{M}_{i\alpha} \vectsymb{q}_{\alpha}$,
\begin{equation}\label{eq:ddadmm_zeta2}
\vectsymb{p}_i^{(t)} := \arg \min_{\vectsymb{p}_i \in \set{R}^{|\Y_i|}}  \sum_{\alpha \in \sett{N}(i)}
\left(
\vectsymb{p}_i  - \left( \vectsymb{q}_{i\alpha} - \eta_t^{-1}{\vectsymb{\lambda}}_{i\alpha} \right)
\right)^2 =
 {|\sett{N}(i)|}^{-1} \sum_{\alpha \in \sett{N}(i)}
\left( \vectsymb{q}_{i\alpha} - \eta_t^{-1}\vectsymb{\lambda}_{i\alpha} \right).
\end{equation}
\begin{proposition}\label{prop:ddadmm_dualfeasible}
Let $\vectsymb{\lambda}^{0} = \vect{0}$ and  $\eta_t := \eta$. Then, the sequence produced by
\eqref{eq:ddadmm_mu}--\eqref{eq:ddadmm_lambda} satisfies
\begin{align}\label{eq:ddadmm_dualfeasible}
\sum_{\alpha \in \sett{N}(i)}\vectsymb{\lambda}_{i\alpha}^{(t)} = 0, \\
\label{eq:ddadmm_zeta_update2}
\vectsymb{p}_i^{(t)} = \frac{1}{|\sett{N}(i)|} \sum_{\alpha \in \sett{N}(i)} \vectsymb{q}_{i\alpha}.
\end{align}
\end{proposition}
\begin{proof}
The proof is by induction: $\vectsymb{\lambda}^{(0)} = \vect{0}$ satisfies \eqref{eq:ddadmm_dualfeasible};  for $t>0$,
if $\vectsymb{\lambda}^{(t-1)}$ satisfies \eqref{eq:ddadmm_dualfeasible}, then,
\begin{align}
\sum_{\alpha \in \sett{N}(i)} \vectsymb{\lambda}_{i\alpha}^{(t)} &= \sum_{\alpha \in \sett{N}(i)} \vectsymb{\lambda}_{i\alpha}^{(t-1)}
 - \eta_t \left( \sum_{\alpha \in \sett{N}(i)} \vectsymb{q}_{i\alpha}^{(t)} - |\sett{N}(i)|\vectsymb{p}_i^{(t)} \right) \nonumber\\
 &= \sum_{\alpha \in \sett{N}(i)} \vectsymb{\lambda}_{i\alpha}^{(t-1)}
 - \eta_t \left( \sum_{\alpha \in \sett{N}(i)} \vectsymb{q}_{i\alpha}^{(t)} -
 \sum_{\alpha \in \sett{N}(i)} \left( \vectsymb{q}_{i\alpha}^{(t)} - \eta_t^{-1} \vectsymb{\lambda}_{i\alpha}^{(t-1)}\right) \right)  = \vect{0},\nonumber
\end{align}
\emph{i.e.}, $\vectsymb{\lambda}^{(t)}$ also satisfies \eqref{eq:ddadmm_dualfeasible},
simply by applying the update step \eqref{eq:ddadmm_lambda}. \end{proof}

\vspace{0.4cm}

Assembling all these pieces together leads to the AD$^3$  (Algorithm~\ref{alg:ad3}). Notice
that AD$^3$ retains the modular structure of Algorithm~\ref{alg:ddsubgrad}: both are
iterative consensus algorithms, alternating between a \emph{broadcast} operation,
where subproblems are distributed across local workers
(lines~\ref{algline:ad3_broadcast_begins}--\ref{algline:ad3_broadcast_ends} in Algorithm~\ref{alg:ad3}),
and a \emph{gather} operation, where the local solutions are assembled by a master node, which
updates the global solution (line~\ref{algline:ad3_gather_01})
and adjusts multipliers to promote a consensus (line~\ref{algline:ad3_gather_02}).
The key difference is that AD$^3$ also broadcasts the current global solution to the workers,
allowing them to regularize their subproblems toward that solution, thus
speeding up the consensus.
This is embodied in the procedure {\sc SolveQP} (line \ref{algline:ad3_solveqp}),
which replaces {\sc ComputeMAP} of Algorithm~\ref{alg:ddsubgrad}.

\begin{algorithm}[t]
   \caption{Alternating Directions Dual Decomposition (AD$^3$)\label{alg:ad3}}
\begin{algorithmic}[1]
\STATE {\bfseries input:} graph $\tup{G}$, parameters $\Score$, penalty constant $\eta$
   \STATE initialize $\vectsymb{p}$ uniformly (\emph{i.e.}, $p_i(y_i) = 1/|\Y_i|$, $\forall i \in \sett{V}, y_i \in \Y_i$)
   \STATE initialize $\vectsymb{\lambda} = \vect{0}$
	\REPEAT
	\FOR{{\bfseries each} factor $\alpha \in \sett{F}$} \label{algline:ad3_broadcast_begins}
	\STATE set unary log-potentials ${\vectsymb{\xi}}_{i\alpha} := {\Score}_{i\alpha} + \vectsymb{\lambda}_{i\alpha}$, for $i\in \sett{N}(\alpha)$
	\label{algline:ad3-unary}
	\STATE set $\widehat{\vectsymb{q}}_{\alpha} :=
	\text{\sc SolveQP}\left({\Score}_{\alpha} + \sum_{i\in
          \sett{N}(\alpha)}\matr{M}_{i\alpha}^{\top}\vectsymb{\xi}_{i\alpha},\,\,
        (\vectsymb{p}_i)_{i \in \sett{N}(\alpha)}\right)$ \label{algline:ad3_solveqp}
	\STATE set $\widehat{\vectsymb{q}}_{i\alpha} :=
	\matr{M}_{i\alpha}\widehat{\vectsymb{q}}_{\alpha}$, for $i\in\sett{N}(\alpha)$
	\ENDFOR\label{algline:ad3_broadcast_ends}
	\STATE compute average ${\vectsymb{p}}_i := |\sett{N}(i)|^{-1}\sum_{\alpha \in \sett{N}(i)} \widehat{\vectsymb{q}}_{i\alpha}$ for each $i \in \sett{V}$\label{algline:ad3_gather_01}
	\STATE update $\vectsymb{\lambda}_{i\alpha} := \vectsymb{\lambda}_{i\alpha} -
	\eta \left( \widehat{\vectsymb{q}}_{i\alpha} - {\vectsymb{p}}_i \right)$ for each $(i,\alpha) \in \sett{E}$\label{algline:ad3_gather_02}
	\UNTIL{convergence} 
   \STATE \textbf{output:} primal variables $\vectsymb{p}$ and $\vectsymb{q}$, dual variable $\vectsymb{\lambda}$, upper bound $g(\vectsymb{\lambda})$
\end{algorithmic}
\end{algorithm}

\subsection{Convergence Analysis}\label{sec:ad3_convergence}
\subsubsection{Proof of Convergence}
Convergence of AD$^3$ follows directly from the general convergence
properties of ADMM. Remarkably, unlike in Algorithm~\ref{alg:ddsubgrad} (projected subgradient),
convergence is ensured with a fixed step size.
\begin{proposition}[Convergence]\label{prop:ddadmm1}
Let $(\vectsymb{q}^{(t)}, \vectsymb{p}^{(t)}, \vectsymb{\lambda}^{(t)})_t$
be the sequence of iterates produced by Algorithm~\ref{alg:ad3} with a fixed $\eta_t = \eta$.
Then the following holds:
\begin{enumerate}
\item primal feasibility of LP-MAP-P \eqref{eq:lpmap-reformulation} is achieved in the limit,
\emph{i.e.},
\begin{equation}
\|\matr{M}_{i\alpha}\vectsymb{q}_{\alpha}^{(t)} - \vectsymb{p}_i^{(t)}\| \rightarrow \vect{0}, \quad \forall (i,\alpha) \in \sett{E};
\end{equation}
\item the sequence $(\vectsymb{p}^{(t)}, \vectsymb{q}^{(t)})_{t \in \set{N}}$
converges to a solution of the LP-MAP-P \eqref{eq:lpmap-reformulation};
\item the sequence $(\vectsymb{\lambda}^{(t)})_{t \in \set{N}}$ converges to a solution of the dual LP-MAP-D \eqref{eq:dualdecomp};
\item $\vectsymb{\lambda}^{(t)}$ is dual feasible, for any $t$, thus $g(\vectsymb{\lambda}^{(t)})$ in  \eqref{eq:dualdecomp} approaches the optimum from above.
\end{enumerate}
\end{proposition}
\begin{proof}
1, 2, and 3 are general properties of ADMM \citep[Theorem~4.2]{Glowinski1989}. 
Statement 4 stems from Proposition~\ref{prop:ddadmm_dualfeasible} (simply compare
\eqref{eq:ddadmm_dualfeasible} with \eqref{eq:dualdecomplambdaset}).
\end{proof}

\bigskip
Although Proposition~\ref{prop:ddadmm1} guarantees convergence for any choice of $\eta$,
this parameter may strongly impact the behavior of the algorithm, as
illustrated in Fig.~\ref{fig:ad3_oscillations}. In our experiments, we dynamically adjust
$\eta$ in earlier iterations using the heuristic described in \citet{Boyd2011}, and freeze
it afterwards, not to compromise convergence.

\begin{figure}[h]
\centering
 \includegraphics[scale=0.6]{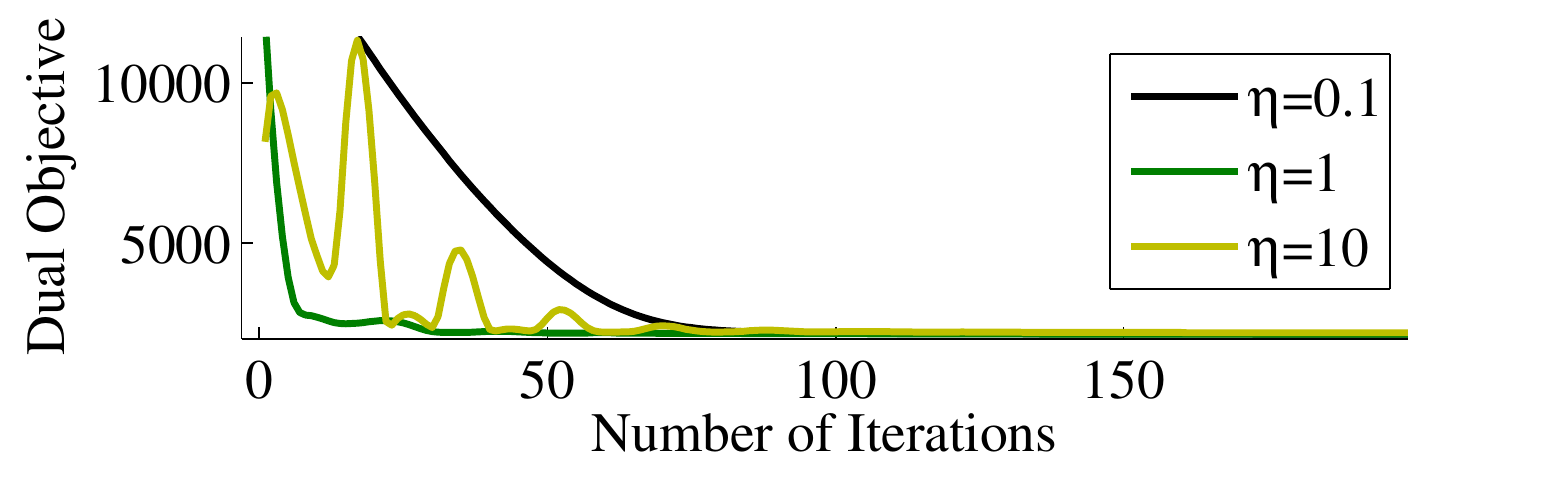}
\caption{Progress in the dual objective for different values of $\eta$
(the factor graph is a binarized $20\times 20$ Potts grid with $8$ states; see
Section~\ref{sec:ddadmm_multivalfact} for details). For $\eta=0.1$,
progress is too slow; for $\eta=10$ there are strong oscillations,
which makes the algorithm slow to take off. In this example, convergence is faster with an
intermediate value, $\eta=1$.}
 \label{fig:ad3_oscillations}
 \end{figure}

\subsubsection{Primal and dual residuals}
Recall from Proposition~\ref{prop:dd-certificate} that the projected subgradient algorithm yields
optimality certificates, if the relaxation is tight (\emph{i.e.}, if the solution of the LP-MAP
problem is the exact MAP), whereas in problems with a relaxation gap, it is harder to decide when to stop.
It would be desirable to have similar guarantees concerning the relaxed primal.%
\footnote{This is particularly important if decoding is embedded in learning, where it is
more useful to obtain a \emph{fractional} solution of the relaxed primal than an approximate
integer one \citep{Kulesza2007,Martins2009ICML}.} %
A general weakness of subgradient algorithms is that they do not have this capacity, being
usually stopped rather arbitrarily after a given number of iterations.
In contrast, ADMM allows keeping track of primal and dual residuals \citep{Boyd2011},
based on which it is possible to obtain certificates, not only for the primal solution
(if the relaxation is tight), but also to terminate when a near optimal relaxed primal
solution has been found.
The \emph{primal residual} $r_P^{(t)}$ is the amount by which the agreement constraints are violated,
\begin{eqnarray}\label{eq:admm_residualprimal}
r_P^{(t)} = \frac{\sum_{(i,\alpha) \in \sett{E}} \|\matr{M}_{i\alpha}\vectsymb{q}_{\alpha}^{(t)} - \vectsymb{p}_i^{(t)}\|^2}{\sum_{(i,\alpha) \in \sett{E}} |\Y_i|} \in [0,1],
\end{eqnarray}
where the constant in the denominator ensures that $r_P^{(t)} \in [0,1]$.
The \emph{dual residual} $r_D^{(t)}$,
\begin{eqnarray}\label{eq:admm_residualdual}
r_D^{(t)} = \frac{\sum_{(i,\alpha) \in \sett{E}} \|\vectsymb{p}_{i}^{(t)} - \vectsymb{p}_i^{(t-1)}\|^2}{\sum_{(i,\alpha) \in \sett{E}} |\Y_i|} \in [0,1],
\end{eqnarray}
is the amount by which a dual optimality condition is violated \citep{Boyd2011}. We adopt as stopping criterion that these two residuals fall
below a threshold, \emph{e.g.}, $10^{-6}$.

\subsubsection{Approximate solutions of the local subproblems}
The next proposition states that convergence may still  hold, even
if the local subproblems are solved approximately, provided the sequence of 
errors is summable. The importance of this result will be
clear in Section~\ref{sec:ad3_generalfactors}, where we describe
a general iterative algorithm for solving the local quadratic subproblems.
Essentially, Proposition~\ref{prop:ddadmm2} allows these subproblems to
be solved numerically up to some accuracy without compromising global convergence,
as long as the accuracy of the solutions improves sufficiently fast over ADMM iterations.
\begin{proposition}[\citealp{Eckstein1992}]\label{prop:ddadmm2}
Let $\eta_t=\eta$, $\widehat{\vectsymb{q}}$ contain the exact solutions of
\eqref{eq:ddadmm_mu2}, and $\tilde{\vectsymb{q}}$ those produced by an approximate algorithm. Then
Proposition~\ref{prop:ddadmm1} still holds, provided that
\begin{equation}
\sum_{t=1}^{\infty} \|\widehat{\vectsymb{q}} - \tilde{\vectsymb{q}}   \| < \infty.
\end{equation}
\end{proposition}

\subsubsection{Convergence rate}

Although ADMM was invented in the 1970s, its convergence rate was unknown until recently.
The next proposition states the $O(1/\epsilon)$ iteration bound of AD$^3$, asymptotically equivalent
to that of the algorithm of \citet{Jojic2010}, therefore better than the $O(1/\epsilon^2)$
bound of Algorithm~\ref{alg:ddsubgrad}.

\begin{proposition} Assume the conditions of Proposition~\ref{prop:ddadmm1}.
Let $\vectsymb{\lambda}^{*}$ be a solution of the dual problem \eqref{eq:dualdecomp},
$\bar{\vectsymb{\lambda}}_T := \sum_{t=1}^T \vectsymb{\lambda}^{(t)}$ be the ``averaged''
Lagrange multipliers after $T$ iterations of AD$^3$, and $g(\bar{\vectsymb{\lambda}}_T)$ the
corresponding estimate of the dual objective. Then, $g(\bar{\vectsymb{\lambda}}_T) - g(\vectsymb{\lambda}^{*}) \leq \epsilon$
after $T \le O(C/\epsilon)$ iterations, where
$C$ satisfies
\begin{equation}\label{eq:ad3_constant_C_finalbound}
C \le \frac{5\eta}{2}
\sum_{i\in\sett{V}} |\sett{N}(i)|\times (1 - |\Y_i|^{-1}) +
\frac{5}{2\eta} \|\vectsymb{\lambda}^{*}\|^2.
\end{equation}
\end{proposition}
\begin{proof}
The proof (detailed in Appendix~\ref{chap:ad3rate})
uses recent results of \citet{He2011} and \citet{Wang2012ICML}, concerning convergence of ADMM in a variational inequality setting.
\end{proof}

\bigskip

As expected, the bound \eqref{eq:ad3_constant_C_finalbound} increases
with the number of overlapping variables,
the size of the sets $\Y_i$, and the magnitude of the optimal dual vector $\vectsymb{\lambda}^{*}$.
Note that if there is a good estimate of $\|\vectsymb{\lambda}^{*}\|$, then \eqref{eq:ad3_constant_C_finalbound}
can be used to choose a step size $\eta$ that minimizes the bound.
Moreover, unlike the $O(1/\epsilon)$ bound of the accelerated method of
\citet{Jojic2010}, which requires specifying $\epsilon$ beforehand,
AD$^3$ does not require pre-specifying a desired accuracy.

The bounds derived so far for all these algorithms are with respect to the \emph{dual} problem---an open
problem is to obtain bounds related to primal convergence.


\subsubsection{Runtime and caching strategies}

In practice, considerable speed-ups can be achieved by caching the subproblems,
a strategy which has also been proposed for the projected subgradient algorithm
\citep{Koo2010EMNLP}.
After a few iterations, many variables $\vectsymb{p}_i$ reach a consensus
(\emph{i.e.},  $\vectsymb{p}_i^{(t)}=\vectsymb{q}_{i\alpha}^{(t+1)}, \forall \alpha \in \sett{N}(i)$)
and enter an idle state: they are left unchanged by the $\vectsymb{p}$-update \eqref{eq:ddadmm_zeta_update2},
and so do the Lagrange variables $\vectsymb{\lambda}^{(t+1)}_{i\alpha}$ \eqref{eq:ddadmm_lambda}).
If at iteration $t$ all variables in a subproblem at factor $\alpha$ are idle,
then $\vectsymb{q}_{\alpha}^{(t+1)}=\vectsymb{q}_{\alpha}^{(t)}$, hence the corresponding subproblem
does not need to be solved.  Typically, many variables and subproblems enter this idle state
after the first few rounds.
We will show the practical benefits of caching in the experimental section (Section~\ref{sec:exp_dependency_parsing}).

\subsection{Exact Inference with Branch-and-Bound}\label{sec:ad3_bnb}


Recall that
AD$^3$, as just described, solves the LP-MAP \emph{relaxation} of the actual problem.
In some problems, this relaxation is
tight (in which case a certificate of optimality will be obtained), but this is not always the case.
When a fractional solution is obtained, it is desirable to have a strategy to recover the
exact MAP solution.

Two observations are noteworthy.
First, as we saw in Section~\ref{sec:lpmap}, the optimal value of the relaxed problem LP-MAP provides an upper bound to the
original problem MAP.
In particular, any feasible dual point
provides an upper bound to the original problem's optimal value.
Second, during execution of the AD$^3$ algorithm, we always keep track
of a sequence of feasible dual points (as guaranteed by Proposition~\ref{prop:ddadmm1}, item 4).
Therefore, each iteration
constructs tighter and tighter upper bounds.
In recent work \citep{Das2012StarSEM}, we proposed a
\emph{branch-and-bound
search} procedure that finds the exact solution of the ILP.
The procedure works recursively as follows: 
\begin{enumerate}
\item Initialize $L=-\infty$ (our best value so far).
\item Run Algorithm~\ref{alg:ad3}. If the solution $\vectsymb{p}^*$ is integer, return $\vectsymb{p}^*$
and set $L$ to the objective value.
If along the execution we obtain an upper bound less than $L$, then
Algorithm~\ref{alg:ad3}  can be safely stopped and return ``infeasible''---this is the \emph{bound} part.
Otherwise (if $\vectsymb{p}^*$ is fractional) go to step 3.
\item Find the ``most fractional'' component of $\vectsymb{p}^*$ (call it $p_j^*(.))$
and \emph{branch}: for every $y_j \in \Y_j$,
constrain $p_j(y_j)=1$ and go to step 2,
eventually obtaining an integer solution $\vectsymb{p}^*|_{y_j}$ or infeasibility.
Return the $\vectsymb{p}^* \in \{\vectsymb{p}^*|_{y_j}\}_{y_j \in \Y_j}$
that yields the largest objective value.
\end{enumerate}
Although this procedure has worst-case exponential runtime,
in many problems for which the relaxations are near-exact
it is found empirically very effective.
We will see one example in Section~\ref{sec:ad3_framenet}.

\section{Local Subproblems in AD$^3$}\label{sec:ad3_subproblems}
\subsection{Introduction}
This section shows how to solve the AD$^3$ local subproblems \eqref{eq:ddadmm_mu2} exactly and efficiently,
in several cases, including Ising models and logic constraint factors. These results will be complemented
in Section~\ref{sec:ad3_generalfactors}, where a new procedure to handle \emph{arbitrary} factors widens
the applicability of AD$^3$.

By subtracting a constant from the objective \eqref{eq:ddadmm_mu2}, re-scaling, and turning the
maximization into a minimization, the problem can be written more compactly as
\begin{align}\label{eq:ddadmm_quad}
\mathrm{minimize} \quad
&
\displaystyle
\frac{1}{2}
\sum_{i \in \sett{N}(\alpha)}
\|\vectsymb{q}_{i\alpha} - \vectsymb{a}_i\|^2
-
\DP{\vectsymb{b}_{\alpha}}{\vectsymb{q}_{\alpha}} \nonumber\\
\text{with respect to}  \quad
& \vectsymb{q}_{\alpha} \in \Delta^{|\Y_{\alpha}|},\,\,  \vectsymb{q}_{i\alpha} \in \set{R}^{|\Y_i|},\,\,\forall i \in \sett{N}(\alpha)\nonumber\\
\text{subject to} \quad
& \vectsymb{q}_{i\alpha} = \matr{M}_{i\alpha}\vectsymb{q}_{\alpha},\,\, \forall i \in \sett{N}(\alpha).
\end{align}
where $\vectsymb{a}_i := \vectsymb{p}_i + \eta^{-1}(\Score_{i\alpha} + \vectsymb{\lambda}_{i\alpha})$ and $\vectsymb{b}_{\alpha} = \eta^{-1}\Score_{\alpha}$.

We show that \eqref{eq:ddadmm_quad} has a closed-form solution or
can be solved exactly and efficiently, in several cases; \emph{e.g.},
for Ising models, for factor graphs imposing first-order logic (FOL) constraints,
and for Potts models (after binarization).
In these cases, AD$^3$ and the projected subgradient algorithm have (asymptotically)
the same computational cost per iteration, up to a logarithmic factor.

\subsection{Ising Models}\label{sec:ddadmm_binpairfact}

Ising models are factor graphs containing only binary pairwise factors.
A binary pairwise factor (say, $\alpha$) is one connecting two binary variables
(say, $Y_1$ and $Y_2$); thus $\Y_1 = \Y_2 = \{0,1\}$ and $\Y_{\alpha} = \{00,01,10,11\}$.
Given that $\vectsymb{q}_{1\alpha} , \vectsymb{q}_{2\alpha} \in \Delta^2$, we can write
$\vectsymb{q}_{1\alpha} = (1-z_1, z_1)$,  $\vectsymb{q}_{2\alpha} = (1-z_2, z_2)$.
Furthermore, since $\vectsymb{q}_{\alpha} \in \Delta^4$ and marginalization requires
that $q_{\alpha}(1,1) + q_{\alpha}(1,0) = z_1$ and $q_{\alpha}(0,1) + q_{\alpha}(1,1) = z_2$,
we can also write $\vectsymb{q}_{\alpha} = (1-z_1-z_2+z_{12}, z_1-z_{12}, z_2-z_{12}, z_{12})$.
Using this parametrization, problem \eqref{eq:ddadmm_quad} reduces to:
\begin{align}\label{eq:pairwise01_main}
\mathrm{minimize} \quad & \frac{1}{2}(z_1 - c_1)^2 + \frac{1}{2}(z_2 - c_2)^2 - c_{12} z_{12} \nonumber\\
\text{with respect to} \quad & z_{1},z_2,z_{12}\in [0,1]^3\nonumber\\
\text{subject to} \quad & z_{12} \le z_1, \quad z_{12} \le z_2, \quad z_{12} \ge z_1+z_2-1,
\end{align}
where
\begin{eqnarray}
c_1 &=& \frac{a_{1\alpha}(1) + 1 - a_{1\alpha}(0) - b_{\alpha}(0,0) + b_{\alpha}(1,0)}{2}\\
c_2 &=& \frac{a_{2\alpha}(1) + 1 - a_{2\alpha}(0) - b_{\alpha}(0,0) + b_{\alpha}(0,1)}{2}\\
c_{12} &=& \frac{b_{\alpha}(0,0) - b_{\alpha}(1,0) - b_{\alpha}(0,1) + b_{\alpha}(1,1)}{2}.
\end{eqnarray}

The next proposition (proved in Appendix \ref{sec:ddadmm_binpairfact_appendix})
establishes a closed form solution for this problem,
which immediately translates into a procedure for {\sc SolveQP} for
binary pairwise factors.
\begin{proposition}\label{prop:ad3_pairwisebinary_solution}
Let $[x]_{\set{U}} := \min\{\max\{x,0\},1\}$ denote projection (clipping) onto the unit interval $\set{U}:=[0,1]$. The solution
$(z_1^*,z_2^*,z_{12}^*)$ of problem \eqref{eq:pairwise01_main} is the following.

If $c_{12}\ge 0$,
\begin{eqnarray}\label{eq:ad3_slave_pairwise_01}
(z_1^*, z_2^*) &=& \left\{
\begin{array}{lll}
([c_1]_{\set{U}}, & [c_2+c_{12}]_{\set{U}}), & \text{if $c_1 > c_2+c_{12}$}\\
([c_1+c_{12}]_{\set{U}}, & [c_2]_{\set{U}}), & \text{if $c_2 > c_1+c_{12}$}\\
([(c_1+c_2+c_{12})/2]_{\set{U}}, & [(c_1+c_2+c_{12})/2]_{\set{U}}), & \text{otherwise,}
\end{array}
\right.\nonumber\\
z_{12}^* &=& \min\{z_1^*, z_2^*\};
\end{eqnarray}
otherwise~,
\begin{eqnarray}\label{eq:ad3_slave_pairwise_02}
(z_1^*, z_2^*) &=& \left\{
\begin{array}{lll}
([c_1+c_{12}]_{\set{U}}, & [c_2+c_{12}]_{\set{U}}), & \text{if $c_1 + c_2+2c_{12} > 1$}\\
([c_1]_{\set{U}}, & [c_2]_{\set{U}}), & \text{if $c_1+c_2 < 1$}\\
([(c_1+1-c_2)/2]_{\set{U}}, & [(c_2+1-c_1)/2]_{\set{U}}), & \text{otherwise,}
\end{array}
\right.\nonumber\\
z_{12}^* &=& \max\{0, z_1^*+z_2^*-1\}.
\end{eqnarray}
\end{proposition}

\subsection{Factor Graphs with First-Order Logic Constraints}
\label{sec:ddadmm_hardconstraintfactors}

Hard constraint factors allow specifying ``forbidden'' configurations,
and have been used in error-correcting decoders \citep{Richardson2008},
bipartite graph matching \citep{Duchi2007}, computer vision
\citep{Nowozin2009}, and natural language processing \citep{Sutton2004,DSmith2008}.
In many applications, \emph{declarative constraints} are useful
for injecting domain knowledge, and first-order logic (FOL) provides a
natural language to express such constraints. This is particularly useful in
learning from scarce annotated data \citep{Roth2004,Punyakanok2005,Richardson2006,Chang2008,Poon2009}.

In this section, we consider hard constraint factors linked to
binary variables, with log-potential functions of the form
\begin{equation}
\score_{\alpha}(\vectsymb{y}_{\alpha}) = \left\{
\begin{array}{ll}
0, & \text{if $\vectsymb{y}_{\alpha} \in \sett{S}_{\alpha}$}\\
-\infty, & \text{otherwise,}
\end{array}
\right.
\end{equation}
where $\sett{S}_{\alpha} \subseteq \{0,1\}^{|\sett{N}(\alpha)|}$ is an \emph{acceptance set}.
These factors can be used for imposing FOL constraints,
as we describe next. We define the \emph{marginal polytope} $\sett{Z}_{\alpha}$ of a hard constraint factor $\alpha$ as the convex hull of its acceptance set,
\begin{equation}\label{eq:marginal_polytope_hard_factor}
\sett{Z}_{\alpha} = \conv \sett{S}_{\alpha}.
\end{equation}

As shown in Appendix~\ref{sec:hard_constraint_factors_intro},
the AD$^3$ subproblem \eqref{eq:ddadmm_quad} associated with a
hard constraint factor is equivalent to that of computing an Euclidean projection
onto its marginal polytope:
\begin{align}\label{eq:quadproj2}
\mathrm{minimize} \quad
&
\|\vectsymb{z} - \vectsymb{z}_0\|^2\nonumber\\
\text{with respect to}  \quad
&
\vectsymb{z} \in \sett{Z}_{\alpha},
\end{align}
where $z_{0i} := (a_i(1) +1 - a_i(0))/2$, for  $i \in \sett{N}(\alpha)$.
We now show how to compute this projection for several
hard constraint factors that are building blocks for writing FOL constraints.
Each of these factors performs a logical function, and hence we represent them
graphically as \emph{logic gates} (Fig.~\ref{fig:marginalpolytopes_logicfactors}).

\begin{figure}[h]
\begin{center}
\includegraphics[width=0.7\textwidth]{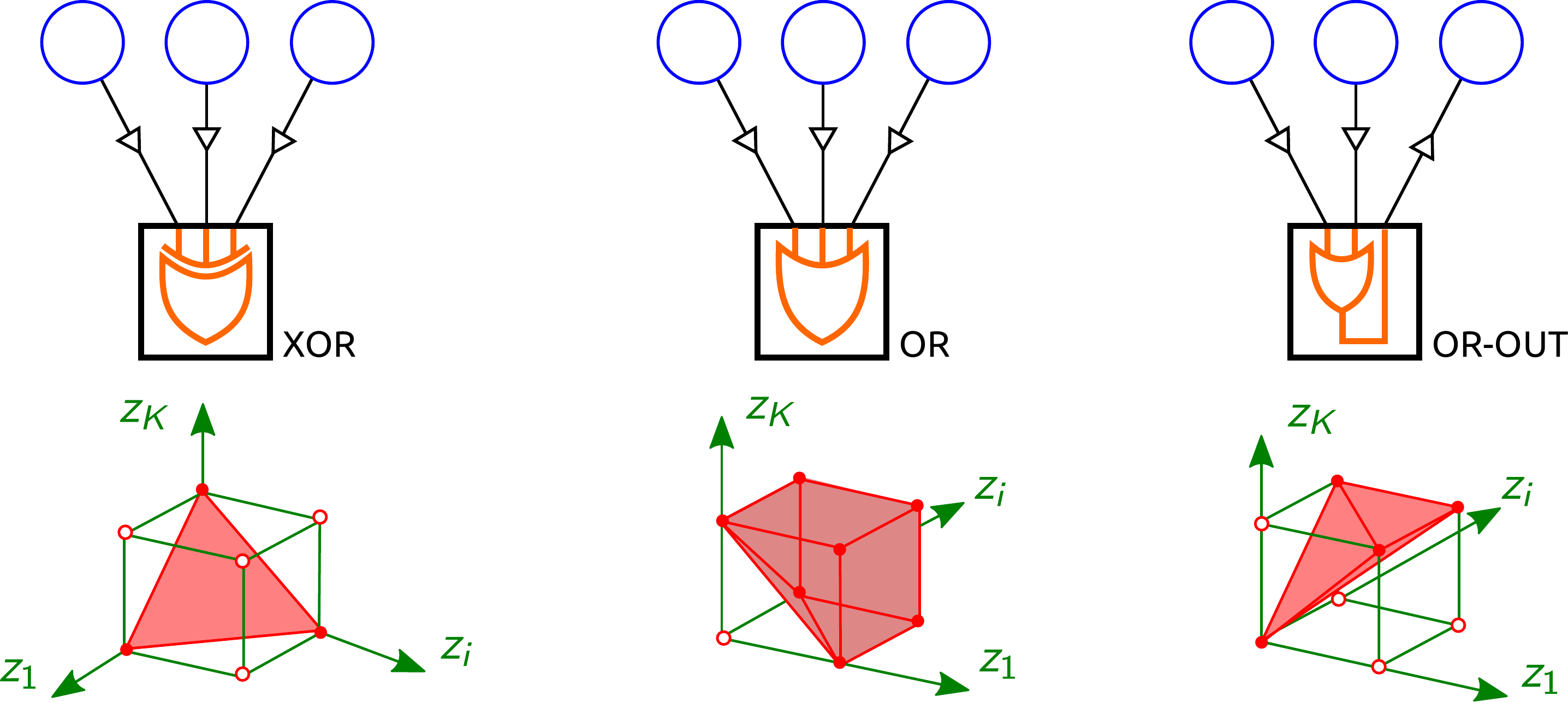}
\caption{Logic factors and their marginal polytopes; the AD$^3$ subproblems \eqref{eq:quadproj2}
are projections onto these polytopes.
Left: the one-hot XOR factor (its marginal polytope is
the probability simplex). Middle: the OR factor. Right: the OR-with-output factor.
}
\label{fig:marginalpolytopes_logicfactors}
\end{center}
\end{figure}

\paragraph{One-hot XOR (Uniqueness Quantification).}
The ``one-hot XOR" factor linked to $K\ge 1$ binary variables is defined through the following potential function:
\begin{equation}
\score_{\mathrm{XOR}}(y_1,\ldots,y_K) :=
\left\{
\begin{array}{ll}
0 & \text{if $\exists!k \in \{1,\ldots,K\}$ s.t. $y_k=1$}\\
-\infty & \text{otherwise,}
\end{array}
\right.
\end{equation}
where $\exists!$ denotes ``there is one and only one.''
The name ``one-hot XOR" stems from the following fact: for $K=2$, $\exp(\score_{\mathrm{XOR}}(.))$ is the logic
eXclusive-OR function; the prefix ``one-hot'' expresses that this generalization to $K>2$ only
accepts configurations with precisely one ``active'' input (not to be mistaken with other
XOR generalizations commonly used for parity checks).
The XOR factor can be used for binarizing a categorical variable,
and to express a statement in FOL of the form $\exists!x : R(x)$.

From \eqref{eq:marginal_polytope_hard_factor},
the marginal polytope associated with the one-hot XOR factor is
\begin{equation}\label{eq:marginalpolytope_xor}
\sett{Z}_{\mathrm{XOR}} = \conv \left\{\vectsymb{y} \in \{0,1\}^K \,\,|\,\, \text{$\exists!k \in \{1,\ldots,K\}$ s.t. $y_k=1$}\right\}
= \Delta^K
\end{equation}
as illustrated in Fig.~\ref{fig:marginalpolytopes_logicfactors}.
Therefore, the AD$3$ subproblem for the XOR factor consists in projecting 
onto the probability simplex, a problem well studied in the literature 
\citep{Brucker1984,Michelot1986,Duchi2008}. 
In Appendix~\ref{sec:appendix_xor}, we describe a simple $O(K\log K)$ algorithm. 
Note that there are $O(K)$ algorithms for this problem 
which are slightly more involved.

\paragraph{OR (Existential Quantification).}
This factor represents a disjunction of $K\ge 1$ binary variables,
\begin{equation}
\score_{\mathrm{OR}}(y_1,\ldots,y_K) :=
\left\{
\begin{array}{ll}
0 & \text{if $\exists k \in \{1,\ldots,K\}$ s.t. $y_k=1$}\\
-\infty & \text{otherwise,}
\end{array}
\right.
\end{equation}
The OR factor can be used to represent a statement in FOL of the form
$\exists x : R(x)$.

From Proposition~\ref{prop:marginalpolytopehard},
the marginal polytope associated with the OR factor is:
\begin{eqnarray}\label{eq:marginalpolytope_or}
\sett{Z}_{\mathrm{OR}} &=& \conv \left\{\vectsymb{y} \in \{0,1\}^K \,\,|\,\, \text{$\exists k \in \{1,\ldots,K\}$ s.t. $y_k=1$}\right\}\\
&=& \left\{\vectsymb{z} \in [0,1]^K \,\,\bigg|\,\, \sum_{k=1}^K z_k \ge 1\right\};
\end{eqnarray}
geometrically, it is a ``truncated'' hypercube, as depicted in Fig.~\ref{fig:marginalpolytopes_logicfactors}.
We derive a $O(K \log K)$ algorithm 
for 
projecting onto $\sett{Z}_{\mathrm{OR}}$,
using a sifting technique and a sort operation (see
Appendix~\ref{sec:appendix_or}).


\paragraph{Logical Variable Assignments: OR-with-output.}

The two factors above define a constraint on a group of existing variables. Alternatively,
we may want to define a new variable (say, $y_{K+1}$) which is the result of an
operation involving other variables (say, $y_1,\ldots,y_{K}$). Among other things, this will
allow dealing with ``soft constraints,'' \emph{i.e.}, constraints that can be violated
but whose violation will decrease the score by some penalty.
An example is the OR-with-output factor:
\begin{equation}
\score_{\mathrm{OR-out}}(y_1,\ldots,y_{K},y_{K+1}) :=
\left\{
\begin{array}{ll}
1 & \text{if $y_{K+1} = y_1 \vee \cdots \vee y_K$}\\
0 & \text{otherwise.}
\end{array}
\right.
\end{equation}
This factor constrains the variable $y_{K+1}$
to indicate the existence of one or more active variables among
$y_1,\ldots,y_{K}$. It can be used to express the following statement in
FOL: $T(x) := \exists z : R(x,z)$.

The marginal polytope associated with the OR-with-output factor (also depicted in Fig.~\ref{fig:marginalpolytopes_logicfactors}):
\begin{eqnarray}\label{eq:marginalpolytope_orout}
\sett{Z}_{\mathrm{OR-out}} &=& \conv \left\{\vectsymb{y} \in \{0,1\}^{K+1} \,\,\bigg|\,\, \text{$y_{K+1} = y_1 \vee \cdots \vee y_K$}\right\}\\
&=& \left\{\vectsymb{z} \in [0,1]^{K+1} \,\,\bigg|\,\, \sum_{k=1}^K z_k \ge z_{K+1}, \,\, z_k \le z_{K+1}, \forall k \in \{1,\ldots,K\}\right\}.
\end{eqnarray}
Although projecting onto $\sett{Z}_{\mathrm{OR-out}}$ is slightly more complicated than
the previous  cases, in Appendix~\ref{sec:appendix_orout},
we propose (and prove correctness of) an $O(K \log K)$ algorithm for this task.

\paragraph{Negations, De Morgan's laws, and AND-with-output.}
The three factors just presented can be extended
to accommodate \emph{negated} inputs, thus adding flexibility.
Solving the corresponding AD$^3$ subproblems can be easily done by
reusing the methods that solve the original problems.
For example, it is straightforward to handle negated conjunctions (NAND),
\begin{eqnarray}
\score_{\mathrm{NAND}}(y_1,\ldots,y_K) &:=&
\left\{
\begin{array}{ll}
-\infty & \text{if $y_k=1$, $\forall k \in \{1,\ldots,K\}$}\\
0 & \text{otherwise,}
\end{array}
\right.\nonumber\\
&=& \score_{\mathrm{OR}}(\neg y_1, \ldots, \neg y_K),
\end{eqnarray}
as well as implications (IMPLY),
\begin{eqnarray}
\score_{\mathrm{IMPLY}}(y_1,\ldots,y_K,y_{K+1}) &:=&
\left\{
\begin{array}{ll}
0 & \text{if $\left( y_1 \wedge \cdots \wedge y_K \right) \Rightarrow y_{K+1}$}\\
-\infty & \text{otherwise}
\end{array}
\right.\nonumber\\
&=& \score_{\mathrm{OR}}(\neg y_1, \ldots, \neg y_K, y_{K+1}).
\end{eqnarray}
In fact, from De Morgan's laws, $\neg \left(  Q_1(x) \wedge \cdots \wedge Q_K(x) \right)$
is equivalent to $\neg Q_1(x) \vee \cdots \vee \neg Q_K(x)$,
and $\left(Q_1 (x) \wedge \cdots \wedge Q_K(x) \right) \Rightarrow R(x)$
is equivalent to $\left( \neg Q_1(x) \vee \cdots \vee \neg Q_K(x)\right) \vee R(x)$.
Another example is the AND-with-output factor,
\begin{eqnarray}
\score_{\mathrm{AND-out}}(y_1,\ldots,y_{K},y_{K+1}) &:=&
\left\{
\begin{array}{ll}
0 & \text{if $y_{K+1} = y_1 \wedge \cdots \wedge y_K$}\\
-\infty & \text{otherwise}
\end{array}
\right.\nonumber\\
&=&
\score_{\mathrm{OR-out}}(\neg y_1,\ldots,\neg y_{K},\neg y_{K+1}),
\end{eqnarray}
which can be used to impose FOL statements of the form $T(x) := \forall z : R(x,z)$.


Let $\alpha$ be a binary constraint factor with marginal polytope $\sett{Z}_{\alpha}$,
and $\beta$ a factor obtained from $\alpha$ by negating the $k$th variable. For notational
convenience, let $\mathrm{sym}_k:[0,1]^K \rightarrow [0,1]^K$ be defined as $(\mathrm{sym}_k(\vectsymb{z}))_k =
1-z_k$ and $(\mathrm{sym}_k(\vectsymb{z}))_i = z_i$, for $i\neq k$.
Then, the marginal polytope $\sett{Z}_{\beta}$ is 
a symmetric transformation of $\sett{Z}_{\alpha}$,
\begin{eqnarray}\label{eq:marginalpolytope_negations}
\sett{Z}_{\beta} &=& \Bigl\{\vectsymb{z} \in [0,1]^K \,\,\big|\,\, \mathrm{sym}_k(\vectsymb{z})  \in \sett{Z}_{\alpha}\Bigr\},
\end{eqnarray}
and, if $\mathrm{proj}_{\sett{Z}_{\alpha}}$ denotes the projection operator onto $\sett{Z}_{\alpha}$,
\begin{equation}
\mathrm{proj}_{\sett{Z}_{\beta}}(\vectsymb{z}) = \mathrm{sym}_k \left( \mathrm{proj}_{\sett{Z}_{\alpha}} (\mathrm{sym}_k(\vectsymb{z}))\right).
\end{equation}
Naturally, $\mathrm{proj}_{\sett{Z}_{\beta}}$ can be computed as efficiently as $\mathrm{proj}_{\sett{Z}_{\alpha}}$ and, by
induction, this procedure can be generalized to an arbitrary number of negated variables.

\subsection{Potts Models and Graph Binarization}\label{sec:ddadmm_multivalfact}
Although general factors lack closed-form solutions of the corresponding AD$^3$ subproblem \eqref{eq:ddadmm_quad},
it is possible to \emph{binarize} the graph, \emph{i.e.}, to convert it into an equivalent one
that only contains binary variables and XOR factors. The procedure is as follows:
\begin{itemize}
\item For each variable node $i \in \sett{V}$, define binary variables $U_{i,y_i} \in \{0,1\}$, for each state $y_i \in \Y_i$;
link these variables to a XOR factor, imposing $\sum_{y_i \in \Y_i} p_i(y_i) = 1$.
\item For each factor $\alpha \in \sett{F}$, define binary variables
$U_{\alpha, \vectsymb{y}_{\alpha}}\in\{0,1\}$ for every $\vectsymb{y}_{\alpha}\in \Y_{\alpha}$.
For each edge $(i,\alpha) \in \sett{E}$ and each $y_i \in \Y_i$, link
variables $\{U_{\alpha, \vectsymb{y}_{\alpha}}\,\,|\,\, \vectsymb{y}_{\alpha} \sim y_i\}$
and $\neg U_{i,y_i}$ to a XOR factor; this imposes the constraint $p_i(y_i) =
\sum_{\vectsymb{y}_{\alpha} \sim y_i } q_{\alpha}(\vectsymb{y}_{\alpha})$.
\end{itemize}
The resulting binary graph is one for which we already presented the machinery
needed for solving efficiently the corresponding AD$^3$ subproblems. As an example, for
Potts models (graphs with only pairwise factors and variables that have more than two states),
the computational cost per AD$^3$ iteration on the binarized graph is asymptotically the same
as that of the projected subgradient method and other message-passing algorithms; for
details, see  \citet{Martins2012PhDThesis}.

\section{An Active Set Method For Solving the AD$^3$ Subproblems}\label{sec:ad3_generalfactors}
When dealing with arbitrary factors, an alternative to binarization is to use an
\emph{inexact} algorithm that becomes increasingly accurate as AD$^3$ proceeds
(see Proposition~\ref{prop:ddadmm2}); this increasing accuracy can be
achieved by warm-starting subproblem solvers with the solutions obtained in the
previous iteration. We next describe such an approach, which, remarkably, only requires
a black-box solving the local MAP subproblems (functions {\sc ComputeMAP} in Algorithm~\ref{alg:ddsubgrad}).
This makes AD$^3$ applicable to a wide range of problems, by invoking specialized combinatorial
algorithms for computing the MAP for factors that impose structural constraints.

The key to our approach is the conversion of the original
quadratic program \eqref{eq:ddadmm_quad} into a sequence of linear problems, accomplished
via an \emph{active set method} \citep[Section~16.4]{Nocedal1999}.
Let us start by writing the local subproblem in \eqref{eq:ddadmm_quad} in the
compact form
\begin{align}\label{eq:ad3_subproblem_primal}
\text{minimize} \quad & \frac{1}{2}\|\vectsymb{u} - \vectsymb{a}\|^2
- \vectsymb{b}^{\top} \vectsymb{v}\\
\text{with respect to} \quad & \vectsymb{u} \in \set{R}^{\sum_{i\in \sett{N}(\alpha)}|\Y_i|}, \quad
\vectsymb{v} \in \set{R}^{|\Y_{\alpha}|}\nonumber\\
\text{subject to} \quad & \vectsymb{u} = \matr{M} \vectsymb{v}, \quad \vect{1}^{\top}\vectsymb{v} = 1, \quad \vectsymb{v} \ge 0,\nonumber
\end{align}
where $\vectsymb{u} := (\vectsymb{q}_{i\alpha})_{i\in \sett{N}(\alpha)}$ collects all the variable marginals
and $\vectsymb{v} := \vectsymb{q}_{\alpha}$ is the factor marginal, and similarly for $\vectsymb{a} := (\vectsymb{a}_i)_{i\in \sett{N}(\alpha)}$
and  $\vectsymb{b} := \vectsymb{b}_{\alpha}$; finally, $\matr{M} := (\matr{M}_{i\alpha})_{i \in \sett{N}(\alpha)}$,
denotes a matrix with $\sum_i |\Y_i|$ rows and $|\Y_{\alpha}|$ columns. The next crucial proposition (proved in Appendix~\ref{sec:proof_ad3_subproblem_primal_sparse}) states that
problem \eqref{eq:ad3_subproblem_primal} always admits a \emph{sparse solution}.
\begin{proposition}\label{prop:ad3_subproblem_primal_sparse}
Problem \eqref{eq:ad3_subproblem_primal} admits a solution $\vectsymb{v}^*\in \set{R}^{|\Y_{\alpha}|}$ with at most $\sum_{i\in \sett{N}(\alpha)} |\Y_i| - \sett{N}(\alpha) + 1$ non-zero components.
\end{proposition}

The fact that the solution lies in a low dimensional subspace, makes
active set methods appealing, since they only keep track of an \emph{active set} of variables,
that is, the non-zero components. Proposition \ref{prop:ad3_subproblem_primal_sparse} shows that
such an algorithm only needs to maintain at most $O(\sum_i |\Y_i|)$ elements in the active
set---note the \emph{additive}, rather than multiplicative, dependency
on the number of values of the variables. This alleviates eventual concerns about memory and storage.

\paragraph{Lagrangian and Dual Problem.}
Problem \eqref{eq:ad3_subproblem_primal} has $O(\Y_{\alpha})$ variables and constraints, a
number that grows exponentially with $|\sett{N}(\alpha)|$. We next
derive a dual formulation with only $O(\sum_{i}|\Y_i|)$ variables.
The Lagrangian of \eqref{eq:ad3_subproblem_primal} is
\begin{equation}
L(\vectsymb{u},\vectsymb{v},\vectsymb{w},\tau,\vectsymb{\lambda}) =
\frac{1}{2}\|\vectsymb{u} - \vectsymb{a}\|^2
- \vectsymb{b}^{\top} \vectsymb{v}
- \vectsymb{w}^{\top} (\matr{M}\vectsymb{v} - \vectsymb{u})
- \tau (1 - \vect{1}^{\top}\vectsymb{v})
- \vectsymb{\lambda}^{\top}\vectsymb{v}.
\end{equation}
Equating $\nabla_{\vectsymb{u}} L(\vectsymb{u},\vectsymb{v},\vectsymb{w},\tau,\vectsymb{\lambda})$ and
$\nabla_{\vectsymb{v}} L(\vectsymb{u},\vectsymb{v},\vectsymb{w},\tau,\vectsymb{\lambda})$ to zero leads to the equations:
\begin{eqnarray}
\vectsymb{u} &=& \vectsymb{a} - \vectsymb{w}\\
\matr{M}^{\top} \vectsymb{w} + \vectsymb{b} &=& \tau \vect{1} - \vectsymb{\lambda}.
\end{eqnarray}
Since the Lagrange variables $\vectsymb{\lambda}$ are constrained to be non-negative,
the dual problem takes the form (after replacing the maximization with a minimization and subtracting a constant),
\begin{align}\label{eq:ad3_subproblem_dual}
\text{minimize} \quad & \frac{1}{2}\|\vectsymb{w} - \vectsymb{a}\|^2 + \tau\\
\text{with respect to} \quad & \vectsymb{w} \in \set{R}^{\sum_i |\Y_i|}, \quad
{\tau} \in \set{R}\nonumber\\
\text{subject to} \quad &
\matr{M}^{\top} \vectsymb{w} + \vectsymb{b} \le \tau \vect{1},\nonumber
\end{align}
which can be seen as a ``projection with slack'' onto $\{\vectsymb{w} \,\,|\,\, \matr{M}^{\top} \vectsymb{w} + \vectsymb{b} \le \vect{0}\}$.
We now apply the active set method of \citet[Section~16.4]{Nocedal1999} to this problem. Letting $\vectsymb{m}_r$ denote the
$r$th column of matrix $\matr{M}$, this method keeps track of a working set of constraints assumed to be \emph{active}:
\begin{equation}\label{eq:ad3_workingset}
\sett{W} := \left\{r \in \{1,\ldots,|\Y_{\alpha}|\}\,\,|\,\, \vectsymb{m}_r^{\top}\vectsymb{w} + b_r - \tau = 0\right\};
\end{equation}
by complementary slackness, at the optimum, this set contains the support of $\vectsymb{v}^*$.



\paragraph{KKT conditions.}

The KKT equations of \eqref{eq:ad3_subproblem_dual} are:
\begin{align}
\vectsymb{u} - \vectsymb{a} + \vectsymb{w} = 0 &\quad (\nabla_{\vectsymb{u}} L = 0)\\
\matr{M}^{\top}\vectsymb{w} + \vectsymb{b} = \tau \vect{1} - \vectsymb{\lambda} &\quad (\nabla_{\vectsymb{v}} L = 0) \\
\matr{M} \vectsymb{v} = \vectsymb{u} &\quad \text{(Primal feasibility)} \\
\vect{1}^{\top} \vectsymb{v} = 1 &\quad \text{(Primal feasibility)} \\
\vectsymb{y} \ge \vect{0} &\quad \text{(Primal feasibility)} \\
\vectsymb{\lambda} \ge \vect{0} &\quad \text{(Dual feasibility)} \\
\vectsymb{\lambda}^{\top}\vectsymb{v}  = \vect{0} &\quad \text{(Complementary slackness)}.
\end{align}
We can eliminate variables $\vectsymb{u}$ and $\vectsymb{w}$ and
reduce the above set of equations to
\begin{eqnarray}
\matr{M}^{\top}\matr{M} \vectsymb{v} + \tau \vect{1}  &=& \matr{M}^{\top}\vectsymb{a} + \vectsymb{b} + \vectsymb{\lambda}
\label{eq:ad3_subproblem_kkt_start}\\
\vect{1}^{\top} \vectsymb{v} &=& 1  \\
\vectsymb{v} &\ge& \vect{0} \\
\vectsymb{\lambda} &\ge& \vect{0}  \\
\vectsymb{\lambda}^{\top}\vectsymb{v}  &=& \vect{0}\label{eq:ad3_subproblem_kkt_end}.
\end{eqnarray}
Let $\sett{I}$ denote the, obviously unknown, support of $\vectsymb{v}$.
The algorithm maintains the active set $\sett{W}$ \eqref{eq:ad3_workingset}, which is a guess of $\sett{I}$.
Let $\vectsymb{v}_{\sett{I}}$ and $\vectsymb{\lambda}_{\sett{I}}$ be the subvectors indexed by $\sett{I}$
and  $\matr{M}_{\sett{I}}$ be the submatrix of $\matr{M}$ with column indices in $\sett{I}$.
Equations \eqref{eq:ad3_subproblem_kkt_start}--\eqref{eq:ad3_subproblem_kkt_end}  imply
$\vectsymb{\lambda}_{\sett{I}} = \vect{0}$ and the system of equations
\begin{equation}\label{eq:ad3_subproblem_kktsystem}
\left[
\begin{array}{cc}
\matr{M}_{\sett{I}}^{\top}\matr{M}_{\sett{I}} & \vect{1}\\
\vect{1}^{\top} & 0
\end{array}
\right]
\left[
\begin{array}{c}
 \vectsymb{v}_{\sett{I}}\\
 \tau
\end{array}
\right] =
\left[
\begin{array}{c}
\matr{M}_{\sett{I}}^{\top}\vectsymb{a} + \vectsymb{b}_{\sett{I}}\\
1
\end{array}
\right].
\end{equation}
If the matrix in the left-hand side is non-singular,
this system has a unique solution $(\widehat{\vectsymb{v}}_{\sett{I}},\widehat{\tau})$.%
\footnote{If this matrix is singular, any vector in its null space can be used as a solution.} %
If $\widehat{\vectsymb{v}}$ (given by zero-padding $\widehat{\vectsymb{v}}_{\sett{I}}$) and $\widehat{\tau}$ satisfy
the KKT conditions (\eqref{eq:ad3_subproblem_kkt_start}--\eqref{eq:ad3_subproblem_kkt_end}),
\emph{i.e.}, if we have $\widehat{\vectsymb{v}} \ge \vect{0}$ and
\begin{eqnarray}\label{eq:ad3_subproblem_kkt}
\matr{M}^{\top}\matr{M} \widehat{\vectsymb{v}} + \widehat{\tau} \vect{1}  &\ge& \matr{M}^{\top}\vectsymb{a} + \vectsymb{b},
\end{eqnarray}
then the set $\sett{I}$ is correct and $\widehat{\vectsymb{v}}$ is a solution of \eqref{eq:ad3_subproblem_primal}.
To use this rationale algorithmically, the following two operations need to be performed:
\begin{itemize}
\item Solving the KKT system \eqref{eq:ad3_subproblem_kktsystem}, \emph{i.e.}, inverting
a $|\sett{I}| \times |\sett{I}|$  matrix. Since this operation is repeated after inserting or removing an
element from the active set, it can be done efficiently (namely, in time $O(|\sett{I}|^2)$) by keeping
track of the inverse matrix. From Proposition~\ref{prop:ad3_subproblem_primal_sparse},
this runtime is manageable, since the size of the active set is limited by $ O(\sum_i|\Y_i|)$.
Note also that adding a new configuration $\vectsymb{y}_{\alpha}$ to the active set,
corresponds to inserting a new column in $\matr{M}_{\sett{I}}$ thus the matrix inversion requires
updating $\matr{M}_{\sett{I}}^{\top}\matr{M}_{\sett{I}}$.
From the definition of $\matr{M}$ and simple algebra, the $(\vectsymb{y}_{\alpha},\vectsymb{y}'_{\alpha})$ entry in
$\matr{M}^{\top}\matr{M}$ is simply the \emph{number of common values} between the configurations
$\vectsymb{y}_{\alpha}$ and $\vectsymb{y}'_{\alpha}$. Hence, when a new configuration
$\vectsymb{y}_{\alpha}$ is added to the active set $\sett{W}$, that configuration needs to
be compared with all the others already in $\sett{W}$.
\item Checking if the constraints \eqref{eq:ad3_subproblem_kkt} are satisfied;
by setting $\widehat{\vectsymb{u}} = \matr{M} \widehat{\vectsymb{v}}$ and $\widehat{\vectsymb{w}} = \vectsymb{a}-\widehat{\vectsymb{u}}$,
the constraints can be written as
$\matr{M}^{\top}\widehat{\vectsymb{w}} + \vectsymb{b} \le \widehat{\tau} \vect{1}$,
and they hold if and only if $\max_{r \in \Y_{\alpha}} \vectsymb{m}_r^{\top} \widehat{\vectsymb{w}} + \vectsymb{b}_r \le \widehat{\tau}.$
Hence, checking the constraints only involves computing this maximum, which, interestingly, coincides
with finding the MAP configuration of factor $\alpha$ with variable log-potentials  $\widehat{\vectsymb{w}}$
and factor log-potentials $\vectsymb{b}$. This can be obtained through the function {\sc ComputeMAP}.
\end{itemize}


\paragraph{Active set algorithm.}
Algorithm~\ref{alg:ad3_subproblem_activeset} formally describes the complete procedure.
At each iteration $s$, an active set $\sett{W}_{s}$ stores the current guess
of the support of $\vectsymb{v}$.
The active set is initialized arbitrarily: \emph{e.g.}, in the first outer iteration of AD$^3$,
it is initialized with the single output $\vectsymb{y}_{\alpha}$, which is the MAP configuration
corresponding to log-potentials $\vectsymb{a}$ and $\vectsymb{b}$; in subsequent AD$^3$ iterations,
it is warm-started with the solution obtained in the previous iteration.
At each inner iteration, the KKT system \eqref{eq:ad3_subproblem_kktsystem} is solved
given the current active set.

If the solution is the same as in the previous round,
a black-box {\sc ComputeMAP} is invoked to check for KKT violations;
if there are none, the algorithm returns; otherwise, it adds the most
violated constraint to the active set.
If the solution of the KKT system $\widehat{\vectsymb{v}}$ is different from the previous one
($\widehat{\vectsymb{v}}_s$), the algorithm sets $\vectsymb{v}_{s+1}
:= (1-\alpha_s)\vectsymb{v}_{s} + \alpha_s \widehat{\vectsymb{v}}$, where the step size
$\alpha_s$ is chosen  by line search to yield the largest improvement in the objective, while
keeping the constraints satisfied; this has a closed form solution \citep[p.~457]{Nocedal1999}.

\begin{algorithm}[t]
   \caption{Active Set Algorithm for Solving a General AD$^3$ Subproblem \label{alg:ad3_subproblem_activeset}}
\begin{algorithmic}[1]
\STATE {\bfseries input:} Parameters $\vectsymb{a},\vectsymb{b},\matr{M}$,
	starting point $\vectsymb{v}_0$
   \STATE initialize $\sett{W}_0$ as the support of $\vectsymb{v}_0$
	\FOR{$s=0,1,2,\ldots$}
	\STATE solve the KKT system and obtain $\widehat{\vectsymb{v}}$ and $\widehat{\tau}$ (Eq.~\ref{eq:ad3_subproblem_kktsystem})
	\IF{$\widehat{\vectsymb{v}} = \vectsymb{v}_s$}
	\STATE compute $\widehat{\vectsymb{u}} := \matr{M} \widehat{\vectsymb{v}}$ and
	$\widehat{\vectsymb{w}} := \vectsymb{a}-\widehat{\vectsymb{u}}$
	\STATE obtain the tighter constraint $r$ via $\vectsymb{e}_r = \text{{\sc
            ComputeMAP}}({\vectsymb{b}} + \matr{M}^{\top}\widehat{\vectsymb{w}})$
	\IF{$\vectsymb{m}_r^{\top} \widehat{\vectsymb{w}}  + \vectsymb{b}_r \le \widehat{\tau}$}
	\STATE return solution $\widehat{\vectsymb{u}}$ and $\widehat{\vectsymb{v}}$
	\ELSE
	\STATE  add the most violated constraint to the active set: $\sett{W}_{s+1} := \sett{W}_{s} \cup \{r\}$
	\ENDIF
	\ELSE
	\STATE compute the interpolation constant $\alpha_s := \min \left\{1, \min_{r \in \sett{W}_{s} \,\, \text{s.t.}\,\, v_{s,r} > \widehat{v}_{r}}
\frac{v_{s,r}}{v_{s,r} - \widehat{v}_{r}} \right\}$
	\STATE set $\vectsymb{v}_{s+1} := (1-\alpha_s)\vectsymb{v}_{s} + \alpha_s \widehat{\vectsymb{v}}$
	\IF{there are blocking constraints}
	\STATE pick a blocking constraint $r$
	\STATE remove that blocking constraint from the active set: $\sett{W}_{s+1} := \sett{W}_{s} \setminus \{r\}$
	\ENDIF
	\ENDIF
	\ENDFOR	
   \STATE \textbf{output:} $\widehat{\vectsymb{u}}$ and $\widehat{\vectsymb{v}}$
\end{algorithmic}
\end{algorithm}

Each iteration of Algorithm~\ref{alg:ad3_subproblem_activeset} always improves the dual objective
and, with a suitable strategy to prevent cycles and stalling, the algorithm is guaranteed to stop after
a finite number of steps \citep[Theorem~16.5]{Nocedal1999}. In practice, since it is run as a subroutine of AD$^3$,
Algorithm~\ref{alg:ad3_subproblem_activeset} does not need to be run to optimality, which is particularly
convenient in early iterations of AD$^3$ (this is supported by Proposition~\ref{prop:ddadmm2}).
The ability to warm-start each outer AD$^3$ iteration with the solution from the
previous round is very useful in practice: we have observed that, thanks to this warm-starting strategy,
very few inner iterations are typically necessary for the correct active set to be  identified.
We will provide some empirical evidence in Section~\ref{sec:exp_dependency_parsing}.



\section{Experiments}\label{sec:experiments}

\subsection{Baselines and Problems}
In this section, we provide an empirical comparison between
AD$^3$ (Algorithm~\ref{alg:ad3}) and four other algorithms:
Star-MSD \citep{Sontag2011OPT}, which is an accelerated version of
the max-sum diffusion algorithm \citep{Kovalevsky1975,Werner2007};
generalized MPLP \citep{Globerson2008}; the projected subgradient
algorithm of \citet{Komodakis2007} (Algorithm~\ref{alg:ddsubgrad})
and its accelerated version introduced by \citet{Jojic2010} (mentioned in Subsection~\ref{sec:weak}).
All these algorithms address the LP-MAP problem; the first are message-passing methods performing
block coordinate descent in the dual, whereas the last two are based on dual decomposition.
All the baselines have the same  algorithmic complexity per iteration, which is
asymptotically the same as that of the AD$^3$ applied to a binarized graph,
but different from that of AD$^3$ with the active set method.

We compare the performance of the algorithms above in several datasets,
including
synthetic Ising and Potts models, protein design problems,
and two problems in natural language processing:
frame-semantic parsing and non-projective dependency parsing.
The graphical models associated with these problems
are quite diverse, containing pairwise binary factors (AD$^3$ subproblems
solved as described in Section~\ref{sec:ddadmm_binpairfact}),
first-order logic factors (addressed using
the tools of Section~\ref{sec:ddadmm_hardconstraintfactors}),
dense factors, and structured factors (tackled with the active set method
of Section~\ref{sec:ad3_generalfactors}).



\subsection{Synthetic Ising and Potts Models}

\paragraph{Ising models.}
Fig.~\ref{fig:exp_pairwisemrf} reports experiments with random Ising models,
with single-node log-potentials chosen as $\score_i(1)-\score_i(0) \sim \sett{U}[-1,1]$
and random edge couplings in $\sett{U}[-\rho,\rho]$, where $\rho \in \{0.5, 1, 1.5, 2\}$.
Decompositions are edge-based for all methods. For MPLP and Star-MSD,
primal feasible solutions $(\widehat{y}_i)_{i\in\sett{V}}$ are obtained by decoding the
single node messages \citep{Globerson2008}; for the dual decomposition methods,
$\widehat{y}_i = \argmax_{y_i} p_i(y_i)$.

We observe that the projected subgradient is the slowest, taking a long time to find a
``good'' primal feasible solution, arguably due to the large number of components.
The accelerated dual decomposition method \citep{Jojic2010} is also not competitive in
this setting, as it takes many iterations to reach a near-optimal region.
MPLP performs slightly better than Star-MSD and both are comparable to AD$^3$ in terms
of convergence of the dual objective. However, AD$^3$ outperforms all competitors at
obtaining a ``good'' feasible primal solution in early iterations (it retrieved the exact
MAP in all cases, in no more than $200$ iterations). We conjecture that this rapid
progress in the primal is due to the penalty term in the augmented Lagrangian
(Eq.~\ref{eq:ddsubgrad_augmlagrangian}), which pushes for a feasible primal
solution of the relaxed LP.

\begin{figure}[t]
\centering{\includegraphics[scale=0.45]{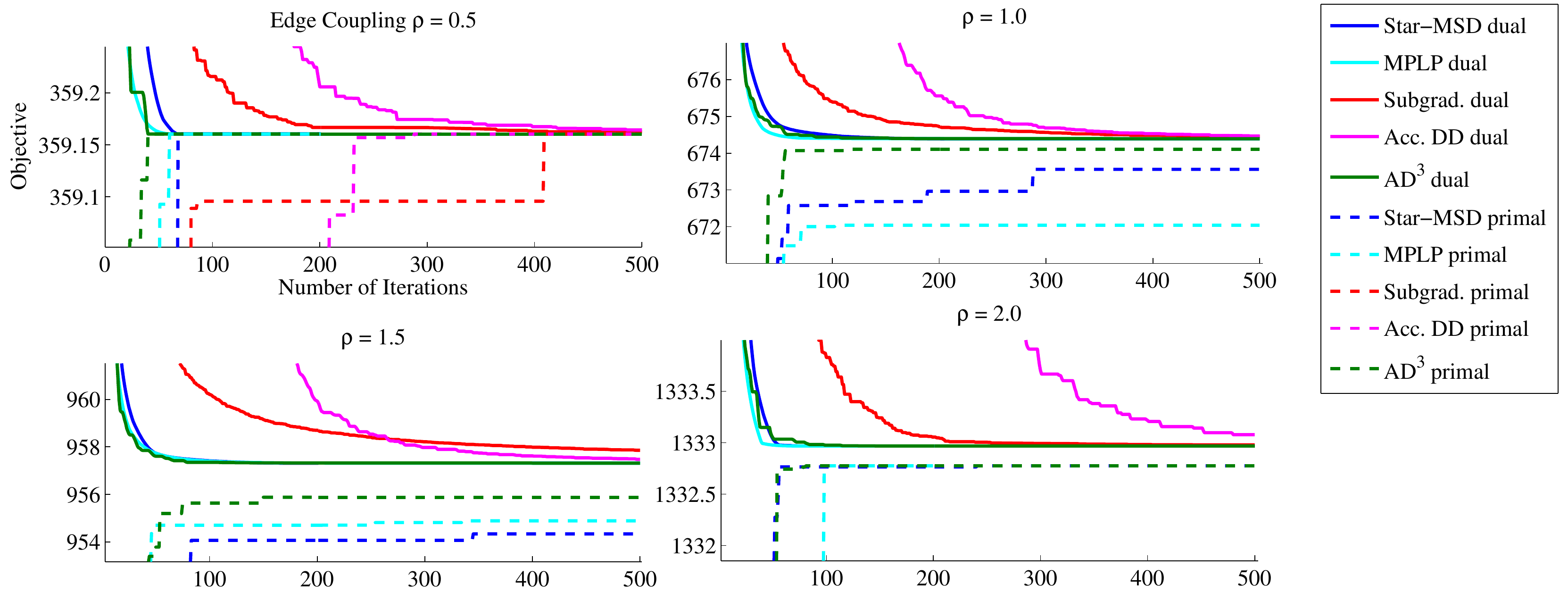}}
\caption{Evolution of the dual objective and the best primal feasible
one in the experiments with $30\times 30$ random Ising models,
generated as described in the main text. For the subgradient method, the
step sizes are $\eta_t \! =\! \eta_0/t$, with $\eta_0$ adjusted (with halving steps)
to yield the maximum dual improvement in $10$ iterations (those steps are not plotted).
For accelerated dual decomposition,  the most favorable $\epsilon\! \in \! \{0.1,1,10,100\}$ is chosen.
AD$^3$ uses $\eta\! =\! 5.0$.}
\label{fig:exp_pairwisemrf}
\end{figure}

\paragraph{Potts models.}
The effectiveness of AD$^3$ in the non-binary case is assessed using
random Potts models, with single-node log-potentials chosen as $\score_i(y_i) \sim \sett{U}[-1,1]$
and pair-wise log-potentials as $\score_{ij}(y_i,y_j) \sim \sett{U}[-10,10]$ if $y_i = y_j$ and $0$ otherwise.
The two baselines (MPLP and accelerated dual decomposition) use the same edge decomposition as
before, since they handle multi-valued variables; for AD$^3$, the graph is binarized (see
Section~\ref{sec:ddadmm_multivalfact}). As observed in the left plot of Fig.~\ref{fig:exp_multi_pairwisemrf},
MPLP decreases the objective very rapidly in the beginning and then slows down;
the accelerated dual decomposition algorithm, although slower in early iterations,
converging faster. AD$^3$ converges as fast as the accelerated dual decomposition algorithm,
and is almost as fast as MPLP in early iterations. The right plot of Fig.~\ref{fig:exp_multi_pairwisemrf}
shows that AD$^3$ with the subproblems solved by the active set method
(see Section~\ref{sec:ad3_generalfactors}) clearly outperforms the binarization-based version.
Notice that since AD$^3$ with the active set method involves more computation per iteration,
we plot the objective values with respect to runtime.

\begin{figure}[t]
\centering
 \includegraphics[scale=0.48]{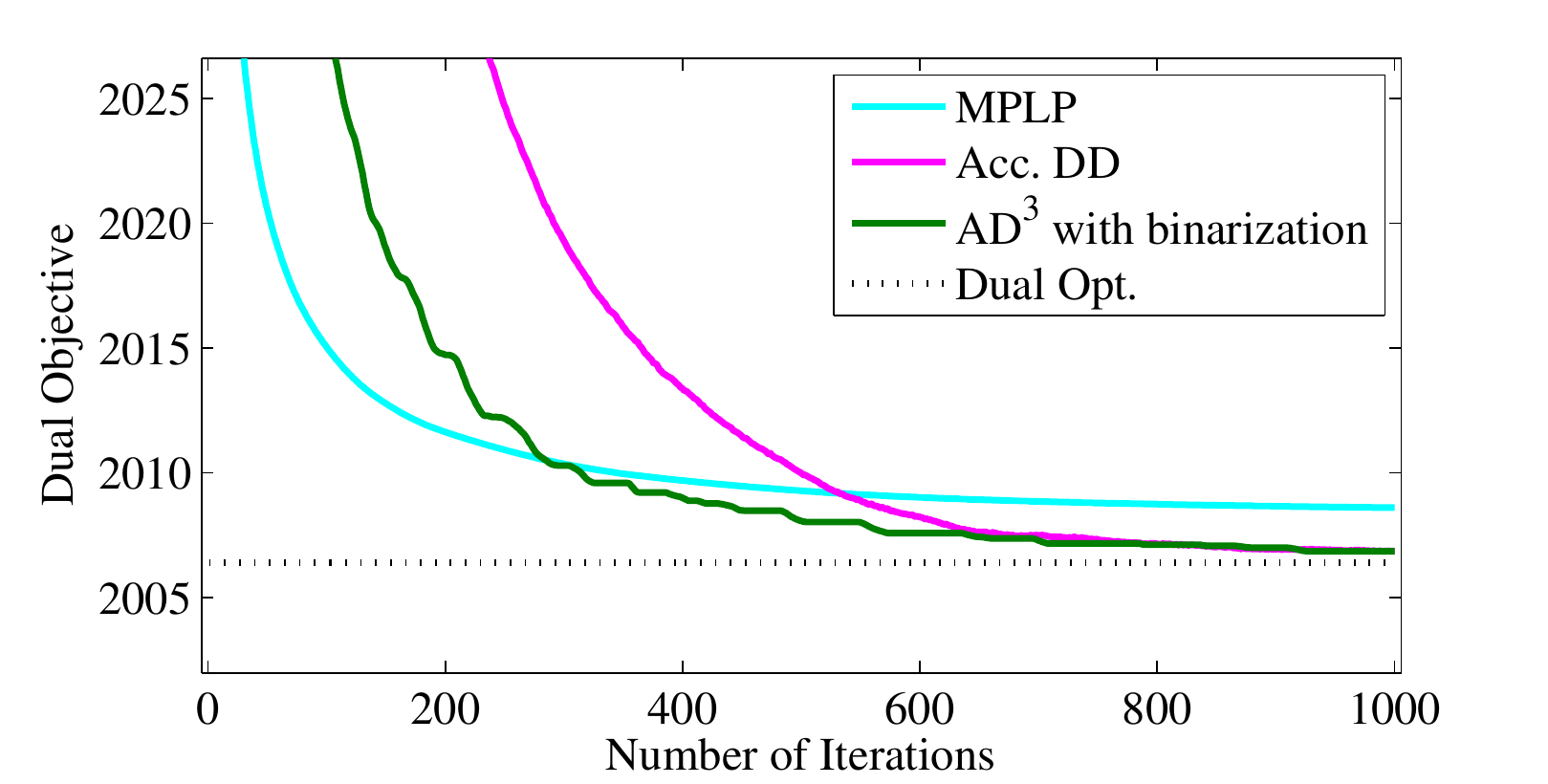}
 \includegraphics[scale=0.45]{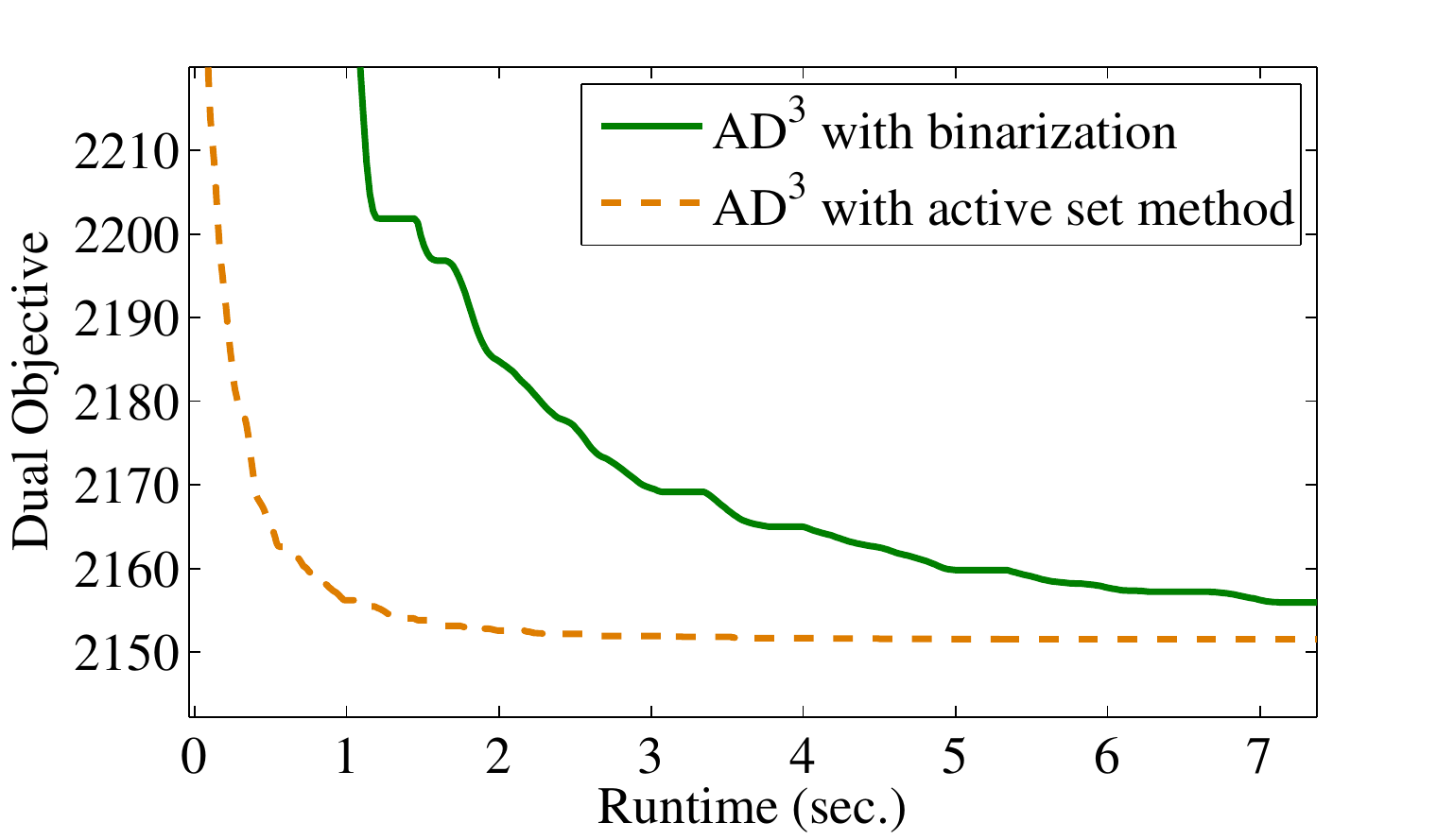}
\caption{Evolution of the dual objective in the experiments with
random $20\times 20$ Potts models with $8$-valued nodes, generated
as described in the main text. For the accelerated dual decomposition
algorithm, $\epsilon=1.0$. For AD$^3$, the left plot used
$\eta = 0.5$ and the right one $\eta=1$ (found to be the most favorable
choice in $\{0.01,0.1,1,10\}$). In the active set method,
the maximum number of inner iterations is $10$.}
 \label{fig:exp_multi_pairwisemrf}
 \end{figure}


\subsection{Protein Design}
We compare AD$^3$ with the MPLP implementation\footnote{Available at
\url{http://cs.nyu.edu/~dsontag/code}; that code includes a ``tightening''
procedure for retrieving the exact MAP, which we don't use, since we are interested in
the LP-MAP relaxation (which is what AD$^3$ addresses).} of \citet{Sontag2008} in the
benchmark protein design problems\footnote{Available at \url{http://www.jmlr.org/papers/volume7/yanover06a/}.}
of \citet{Yanover2006}. In these problems, the input is a three-dimensional shape,
and the goal is to find the most stable sequence of amino acids in that shape.
The problems can be represented as a pairwise factor graphs, whose variables correspond to
the identity of amino acids and rotamer configurations, thus having hundreds of possible states.
Fig.~\ref{fig:exp_protein} plots the evolution of the dual objective over runtime,
for two of the largest problem instances, \emph{i.e.}, those with 3167 (1fbo) and 1163
(1kw4) factors. These plots are representative of the typical performance obtained in
other instances. In both cases, MPLP steeply decreases the objective at early iterations,
but then reaches a plateau with no further significant improvement. AD$^3$ rapidly
surpasses MPLP in obtaining a better dual objective.
Finally, observe that although earlier iterations of AD$^3$ take longer than those
of MPLP, this cost is amortized in later iterations, by warm-starting the active set method.


\begin{figure}[h]
\centering
 \includegraphics[scale=0.48]{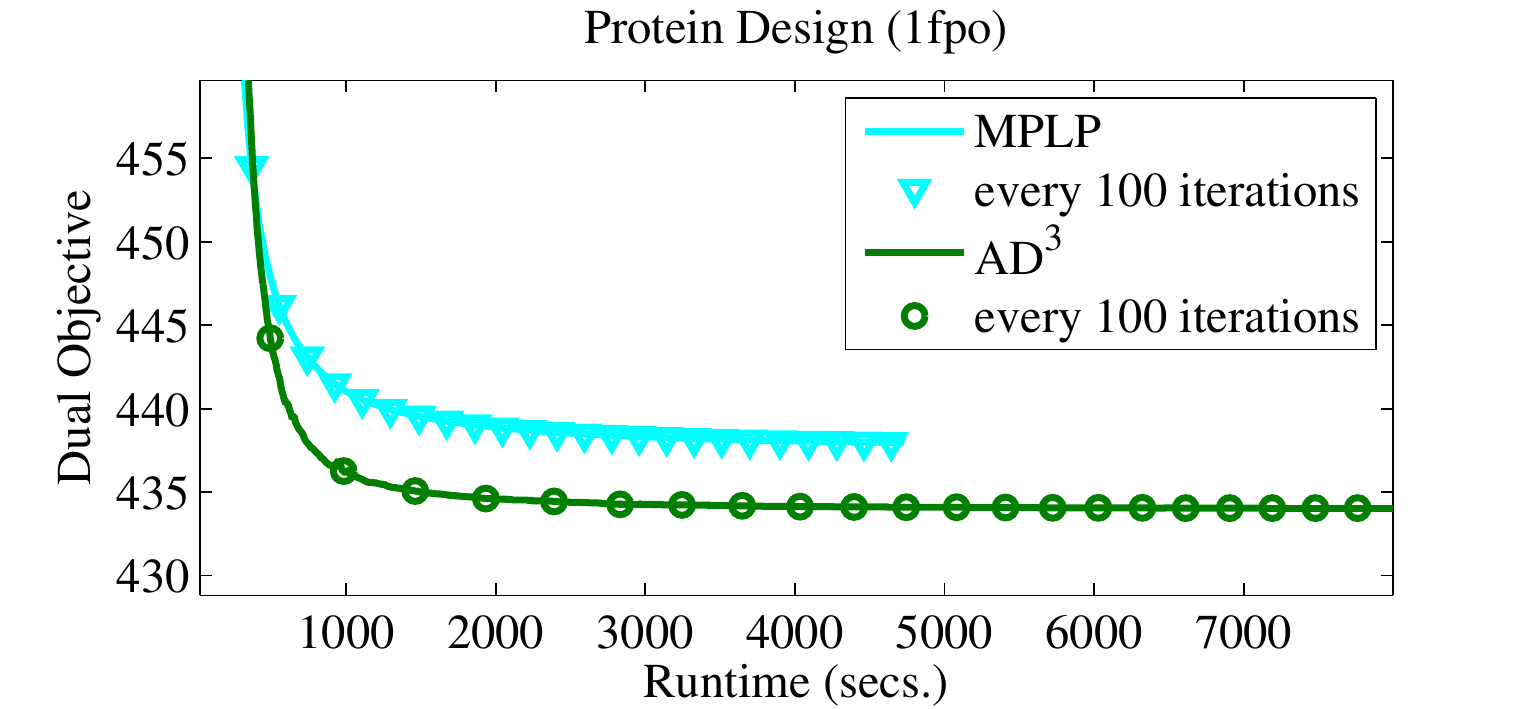}
 \includegraphics[scale=0.48]{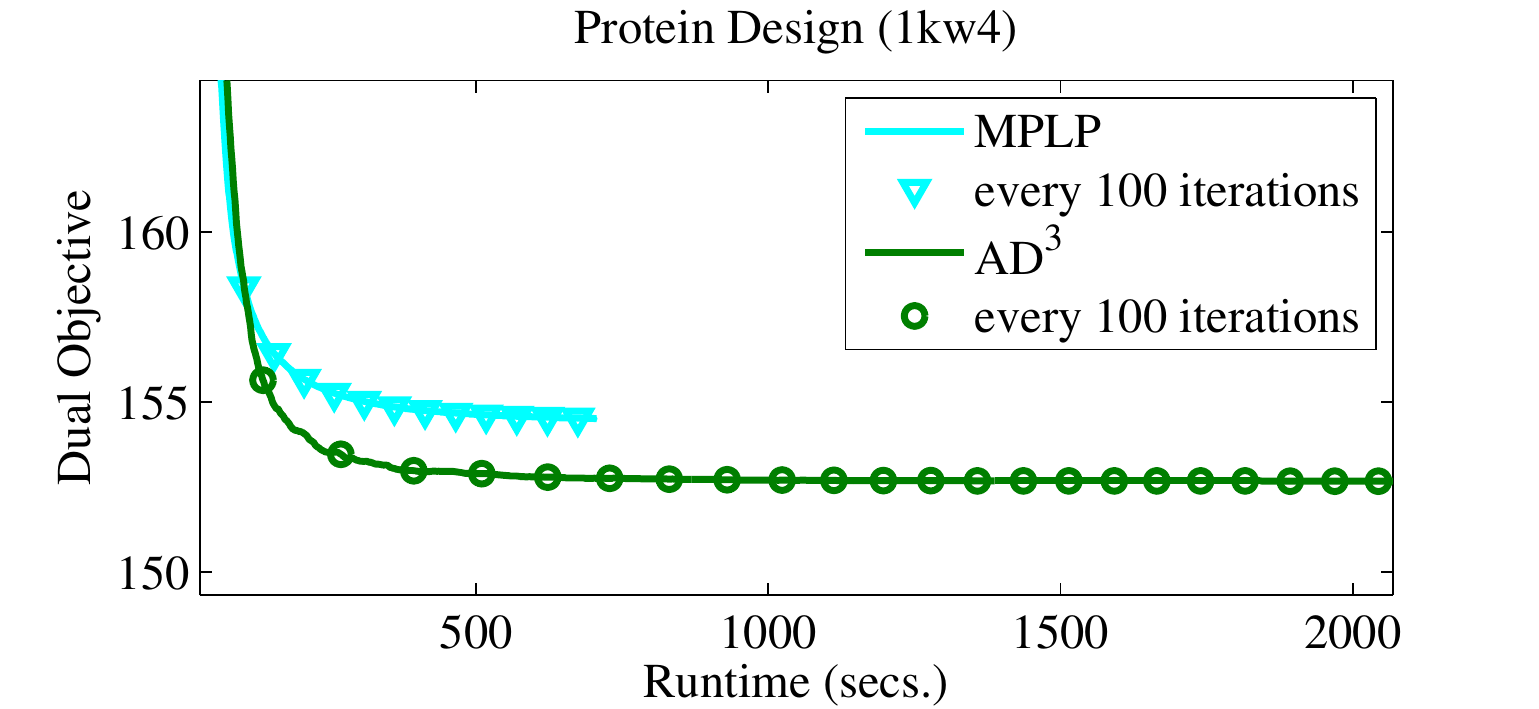}
\caption{Protein design experiments (see main text for details).
In AD$^3$, $\eta$ is adjusted as proposed by \citet{Boyd2011},
initialized at $\eta = 1.0$ and the subproblems are solved
by the proposed active set method, which shows better performance
that the binarization-based version. Although the plots are
with respect to runtime, they also
indicate iteration counts.
}
 \label{fig:exp_protein}
\end{figure}

\subsection{Frame-Semantic Parsing}\label{sec:ad3_framenet}

\begin{figure}[t]
\centering{\includegraphics[scale=0.45]{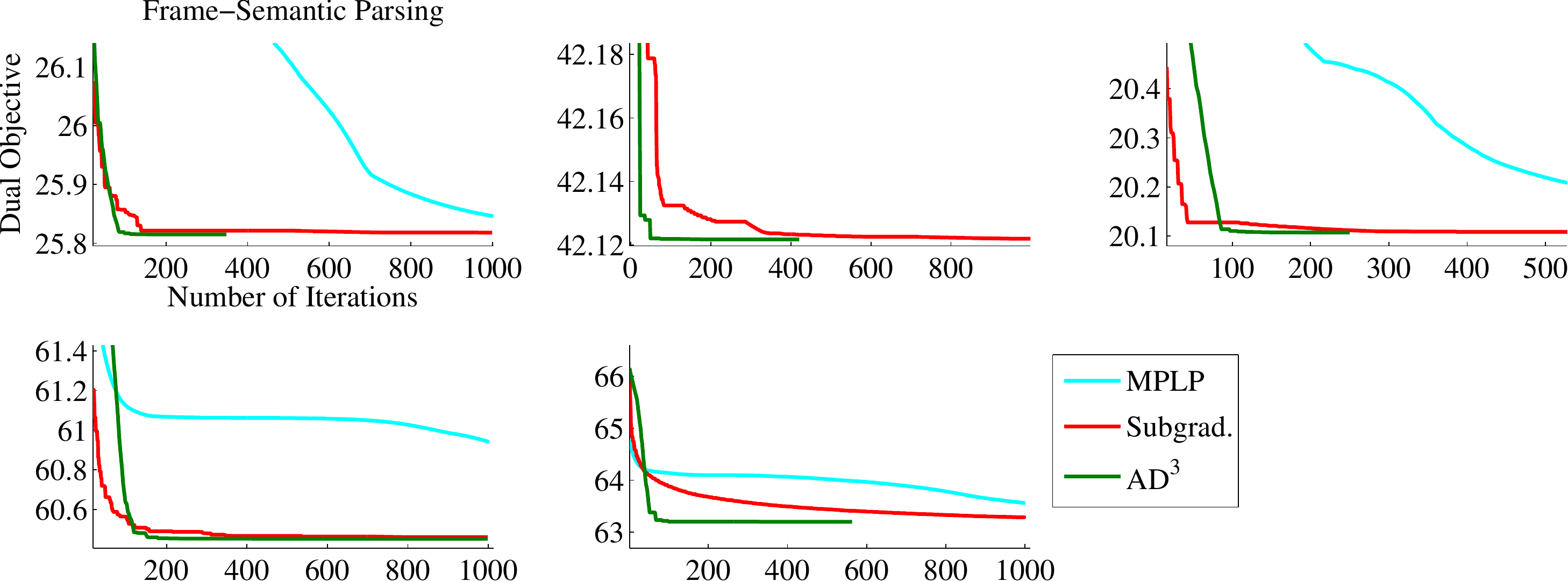}}
\caption{Experiments in five frame-semantic parsing problems
\citep[Section~5.5]{Dipanjan2012Thesis}.
The projected subgradient uses $\eta_t = \eta_0/t$, with
$\eta_0 = 1.0$ (found to be the best choice for all examples).
In AD$^3$, $\eta$ is adjusted as proposed by \citet{Boyd2011},
initialized at $\eta = 1.0$.}
\label{fig:exp_srl}
\end{figure}

We now report experiments on a natural language processing task
involving logic constraints: \emph{frame-semantic parsing}, using the FrameNet lexicon \citep{Fillmore1976}.
The goal is to predict the set of arguments and roles for a predicate word in a sentence,
while respecting several constraints about the frames that can be evoked.
The resulting graphical models are binary constrained factor graphs with FOL
constraints (see \citet{Das2012StarSEM} for details about this task).
Fig.~\ref{fig:exp_srl} shows the
results of AD$^3$, MPLP, and projected subgradient on the five most difficult problems
(which have between $321$ and $884$ variables, and between $32$ and $59$ factors),
the ones in which the LP relaxation is not tight. Unlike MPLP and projected subgradient,
which did not converge after 1000 iterations, AD$^3$ achieves convergence in a few hundreds
of iterations for all but one example. Since these examples have a fractional LP-MAP solution,
we applied the branch-and-bound procedure described in Section~\ref{sec:ad3_bnb} to obtain the
exact MAP for these examples. The whole dataset contains 4,462 instances, which were parsed
by this exact variant of the AD$^3$ algorithm  in only $4.78$ seconds,
against $43.12$ seconds of CPLEX, a state-of-the-art commercial ILP solver.



\subsection{Dependency Parsing}\label{sec:exp_dependency_parsing}

The final set of experiments assesses the ability of AD$^3$ to
handle problems with structured factors. The task is \emph{dependency parsing} (illustrated
in the left part of Fig.~\ref{fig:example_nonproj}), an important problem in
natural language processing \citep{Eisner1996,McDonald2005b}, to which dual decomposition has
been recently applied \citep{Koo2010EMNLP}. We use an English dataset derived from the Penn
Treebank (PTB)\citep{Marcus1993}, converted to dependencies by applying the head rules of
\citet{Yamada2003}; we follow the common procedure of training in
sections \S02--21 (39,832 sentences), using \S22 as validation data (1,700 sentences),
and testing on \S23 (2,416 sentences). We ran a part-of-speech tagger on the validation
and test splits, and devised a linear model using various features depending on
words, part-of-speech tags, and arc direction and length. Our features decompose over
the parts illustrated in the right part of Fig.~\ref{fig:example_nonproj}.
We consider two different models in our experiments: a \emph{second order model}
with scores for arcs, consecutive siblings, and grandparents; a \emph{full model}, which
also has scores for arbitrary siblings and head bigrams.

\begin{figure}[t]
\centering
\includegraphics[scale=0.65]{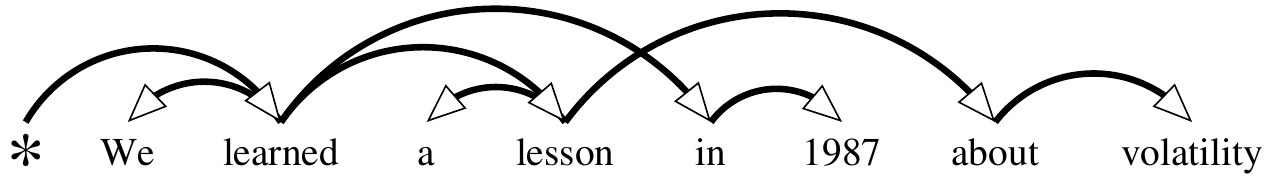}\quad
\includegraphics[scale=0.35]{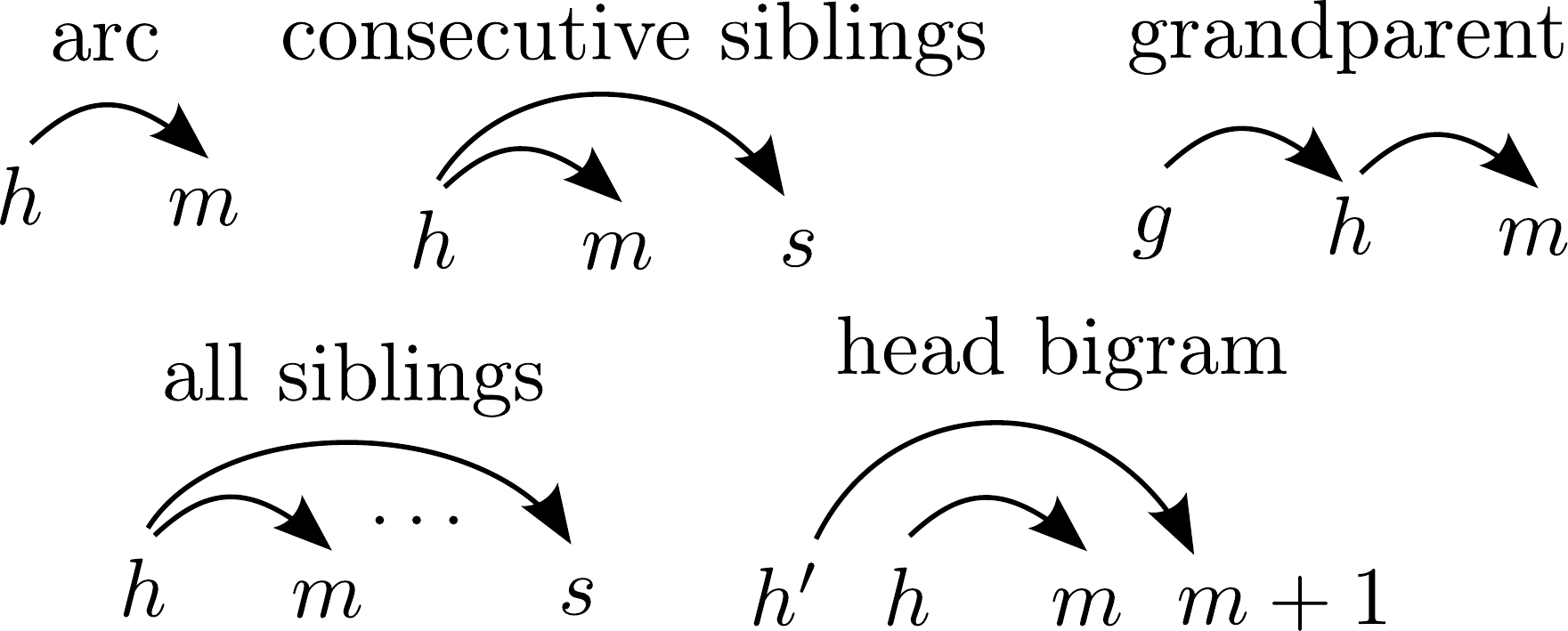}
\caption{Left: example of a sentence (input)
and its dependency parse tree (output to be predicted);
this is a directed spanning tree where each arc $(h,m)$
represent a syntactic relationships between a \emph{head} word
$h$ and the a \emph{modifier} word $m$. Right: the parts used in our models.
\emph{Arcs} are the basic parts: any dependency tree can be
``read out'' from its arcs.  \emph{Consecutive siblings} and \emph{grandparent} parts
introduce horizontal and vertical Markovization. 
We break the horizontal Markovianity via \emph{all siblings} parts (which look
at arbitrary pairs of siblings, not necessarily consecutive). Inspired by transition-based parsers,
we also adopt \emph{head bigram} parts,
which look at the heads attached to consecutive words.}
\label{fig:example_nonproj}
\end{figure}

If only scores for arcs were used,  the problem of obtaining a parse tree with maximal score could be
solved efficiently with a maximum directed spanning tree algorithm \citep{Chu1965,Edmonds1967,McDonald2005b};
the addition of any of the other scores makes the problem NP-hard \citep{McDonald2007}.
A factor graph representing the second order model, proposed by \citet{DSmith2008}
and \citet{Koo2010EMNLP}, contains binary variables representing the candidate arcs,
a hard-constraint factor imposing the tree constraint, and head automata factors
modeling the sequences of consecutive siblings and grandparents.
The full model has additional binary pairwise factors for each possible pair of
siblings (significantly increasing the number of factors), and a sequential
factor modeling the sequence of heads.%
\footnote{In previous work \citep{Martins2011bEMNLP},
we implemented a similar model with a more complex factor graph based on a
multi-commodity flow formulation, requiring only the FOL factors described in
Section~\ref{sec:ddadmm_hardconstraintfactors}. 
This resulted in a factor graph with many overlapping factors
(in the order of tens of thousands), for which
we have shown that the subgradient algorithm became too slow.
In the current paper, we consider a smaller graph with structured factors,
for which the subgradient algorithm is a stronger competitor.
This smaller graph was also highly beneficial for AD$^3$,
which along with the active set method led to a  significant speed-up.} %
We have shown in previous work \citep{Martins2011ICML} that MPLP and the accelerated dual decomposition
methods are not competitive with AD$^3$ for this task,
hence we leave them out of this experiment.

\begin{figure}[tb]
\centering
\includegraphics[scale=0.55]{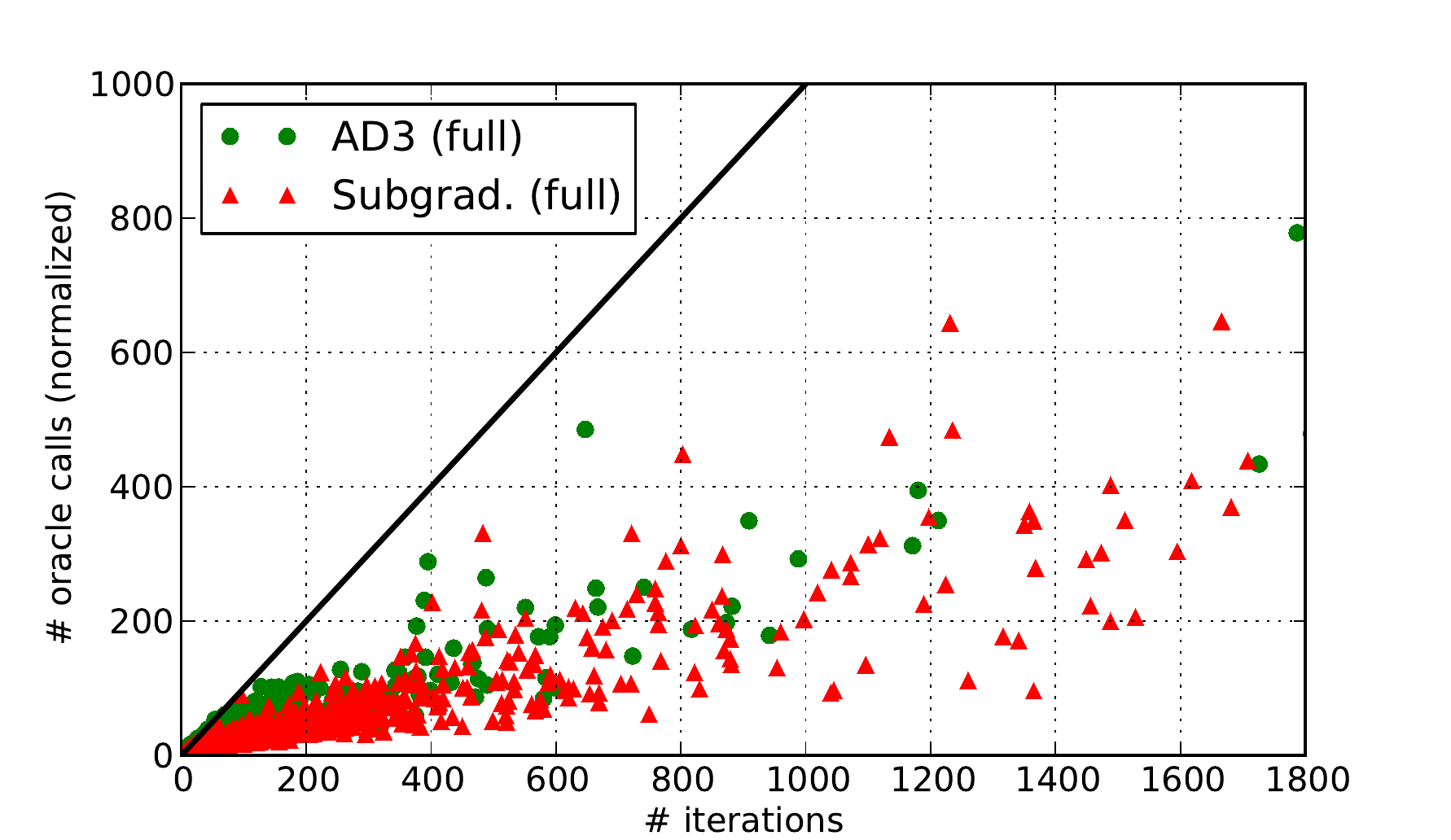}%
\caption{Number of calls to {\sc ComputeMAP} for AD$^3$ and
projected subgradient, as a function of the number of iterations. The number of calls is
normalized by dividing by the number of factors: in the subgradient method, this number
would equal the number of iterations if there was no caching (black line);
each iteration of AD$^3$ runs $10$ iterations of the active set method, thus
without caching or warm-starting the normalized number of calls would be ten times the
number of AD$^3$ iterations. Yet, it is clear that both algorithms make significantly fewer calls.
Remarkably, after just a few iterations, the number of calls made by the AD$^3$ and the
subgradient algorithms are comparable, which means that the number of active set iterations is
quickly amortized during the execution of AD$^3$.}
\label{fig:oracle_calls}
\end{figure}

Fig.~\ref{fig:oracle_calls} illustrates the remarkable speed-ups that
the caching and warm-starting  procedures bring to both the AD$^3$ and projected subgradient algorithms.
A similar conclusion was obtained by \citet{Koo2010EMNLP} for projected subgradient and
by \citet{Martins2011bEMNLP} for AD$^3$ in a different factor graph.
Fig.~\ref{fig:deppars_experiments} shows average runtimes for both algorithms,
as a function of the sentence length, and plots the percentage of instances for
which the exact solution was obtained, along with a certificate of optimality.
For the second-order model, AD$^3$ was able to solve all the instances to optimality,
and in 98.2\% of the cases, the LP-MAP was exact. For the full model, AD$^3$ solved 99.8\%
of the instances to optimality, being exact in 96.5\% of the cases.
For the second order model, we obtained in the test set (PTB \S23) a parsing speed of
1200 tokens per second and an unlabeled attachment score of 92.48\% (fraction of correct
dependency attachments excluding punctuation). For the full model, we obtained a speed of ~900
tokens per second and a score of 92.62\%. All speeds were measured in a desktop PC with
Intel Core i7 CPU 3.4 GHz and 8GB RAM. The parser is publicly available as an open-source project
in \url{http://www.ark.cs.cmu.edu/TurboParser}.

\begin{figure}[tb]
\centering
\includegraphics[scale=0.45]{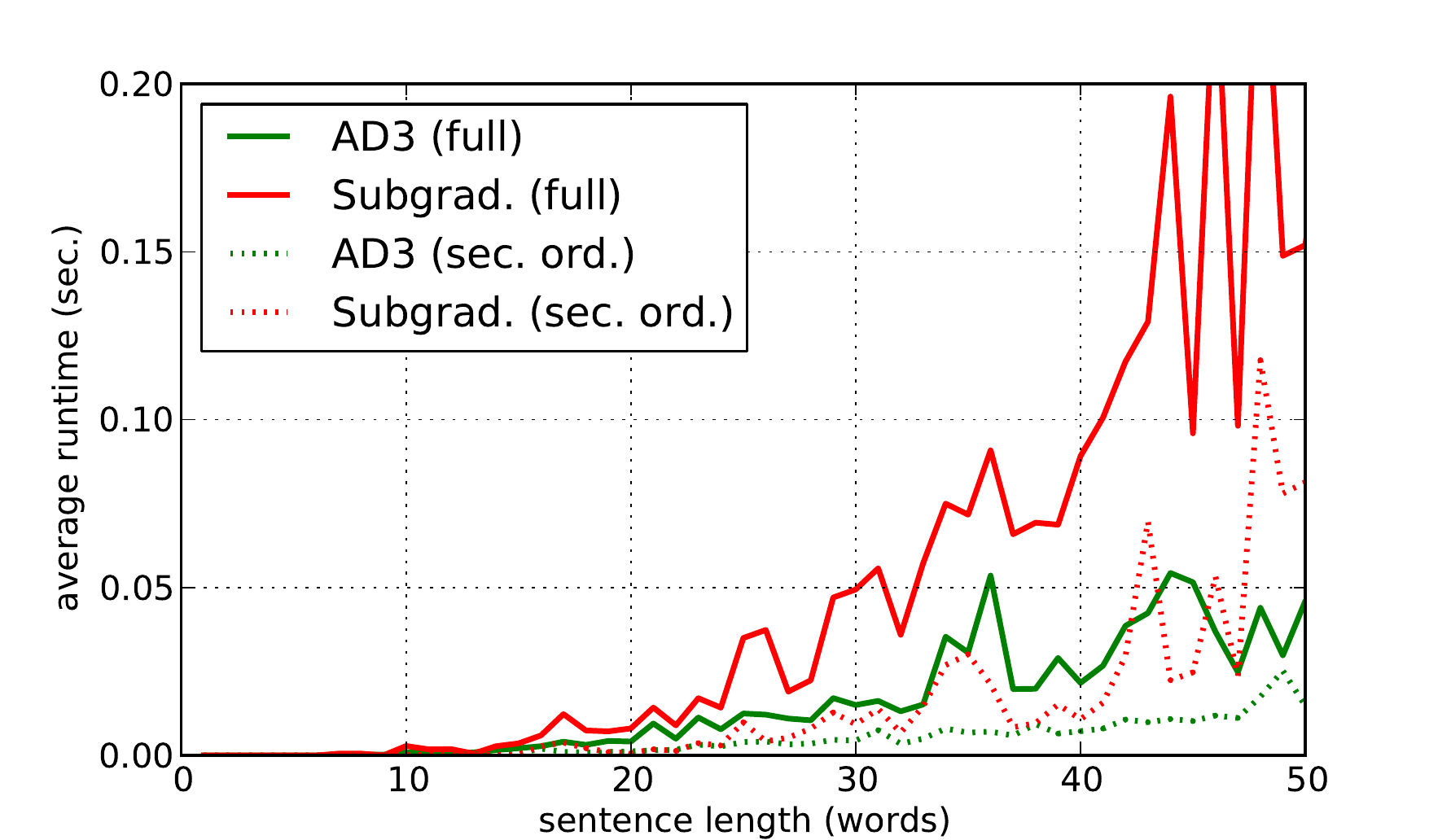}%
\includegraphics[scale=0.45]{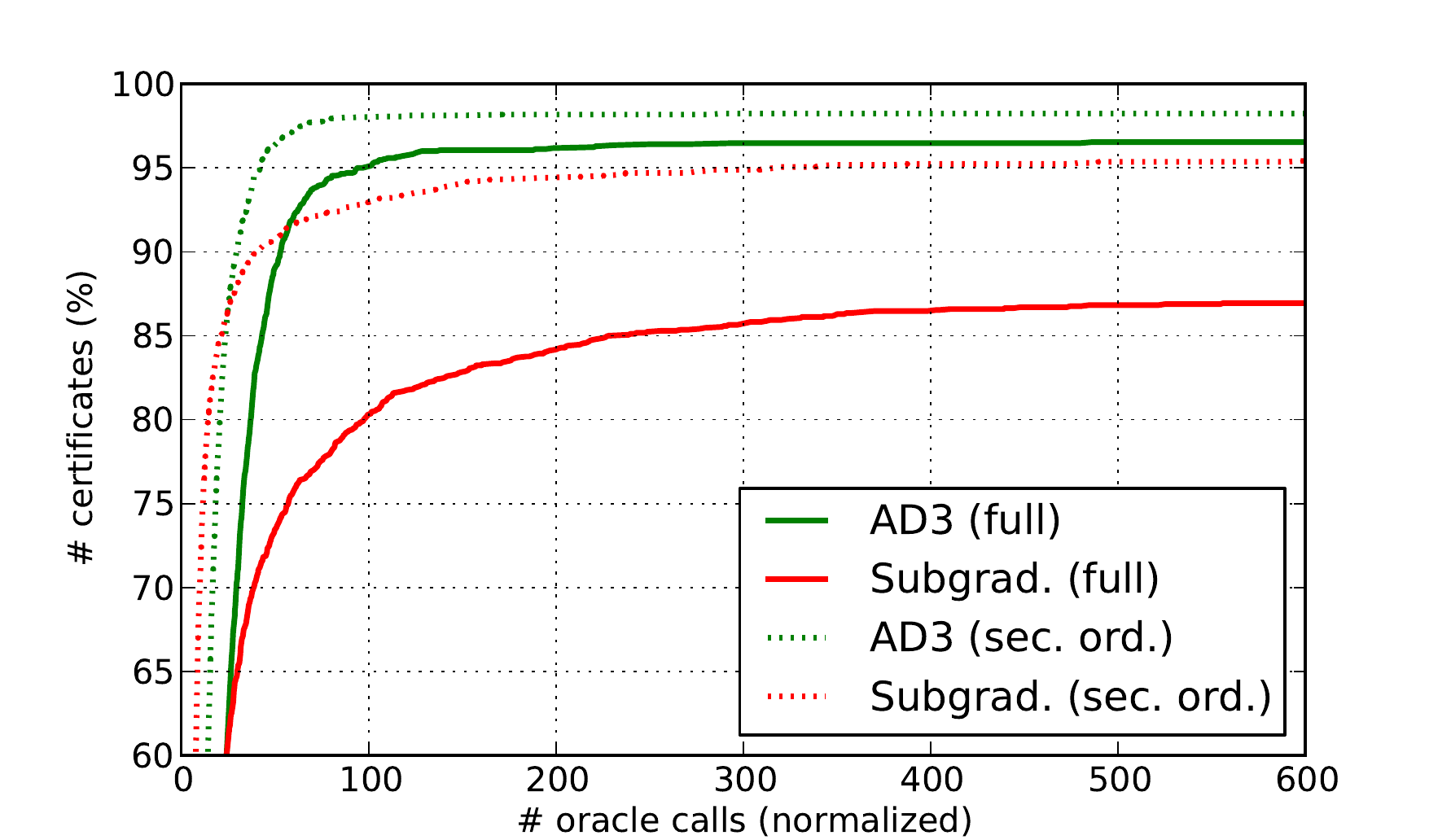}%
\caption{Left: average runtime in PTB \S22, as a function of sentence length.
Right: percentage of instances, as a function of the normalized number of {\sc ComputeMAP}
calls (see the caption of Fig.~\ref{fig:oracle_calls}), for which
the exact solution was obtained along with a certificate of optimality.
The maximum number of iterations is 2000 for both methods.}
\label{fig:deppars_experiments}
\end{figure}

\section{Recent Related Work}\label{sec:relatedwork}
During the preparation of this paper, and following our earlier work \citep{Martins2010bNIPS,Martins2011ICML},
a few related methods have appeared. \citet{Meshi2011} also applied ADMM to MAP inference in graphical models, 
although addressing the dual problem (the one underlying the MPLP algorithm) rather than the primal.
\citet{Yedidia2011} proposed the \emph{divide-and-concur algorithm}
for LDPC (low-density parity check) decoding,  which shares aspects of  AD$^3$, and
can be seen as an instance of non-convex ADMM.
More recently, \citet{Barman2011}
proposed an algorithm analogous to AD$^3$ for the same LDPC decoding problem;
in that algorithm, the subproblems correspond to projections onto the parity polytope,
for which they have derived an efficient $O(L \log L)$ algorithm,
where $L$ is the number of parity check bits. 

\section{Conclusions}\label{sec:conclusions}

We introduced AD$^3$, a new LP-MAP inference algorithm\footnote{Available at \url{http://www.ark.cs.cmu.edu/AD3}}
based on the alternating directions method of multipliers (ADMM) 
\citep{Glowinski75,Gabay1976}. 

AD$^3$ enjoys the modularity of dual decomposition methods,
but achieves faster consensus,
by penalizing, for each subproblem,
deviations from the current global solution.

Blending older and newer results for 
ADMM \citep{Glowinski1989,Eckstein1992,He2011,Wang2012ICML},
we showed that AD$^3$ converges to an $\epsilon$-accurate
solution with an iteration bound of $O(1/\epsilon)$. 

AD$^3$ can handle factor graphs with hard constraints in first-order logic, 
using efficient procedures for projecting onto the marginal polytopes of
the corresponding hard constraint factors. This opens the door for 
using AD$^3$ in problems with declarative constraints \citep{Roth2004,Richardson2006}.
Up to a logarithmic term, the asymptotic cost of projecting onto those polytopes is 
the same as that of message passing. A closed-form solution of the AD$^3$ subproblem 
was also derived for pairwise binary factors.

We introduced a new \emph{active set method} for solving the AD$^3$ subproblems
for arbitrary factors. This method requires only a local MAP oracle, as the  
projected subgradient method (see Section~\ref{sec:dualdecomposition}).
The active set method is particulary suitable for these problems,
since it can take advantage of warm starting and it deals well with sparse 
solutions---which are guaranteed by Proposition~\ref{prop:ad3_subproblem_primal_sparse}.
We also show how AD$^3$ can be wrapped in a branch-and-bound procedure to retrieve the 
exact MAP.

Experiments with synthetic and real-world datasets have shown that
AD$^3$ is able to solve the LP-MAP problem more efficiently than other methods
for a variety of problems, including MAP inference in Ising and Potts models, 
protein design,  frame-semantic parsing, and dependency parsing.

Our contributions open  several directions for future research.
One possible extension is to replace the quadratic (Euclidean) 
penalty of ADMM by a general Bregman divergence (notice that an entropic penalty would \emph{not}
lead to the same subproblems as in \citet{Jojic2010}). 
Although extensions of ADMM to Bregman penalties have been considered in the literature,
to the best of our knowledge, convergence has only been shown for 
quadratic penalties. 
The convergence proofs, however, can be trivially extended
to Mahalanobis distances, since they correspond to
an affine transformation of the subspace defined by the equality constraints
of Eq.~\ref{eq:convexproblemequalityconstraintschanges}.
Simple operations, such as scaling these constraints, do not affect the algorithms 
that are used to solve the subproblems, thus AD$^3$ can be generalized by 
including scaling parameters.


Since the AD$^3$ subproblems can be solved in parallel, significant speed-ups 
may be obtained in multi-core architectures or using GPU programming.
This has been shown to be very useful for large-scale
message-passing inference in graphical models
\citep{Low2010}. 

The branch-and-bound algorithm for obtaining the exact MAP deserves
further experimental study. An advantage of AD$^3$ is its ability 
to quickly produce sharp upper bounds, useful for embedding in a branch-and-bound
procedure. For many problems, there are effective rounding procedures
that can also produce lower bounds, which can be exploited for
guiding the search. There are also alternatives to branch-and-bound,
such as tightening procedures \citep{Sontag2008}, which progressively add 
larger factors to decrease the duality gap. The variant of AD$^3$ with 
the active set method can be used to handle these larger factors.


\paragraph*{Acknowledgments.}
{A.~M.~was supported by a FCT/ICTI grant through the
CMU-Portugal Program, and also by Priberam.
This work was partially supported by the FET programme (EU FP7), under the SIMBAD project (contract 213250),
and by a FCT grant PTDC/EEA-TEL/72572/2006.
N.~S.~was supported by NSF CAREER IIS-1054319.
E.~X.~was supported by AFOSR FA9550010247, ONR N000140910758, NSF CAREER DBI-0546594, NSF IIS-0713379, and an Alfred P. Sloan Fellowship.}

\appendix

\section{Proof of Convergence Rate of AD$^3$}\label{chap:ad3rate}

In this appendix, we show the $O(1/\epsilon)$ convergence bound of the ADMM algorithm. 
We use a recent result established by \citet{Wang2012ICML} 
regarding convergence in a variational setting, 
from which we derive the convergence of ADMM in the dual objective.
We then consider the special case of AD$^3$, interpreting the constants
in the bound in terms of properties of the graphical model.

We start with the following proposition,
which states the variational inequality associated with the Lagrangian saddle point 
problem associated with \eqref{eq:convexproblemequalityconstraints},
\begin{equation}\label{eq:lagrangian-saddlepoint}
\min_{\vectsymb{\lambda} \in \Lambda} \max_{\vectsymb{q} \in \sett{Q}, \vectsymb{p} \in \sett{P}} L(\vectsymb{q},\vectsymb{p},\vectsymb{\lambda}),
\end{equation}
where 
$L(\vectsymb{q},\vectsymb{p},\vectsymb{\lambda}) := f_1(\vectsymb{q}) + f_2(\vectsymb{p}) + \DP{\vectsymb{\lambda}}{(\matr{A}\vectsymb{q} + \matr{B}\vectsymb{p} - \vectsymb{c})}$
is the standard Lagrangian, 
and
$$\Lambda := \{\vectsymb{\lambda} \,\,|\,\, \max_{\vectsymb{q} \in \sett{Q}, \vectsymb{p} \in \sett{P}} L(\vectsymb{q},\vectsymb{p},\vectsymb{\lambda}) < \infty\}.$$

\begin{proposition}[Variational inequality]\label{prop:admm_variational}
Let $\sett{W} := \sett{Q} \times \sett{P} \times \Lambda$.
Given $\vectsymb{w} = (\vectsymb{q}, \vectsymb{p}, \vectsymb{\lambda}) \in \sett{W}$,
define $h(\vectsymb{w}) := f_1(\vectsymb{q}) + f_2(\vectsymb{p})$
and $F(\vectsymb{w}) := (\matr{A}^{\top}\vectsymb{\lambda}, \matr{B}^{\top}\vectsymb{\lambda}, -(\matr{A}\vectsymb{q} + \matr{B}\vectsymb{p} - \vectsymb{c}))$.
Then, $\vectsymb{w}^* := (\vectsymb{q}^*,\vectsymb{p}^*,\vectsymb{\lambda}^*) \in \sett{W}$ is a primal-dual solution of
Eq.~\ref{eq:lagrangian-saddlepoint} if and only if:
\begin{equation}\label{eq:ddadmm_vi_main}
\forall \vectsymb{w} \in \sett{W}, \quad h(\vectsymb{w}) - h(\vectsymb{w}^*) + \DP{(\vectsymb{w} - \vectsymb{w}^*)}{F(\vectsymb{w}^*)} \le 0.
\end{equation}
\end{proposition}
\begin{proof}
Assume $\vectsymb{w}^*$ is a primal-dual solution of
Eq.~\ref{eq:lagrangian-saddlepoint}. Then,
from the saddle point conditions, we have for every $\vectsymb{w} := (\vectsymb{q},\vectsymb{p},\vectsymb{\lambda}) \in \sett{W}$:
\begin{equation}
L(\vectsymb{q},\vectsymb{p},\vectsymb{\lambda}^*) \le L(\vectsymb{q}^*,\vectsymb{p}^*,\vectsymb{\lambda}^*) \le
L(\vectsymb{q}^*,\vectsymb{p}^*,\vectsymb{\lambda}).
\end{equation}
Hence:
\begin{eqnarray}
0 &\ge& L(\vectsymb{q},\vectsymb{p},\vectsymb{\lambda}^*) - L(\vectsymb{q}^*,\vectsymb{p}^*,\vectsymb{\lambda})\nonumber\\
&=& f_1(\vectsymb{q}) + f_2(\vectsymb{p}) + \DP{\vectsymb{\lambda}^*}{(\matr{A}\vectsymb{q} + \matr{B}\vectsymb{p} - \vectsymb{c})}
- f_1(\vectsymb{q}^*) - f_2(\vectsymb{p}^*) - \DP{\vectsymb{\lambda}}{(\matr{A}\vectsymb{q}^* + \matr{B}\vectsymb{p}^* - \vectsymb{c})}\nonumber\\
&=& h(\vectsymb{w}) - h(\vectsymb{w}^*) + \DP{\vectsymb{q}}{\matr{A}^{\top}\vectsymb{\lambda}^*} + \DP{\vectsymb{p}}{\matr{B}^{\top}\vectsymb{\lambda}^*}
- \DP{(\vectsymb{\lambda} - \vectsymb{\lambda}^*)}{(\matr{A}\vectsymb{q}^* + \matr{B}\vectsymb{p}^* - \vectsymb{c})}
- \DP{\vectsymb{\lambda}^*}{(\matr{A}\vectsymb{q}^* + \matr{B}\vectsymb{p}^*)}\nonumber\\
&=& h(\vectsymb{w}) - h(\vectsymb{w}^*) + \DP{(\vectsymb{w} - \vectsymb{w}^*)}{F(\vectsymb{w}^*)}.
\end{eqnarray}
Conversely, let $\vectsymb{w}^*$ satisfy Eq.~\ref{eq:ddadmm_vi_main}. Taking
$\vectsymb{w} = (\vectsymb{q}^*, \vectsymb{p}^*, \vectsymb{\lambda})$, we obtain
$L(\vectsymb{q}^*,\vectsymb{p}^*,\vectsymb{\lambda}^*) \le L(\vectsymb{q}^*,\vectsymb{p}^*,\vectsymb{\lambda})$.
Taking
$\vectsymb{w} = (\vectsymb{q}, \vectsymb{p}, \vectsymb{\lambda}^*)$, we obtain
$L(\vectsymb{q},\vectsymb{p},\vectsymb{\lambda}^*) \le L(\vectsymb{q}^*,\vectsymb{p}^*,\vectsymb{\lambda}^*)$.
Hence $(\vectsymb{q}^*,\vectsymb{p}^*,\vectsymb{\lambda}^*)$ is a saddle point,
and therefore a primal-dual solution.
\end{proof}

\bigskip

The next result, due to \citet{Wang2012ICML}
and 
related to previous work by \citet{He2011},
concerns the convergence rate of ADMM in terms of the variational
inequality stated above. 

\begin{proposition}[Variational convergence rate]\label{prop:convergence_admm_vi}
Assume the conditions in Proposition~\ref{prop:ddadmm1}. 
Let $\bar{\vectsymb{w}}_T = \frac{1}{T} \sum_{t=1}^T \vectsymb{w}^t$,
where $\vectsymb{w}^t := (\vectsymb{q}^t,\vectsymb{p}^t,\vectsymb{\lambda}^t)$ are the
ADMM iterates with  
$\vectsymb{\lambda}^{0} = \vect{0}$. Then, after $T$ iterations:
\begin{equation}\label{eq:convergence_admm_vi_constant}
\forall \vectsymb{w} \in \sett{W}, \quad h(\vectsymb{w}) - h(\bar{\vectsymb{w}}_T) + \DP{(\vectsymb{w} - \bar{\vectsymb{w}}_T)}{F(\bar{\vectsymb{w}}_T)} \le
\frac{C}{T},
\end{equation}
where
$C = \frac{\eta}{2}
\|\matr{A}\vectsymb{q} + \matr{B}\vectsymb{p}^{0} - \vectsymb{c}\|^2
+
\frac{1}{2\eta}
\|\vectsymb{\lambda}\|^2$ is independent of $T$.
\end{proposition}
\begin{proof}
From the variational inequality associated with the $\vectsymb{q}$-update \eqref{eq:admm_u_update} we have for every $\vectsymb{q} \in \sett{Q}$%
\footnote{For a maximization problem of the form 
$\max_{\vectsymb{x} \in \sett{X}} f(\vectsymb{x})$,
where $f$ is concave and differentiable and $\sett{X}$ is a convex set,
a point $\vectsymb{x}^* \in \sett{X}$ is a maximizer if and only if it satisfies the variational 
inequality $\nabla f(\vectsymb{x}^*)^{\top} (\vectsymb{x}-\vectsymb{x}^*) \le 0$
for all $\vectsymb{x} \in \sett{X}$. See \citealt{Facchinei2003} for a comprehensive overview
of variational inequalities.} %
\begin{eqnarray}\label{eq:vi_u}
0 &\ge&  \DP{\nabla_{\vectsymb{q}} L_{\eta} (\vectsymb{q}^{t+1}, \vectsymb{p}^{t}, \vectsymb{\lambda}^{t})}{(\vectsymb{q} - \vectsymb{q}^{t+1})}\nonumber\\
&=& \DP{\nabla f_1(\vectsymb{q}^{t+1})}{(\vectsymb{q} - \vectsymb{q}^{t+1})} +
\DP{(\vectsymb{q} - \vectsymb{q}^{t+1})}{\matr{A}^{\top}(\vectsymb{\lambda}^t - \eta(\matr{A}\vectsymb{q}^{t+1} + \matr{B}\vectsymb{p}^{t} - \vectsymb{c}))}\nonumber\\
&\ge^{(i)}& f_1(\vectsymb{q}) - f_1(\vectsymb{q}^{t+1}) +
\DP{(\vectsymb{q} - \vectsymb{q}^{t+1})}{\matr{A}^{\top}(\vectsymb{\lambda}^t - \eta(\matr{A}\vectsymb{q}^{t+1} + \matr{B}\vectsymb{p}^{t} - \vectsymb{c}))}\nonumber\\
&=^{(ii)}& f_1(\vectsymb{q}) - f_1(\vectsymb{q}^{t+1}) +
\DP{(\vectsymb{q} - \vectsymb{q}^{t+1})}{\matr{A}^{\top}\vectsymb{\lambda}^{t+1}}
- \eta(\DP{\matr{A}(\vectsymb{q} - \vectsymb{q}^{t+1}))}{\matr{B}(\vectsymb{p}^{t} - \vectsymb{p}^{t+1})},
\end{eqnarray}
where in (i) we have used the concavity of $f_1$, and in (ii) we used Eq.~\ref{eq:almm_lambda_update}  for the $\vectsymb{\lambda}$-updates.
Similarly, the variational inequality associated with the $\vectsymb{p}$-updates
\eqref{eq:admm_v_update} yields, for every $\vectsymb{p} \in \sett{P}$:
\begin{eqnarray}\label{eq:vi_v}
0 &\ge&  \DP{\nabla_{\vectsymb{p}} L_{\eta} (\vectsymb{q}^{t+1}, \vectsymb{p}^{t+1}, \vectsymb{\lambda}^{t})}{(\vectsymb{p} - \vectsymb{p}^{t+1})}\nonumber\\
&=& \DP{\nabla f_2(\vectsymb{p}^{t+1})}{(\vectsymb{p} - \vectsymb{p}^{t+1})} +
\DP{(\vectsymb{p} - \vectsymb{p}^{t+1})}{\matr{B}^{\top}(\vectsymb{\lambda}^t - \eta(\matr{A}\vectsymb{q}^{t+1} + \matr{B}\vectsymb{p}^{t+1} - \vectsymb{c}))}\nonumber\\
&\ge^{(i)}& f_2(\vectsymb{p}) - f_2(\vectsymb{p}^{t+1}) +
\DP{(\vectsymb{p} - \vectsymb{p}^{t+1})}{\matr{B}^{\top}\vectsymb{\lambda}^{t+1}},
\end{eqnarray}
where in (i) we have used the concavity of $f_2$.
Summing \eqref{eq:vi_u} and \eqref{eq:vi_v},
and noting again that 
$\vectsymb{\lambda}^{t+1} = \vectsymb{\lambda}^{t} - \eta(\matr{A}\vectsymb{q}^{t+1} + 
\matr{B}\vectsymb{p}^{t+1} - \vectsymb{c})$,
we obtain, for every $\vectsymb{w} \in \sett{W}$,
\begin{eqnarray}\label{eq:bound_admm_001}
\lefteqn{h(\vectsymb{w}^{t+1}) - h(\vectsymb{w}) +
\DP{(\vectsymb{w}^{t+1} - \vectsymb{w})}{F(\vectsymb{w}^{t+1})}} \nonumber\\
&\ge&
-\eta\DP{\matr{A}(\vectsymb{q} - \vectsymb{q}^{t+1})}{\matr{B}(\vectsymb{p}^{t} - \vectsymb{p}^{t+1})}
-
\eta^{-1}\DP{(\vectsymb{\lambda} - \vectsymb{\lambda}^{t+1})}{(\vectsymb{\lambda}^{t+1} - \vectsymb{\lambda}^{t})}.
\end{eqnarray}
We next rewrite the two terms in the right hand side.
We have
\begin{eqnarray}\label{eq:bound_admm_001_a}
\eta\DP{\matr{A}(\vectsymb{q} - \vectsymb{q}^{t+1})}{\matr{B}(\vectsymb{p}^{t} - \vectsymb{p}^{t+1})}
&=& \frac{\eta}{2} \left(
\|\matr{A}\vectsymb{q} + \matr{B}\vectsymb{p}^{t} - \vectsymb{c}\|^2 -
\|\matr{A}\vectsymb{q} + \matr{B}\vectsymb{p}^{t+1} - \vectsymb{c}\|^2 \right. \nonumber\\
&+&
\left. \|\matr{A}\vectsymb{q}^{t+1} + \matr{B}\vectsymb{p}^{t+1} - \vectsymb{c}\|^2 -
\|\matr{A}\vectsymb{q}^{t+1} + \matr{B}\vectsymb{p}^{t} - \vectsymb{c}\|^2\right)
\end{eqnarray}
and
\begin{eqnarray}\label{eq:bound_admm_001_b}
\eta^{-1}\DP{(\vectsymb{\lambda} - \vectsymb{\lambda}^{t+1})}{(\vectsymb{\lambda}^{t+1} - \vectsymb{\lambda}^{t})} &=& 
\frac{1}{2\eta} \left(
\|\vectsymb{\lambda} - \vectsymb{\lambda}^{t}\|^2 -
\|\vectsymb{\lambda} - \vectsymb{\lambda}^{t+1}\|^2 -
\|\vectsymb{\lambda}^t - \vectsymb{\lambda}^{t+1}\|^2 \right).\label{eq:bound_admm_001_b}
\end{eqnarray}
Summing \eqref{eq:bound_admm_001} over $t$ and noting that
$\eta^{-1}\|\vectsymb{\lambda}^t - \vectsymb{\lambda}^{t+1}\|^2 =
\eta\|\matr{A}\vectsymb{q}^{t+1} + \matr{B}\vectsymb{p}^{t+1} - \vectsymb{c}\|^2$,
we obtain by the telescoping sum property:
\begin{eqnarray}\label{eq:bound_admm_002}
\lefteqn{\sum_{t=0}^{T-1} \left( h(\vectsymb{w}^{t+1}) - h(\vectsymb{w}) +
\DP{(\vectsymb{w}^{t+1} - \vectsymb{w})}{F(\vectsymb{w}^{t+1})} \right)} \nonumber\\
&\ge&
-\frac{\eta}{2} \left(
\|\matr{A}\vectsymb{q} + \matr{B}\vectsymb{p}^{0} - \vectsymb{c}\|^2 -
\|\matr{A}\vectsymb{q} + \matr{B}\vectsymb{p}^{T} - \vectsymb{c}\|^2 -
\sum_{t=0}^{T-1} \|\matr{A}\vectsymb{q}^{t+1} + \matr{B}\vectsymb{p}^{t} - \vectsymb{c}\|^2\right)\nonumber\\ &&
-
\frac{1}{2\eta} \left(
\|\vectsymb{\lambda} - \vectsymb{\lambda}^{0}\|^2 -
\|\vectsymb{\lambda} - \vectsymb{\lambda}^{T}\|^2 \right)\nonumber\\
&\ge&
-\frac{\eta}{2}
\|\matr{A}\vectsymb{q} + \matr{B}\vectsymb{p}^{0} - \vectsymb{c}\|^2
-
\frac{1}{2\eta}
\|\vectsymb{\lambda}\|^2.
\end{eqnarray}
From the concavity of $h$,
we have that
$h(\bar{\vectsymb{w}}_T) 
\ge \frac{1}{T}\sum_{t=0}^{T-1} h(\vectsymb{w}^{t+1})$.
Note also that,
for every $\tilde{\vectsymb{w}}$, the function ${\vectsymb{w}} \mapsto \DP{({\vectsymb{w}} - \tilde{\vectsymb{w}})}{F({\vectsymb{w}})}$
is affine:
\begin{eqnarray}
\DP{({\vectsymb{w}} - \tilde{\vectsymb{w}})}{F({\vectsymb{w}})} &=&
\DP{({\vectsymb{q}} - \tilde{\vectsymb{q}})}{\matr{A}^{\top}\vectsymb{\lambda}}
+
\DP{({\vectsymb{p}} - \tilde{\vectsymb{p}})}{\matr{B}^{\top}\vectsymb{\lambda}}
-
\DP{({\vectsymb{\lambda}} - \tilde{\vectsymb{\lambda}})}{(\matr{A}\vectsymb{q} + \matr{B}\vectsymb{p} - \vectsymb{c})}\nonumber\\
&=&
-\DP{(\matr{A}\tilde{\vectsymb{q}} + \matr{B}\tilde{\vectsymb{p}} - \vectsymb{c})}{\vectsymb{\lambda}}
+\DP{\tilde{\vectsymb{\lambda}}}{(\matr{A}\vectsymb{q} + \matr{B}\vectsymb{p} - \vectsymb{c})}\nonumber\\
&=&
\DP{F(\tilde{\vectsymb{w}})}{\vectsymb{w}} - \DP{{\vectsymb{c}}}{\tilde{\vectsymb{\lambda}}}.
\end{eqnarray}
As a consequence,
\begin{eqnarray}
\frac{1}{T}\sum_{t=0}^{T-1} \left( h(\vectsymb{w}^{t+1}) +
\DP{(\vectsymb{w}^{t+1} - \vectsymb{w})}{F(\vectsymb{w}^{t+1})} \right) \le
h(\bar{\vectsymb{w}}_T) + \DP{(\bar{\vectsymb{w}}_T - \vectsymb{w})}{F(\bar{\vectsymb{w}}_T)},
\end{eqnarray}
and from \eqref{eq:bound_admm_002},
we have that
$h(\vectsymb{w}) - h(\bar{\vectsymb{w}}_T) + \DP{(\vectsymb{w} - \bar{\vectsymb{w}}_T)}{F(\bar{\vectsymb{w}}_T)} \le C/T$,
with $C$ as in Eq.~\ref{eq:convergence_admm_vi_constant}.
Note also that, since $\Lambda$ is convex, 
we must have $\bar{\vectsymb{\lambda}}_T \in \Lambda$.
\end{proof}

\bigskip

Next, we use the bound in Proposition~\ref{prop:convergence_admm_vi}
to derive a convergence rate for the dual problem.
\begin{proposition}[Dual convergence rate]\label{prop:convergence_admm_dual}
Assume the conditions stated in Proposition~\ref{prop:convergence_admm_vi},
with $\bar{\vectsymb{w}}_T$
defined analogously.
Let $g:\Lambda \rightarrow \set{R}$ denote the dual objective function:
\begin{equation}
g(\vectsymb{\lambda}) := \max_{\vectsymb{q} \in \sett{Q}, \vectsymb{p} \in \sett{P}} L(\vectsymb{q},\vectsymb{p},\vectsymb{\lambda}),
\end{equation}
and let $\vectsymb{\lambda}^* = \arg\min_{\vectsymb{\lambda} \in \Lambda} g(\vectsymb{\lambda})$
be a dual solution.
Then, after $T$ iterations, ADMM achieves an $O(\frac{1}{T})$-accurate solution $\bar{\vectsymb{\lambda}}_T$:
\begin{equation}
g({\vectsymb{\lambda}}^*)
\le
g(\bar{\vectsymb{\lambda}}_T)
\le
g({\vectsymb{\lambda}}^*) + \frac{C}{T},
\end{equation}
where the constant $C$ is given by
\begin{eqnarray}\label{eq:admm_constant_C_finalbound}
C &=& 
\frac{5\eta}{2} \left(
\max_{\vectsymb{q} \in \sett{Q}}\|\matr{A}{\vectsymb{q}} + \matr{B}\vectsymb{p}^{0} - \vectsymb{c}\|^2
\right)
+
\frac{5}{2\eta}
\|\vectsymb{\lambda}^*\|^2.
\end{eqnarray}
\end{proposition}
\begin{proof}
By applying Proposition~\ref{prop:convergence_admm_vi} to $\vectsymb{w} = (\bar{\vectsymb{q}}_T, \bar{\vectsymb{p}}_T, {\vectsymb{\lambda}})$ we obtain for
arbitrary $\vectsymb{\lambda} \in \Lambda$:
\begin{eqnarray}\label{eq:admm_convergence_aux00}
-\DP{({\vectsymb{\lambda}} - \bar{\vectsymb{\lambda}}_T)}{(\matr{A}\bar{\vectsymb{q}}_T + \matr{B}\bar{\vectsymb{p}}_T -
\vectsymb{c})} \le O(1/T).
\end{eqnarray}
By applying Proposition~\ref{prop:convergence_admm_vi} to $\vectsymb{w} = ({\vectsymb{q}}, {\vectsymb{p}}, \bar{\vectsymb{\lambda}}_T)$ we obtain for
arbitrary ${\vectsymb{q}}\in \sett{Q}$ and ${\vectsymb{p}}\in \sett{P}$:
\begin{eqnarray}\label{eq:admm_convergence_aux01}
\lefteqn{f_1(\bar{\vectsymb{q}}_T) + f_2(\bar{\vectsymb{p}}_T) + \DP{(\matr{A}\bar{\vectsymb{q}}_T + \matr{B}\bar{\vectsymb{p}}_T - \vectsymb{c})}{\bar{\vectsymb{\lambda}}_T}}\nonumber\\
&\ge&
f_1(\vectsymb{q}) + f_2(\vectsymb{p}) + \DP{(\matr{A}{\vectsymb{q}} + \matr{B}{\vectsymb{p}} - \vectsymb{c})}{\bar{\vectsymb{\lambda}}_T}
- O(1/T).
\end{eqnarray}
In particular, let $g(\bar{\vectsymb{\lambda}}_T) =
\max_{\vectsymb{q} \in \sett{Q}, \vectsymb{p} \in \sett{P}} L(\vectsymb{q},\vectsymb{p},\bar{\vectsymb{\lambda}}_T)
= L(\widehat{\vectsymb{q}}_T,\widehat{\vectsymb{p}}_T,\bar{\vectsymb{\lambda}}_T)$
be the value of the dual objective at $\bar{\vectsymb{\lambda}}_T$,
where
$(\widehat{\vectsymb{q}}_T, \widehat{\vectsymb{p}}_T)$ are the corresponding maximizers.
We then have:
\begin{eqnarray}\label{eq:admm_convergence_aux02}
f_1(\bar{\vectsymb{q}}_T) + f_2(\bar{\vectsymb{p}}_T) + \DP{(\matr{A}\bar{\vectsymb{q}}_T + \matr{B}\bar{\vectsymb{p}}_T - \vectsymb{c})}{\bar{\vectsymb{\lambda}}_T}
\ge
g(\bar{\vectsymb{\lambda}}_T) - O(1/T).
\end{eqnarray}
Finally we have (letting $\vectsymb{w}^* = (\vectsymb{q}^*,\vectsymb{p}^*,\vectsymb{\lambda}^*)$ be the optimal primal-dual solution):
\begin{eqnarray}
g(\vectsymb{\lambda}^*) &=& \max_{\vectsymb{q} \in \sett{Q}, \vectsymb{p} \in \sett{P}} f_1(\vectsymb{q}) + f_2(\vectsymb{p}) +
\DP{\vectsymb{\lambda}^*}{(\matr{A}\vectsymb{q} + \matr{B}\vectsymb{p} - \vectsymb{c})}\nonumber\\
&\ge& f_1(\bar{\vectsymb{q}}_T) + f_2(\bar{\vectsymb{p}}_T) +
\DP{\vectsymb{\lambda}^*}{(\matr{A}\bar{\vectsymb{q}}_T + \matr{B}\bar{\vectsymb{p}}_T - \vectsymb{c})}\nonumber\\
&\ge^{(i)}& f_1(\bar{\vectsymb{q}}_T) + f_2(\bar{\vectsymb{p}}_T) +
\DP{\bar{\vectsymb{\lambda}}_T}{(\matr{A}\bar{\vectsymb{q}}_T + \matr{B}\bar{\vectsymb{p}}_T - \vectsymb{c})}- O(1/T)\nonumber\\
&\ge^{(ii)}& g(\bar{\vectsymb{\lambda}}_T) - O(1/T),
\end{eqnarray}
where in $(i)$ we used Eq.~\ref{eq:admm_convergence_aux00} and in $(ii)$ we used Eq.~\ref{eq:admm_convergence_aux02}.
By definition of $\vectsymb{\lambda}^*$,
we also have $g(\bar{\vectsymb{\lambda}}_T) \ge g(\vectsymb{\lambda}^*)$.
Since we applied Proposition~\ref{prop:convergence_admm_vi} twice, the constant inside the $O$-notation becomes
\begin{equation}\label{eq:admm_constant_C_almostfinalbound}
C = \frac{\eta}{2} \left(
\|\matr{A}\bar{\vectsymb{q}}_T + \matr{B}\vectsymb{p}^{0} - \vectsymb{c}\|^2
+
\|\matr{A}\widehat{\vectsymb{q}}_T + \matr{B}\vectsymb{p}^{0} - \vectsymb{c}\|^2
\right)
+
\frac{1}{2\eta}
\left(
\|\vectsymb{\lambda}^*\|^2
+
\|\bar{\vectsymb{\lambda}}_T\|^2
\right).
\end{equation}
Even though $C$ depends on $\bar{\vectsymb{q}}_T$,
$\widehat{\vectsymb{q}}_T$, and $\bar{\vectsymb{\lambda}}_T$, it
is easy to obtain an upper bound on $C$ when $\sett{Q}$ is a bounded set,
using the fact that the sequence $(\vectsymb{\lambda}^t)_{t\in\set{N}}$ is bounded by a constant,
which implies that the average $\bar{\vectsymb{\lambda}}_T$ is also bounded.
Indeed, from \citet[p.107]{Boyd2011}, we have that
\begin{equation}
V^{t} := \eta^{-1}\|\vectsymb{\lambda}^* - \vectsymb{\lambda}^t\|^2 + \eta\|\matr{B}(\vectsymb{p}^* - \vectsymb{p}^t)\|^2
\end{equation}
is a Lyapunov function, \emph{i.e.}, $0 \le V^{t+1} \le V^{t}$ for every $t \in \set{N}$.
This implies that
$V^{t} \le V^{0} = \eta^{-1}\|\vectsymb{\lambda}^*\|^2
+ \eta\|\matr{B}(\vectsymb{p}^* - \vectsymb{p}^0)\|^2$;
since $V^{t} \ge \eta^{-1}\|\vectsymb{\lambda}^* - \vectsymb{\lambda}^t\|^2$, we can replace above
and write:
\begin{eqnarray}
0 &\ge&
\|\vectsymb{\lambda}^* - \vectsymb{\lambda}^t\|^2 - \|\vectsymb{\lambda}^*\|^2 - \eta^2\|\matr{B}(\vectsymb{p}^* - \vectsymb{p}^0)\|^2
=
\|\vectsymb{\lambda}^t\|^2 - 2\DP{\vectsymb{\lambda}^*}{\vectsymb{\lambda}^t} - \eta^2\|\matr{B}(\vectsymb{p}^* - \vectsymb{p}^0)\|^2\nonumber\\
&\ge&
\|\vectsymb{\lambda}^t\|^2 - 2\|\vectsymb{\lambda}^*\| \|\vectsymb{\lambda}^t\| - \eta^2\|\matr{B}(\vectsymb{p}^* - \vectsymb{p}^0)\|^2,
\end{eqnarray}
where in the last line we invoked the Cauchy-Schwarz inequality.
Solving the quadratic equation, we obtain
$\|\vectsymb{\lambda}^t\| \le \|\vectsymb{\lambda}^*\| +
\sqrt{\|\vectsymb{\lambda}^*\|^2
+ \eta^2 \|\matr{B}(\vectsymb{p}^0 - \vectsymb{p}^*)\|^2}$,
which in turn implies
\begin{eqnarray}\label{eq:admm_finalbound}
\|\vectsymb{\lambda}^t\|^2 &\le& 2\|\vectsymb{\lambda}^*\|^2
+ \eta^2 \|\matr{B}(\vectsymb{p}^0 - \vectsymb{p}^*)\|^2 + 2\|\vectsymb{\lambda}^*\|\sqrt{\|\vectsymb{\lambda}^*\|^2
+ \eta^2 \|\matr{B}(\vectsymb{p}^0 - \vectsymb{p}^*)\|^2} \nonumber\\
&\le& 2\|\vectsymb{\lambda}^*\|^2
+ \eta^2 \|\matr{B}(\vectsymb{p}^0 - \vectsymb{p}^*)\|^2 + 2(\|\vectsymb{\lambda}^*\|^2
+ \eta^2 \|\matr{B}(\vectsymb{p}^0 - \vectsymb{p}^*)\|^2) \nonumber\\
&=& 4\|\vectsymb{\lambda}^*\|^2
+ 3\eta^2 \|\matr{B}(\vectsymb{p}^0 - \vectsymb{p}^*)\|^2\nonumber\\
&=& 4\|\vectsymb{\lambda}^*\|^2
+ 3\eta^2 \|\matr{A}\vectsymb{q}^* + \matr{B}\vectsymb{p}^0 - \vectsymb{c}\|^2,
\end{eqnarray}
the last line following from $\matr{A}\vectsymb{q}^* + \matr{B}\vectsymb{p}^* = \vectsymb{c}$.
Replacing \eqref{eq:admm_finalbound} in \eqref{eq:admm_constant_C_almostfinalbound} yields Eq.~\ref{eq:admm_constant_C_finalbound}.
\end{proof}

\bigskip

Finally, we will see how the bounds above apply to the AD$^3$ algorithm, relating 
the constant in the bound with the structure of the graphical model.

\begin{proposition}[Dual convergence rate of AD$^3$]\label{prop:convergence_ad3_dual}
After  $T$ iterations of AD$^3$,
we achieve an $O(\frac{1}{T})$-accurate solution $\bar{\vectsymb{\lambda}}_T := \sum_{t=0}^{T-1} \vectsymb{\lambda}^{(t)}$:
\begin{equation}
g({\vectsymb{\lambda}}^*)
\le
g(\bar{\vectsymb{\lambda}}_T)
\le
g({\vectsymb{\lambda}}^*) + \frac{C}{T},
\end{equation}
where 
$C = \frac{5\eta}{2}
\sum_i |\sett{N}(i)| (1 - |\Y_i|^{-1}) +
\frac{5}{2\eta} \|\vectsymb{\lambda}^{*}\|^2$ is a constant independent of $T$.
\end{proposition}
\begin{proof}
With the uniform initialization of the $\vectsymb{p}$-variables in AD$^3$,
the first term in the second line of Eq.~\ref{eq:admm_constant_C_finalbound} is maximized
by a choice of $\vectsymb{q}_{\alpha}$-variables that puts all mass in a single configuration,
for each factor $\alpha \in \sett{F}$.
That is, we have for each $i \in \sett{N}(\alpha)$:
\begin{eqnarray}
\max_{\vectsymb{q}_{i\alpha}} \|\vectsymb{q}_{i\alpha} - |\Y_i|^{-1} \vect{1} \|^2
= \left( (1 - |\Y_i|^{-1})^2 + (|\Y_i|-1)|\Y_i|^{-2} \right)
= 1 - |\Y_i|^{-1}.
\end{eqnarray}
This leads to the desired bound.
\end{proof}



\section{Derivation of Solutions for AD$^3$ Subproblems}\label{chap:ad3qp}

\subsection{Binary Pairwise Factors}\label{sec:ddadmm_binpairfact_appendix}

In this section, we prove Proposition~\ref{prop:ad3_pairwisebinary_solution}.

Let us first assume that $c_{12}\ge 0$. 
In this case, the lower bound constraints $z_{12} \ge z_1+z_2-1$ and $z_{12} \ge 0$ in \eqref{eq:pairwise01_main}
are always inactive and the problem can be simplified to:
\begin{align}\label{eq:pairwise02_main}
\mathrm{minimize} \quad & \frac{1}{2}(z_1 - c_1)^2 + \frac{1}{2}(z_2 - c_2)^2 - c_{12} z_{12} \nonumber\\
\text{with respect to} \quad & z_1,z_2,z_{12} \nonumber\\
\text{subject to} \quad & z_{12} \le z_1, \quad z_{12} \le z_2, \quad z_1 \in [0,1], \quad z_2 \in [0,1].
\end{align}
If $c_{12}= 0$, the problem becomes separable, and a solution is
\begin{eqnarray}
z_1^* = [c_1]_{\mathbb{U}}, \quad
z_2^* = [c_2]_{\mathbb{U}}, \quad
z_{12}^* = \min\{z_1^*, z_2^*\},
\end{eqnarray}
which complies with Eq.~\ref{eq:ad3_slave_pairwise_01}.
We next analyze the case where $c_{12}>0$.
The Lagrangian of \eqref{eq:pairwise02_main} is:
\begin{eqnarray}
L(\vectsymb{z},\vectsymb{\mu},\vectsymb{\lambda},\vectsymb{\nu}) &=&
\frac{1}{2}(z_1 - c_1)^2 + \frac{1}{2}(z_2 - c_2)^2 - c_{12} z_{12} + \mu_1(z_{12}-z_1) + \mu_2(z_{12}-z_2)  \nonumber\\
&&
-\lambda_1 z_1-\lambda_2 z_2 + \nu_1 (z_1 - 1)  + \nu_2 (z_2 - 1).
\end{eqnarray}
At optimality, the following Karush-Kuhn-Tucker (KKT) conditions need to be satisfied:
\begin{eqnarray}
\nabla_{z_1} L(\vectsymb{z}^*,\vectsymb{\mu}^*,\vectsymb{\lambda}^*,\vectsymb{\nu}^*) = 0 &\Rightarrow& z_1^* = c_1 + \mu_1^* + \lambda_1^* - \nu_1^* \label{eq:kkt_grad01}\\
\nabla_{z_2} L(\vectsymb{z}^*,\vectsymb{\mu}^*,\vectsymb{\lambda}^*,\vectsymb{\nu}^*) = 0 &\Rightarrow& z_2^* = c_2 + \mu_2^* + \lambda_2^* - \nu_2^* \label{eq:kkt_grad02}\\
\nabla_{z_{12}} L(\vectsymb{z}^*,\vectsymb{\mu}^*,\vectsymb{\lambda}^*,\vectsymb{\nu}^*) = 0 &\Rightarrow& c_{12}  = \mu_1^* +\mu_2^* \label{eq:kkt_grad03}\\
\lambda_1^* z_1^*= 0 && \label{eq:complslack01}\\
\lambda_2^* z_2^*= 0 && \label{eq:complslack02}\\
\mu_1^* (z_{12}^* - z_1^*) = 0 && \label{eq:complslack03}\\
\mu_2^* (z_{12}^* - z_2^*) = 0 && \label{eq:complslack04}\\
\nu_1^* (z_1^* - 1) = 0 && \label{eq:complslack05}\\
\nu_2^* (z_2^* - 1) = 0 && \label{eq:complslack06}\\
\vectsymb{\mu}^*,\vectsymb{\lambda}^*,\vectsymb{\nu}^* \ge 0&&\label{eq:dualfeas01} \\
z_{12}^* \le z_1^*, \quad z_{12}^* \le z_2^*, \quad z_1^* \in [0,1], \quad z_2^* \in [0,1]\label{eq:primalfeas01}
\end{eqnarray}
We are going to consider three cases separately:
\begin{enumerate}
\item \framebox{$z_1^* > z_2^*$}

From the primal feasibility conditions \eqref{eq:primalfeas01},
this implies $z_1^* > 0$, $z_2^* < 1$, and $z_{12}^* < z_1^*$.
Complementary slackness (\ref{eq:complslack01},\ref{eq:complslack06},\ref{eq:complslack03})
implies in turn $\lambda_1^* = 0$, $\nu_2^* = 0$, and $\mu_1^* = 0$.
From \eqref{eq:kkt_grad03} we have $\mu_2^* = c_{12}$.
Since we are assuming $c_{12} > 0$, we then have $\mu_2^*>0$, and complementary slackness
\eqref{eq:complslack04} implies $z_{12}^* = z_2^*$. 
Plugging this into \eqref{eq:kkt_grad01}--\eqref{eq:kkt_grad02} we obtain
\begin{eqnarray}
z_1^* = c_1 - \nu_1^* \le c_1, \qquad
z_2^* = c_2 + \lambda_2^* + c_{12} \ge c_2 + c_{12}.
\end{eqnarray}
Now we have the following:
\begin{itemize}
\item Either $z_1^* = 1$ or $z_1^* < 1$. In the latter case, $\nu_1^* = 0$ by complementary slackness \eqref{eq:complslack05},
hence $z_1^* = c_1$. Since in any case we must have $z_1^* \le c_1$, we conclude that
$z_1^* = \min\{c_1,1\}$.
\item Either $z_2^* = 0$ or $z_2^* > 0$. In the latter case, $\lambda_2^* = 0$ by complementary slackness \eqref{eq:complslack02},
hence $z_2^* = c_2 + c_{12}$. Since in any case we must have $z_2^* \ge \lambda_2$, we conclude that
$z_2^* = \max\{0,c_2 + c_{12}\}$.
\end{itemize}
In sum:
\begin{equation}
z_1^* = \min\{c_1,1\}, \quad z_{12}^* = z_2^* = \max\{0,c_2 + c_{12}\},
\end{equation}
and our assumption $z_1^* > z_2^*$ can only be valid if $c_1 > c_2 + c_{12}$.

\item \framebox{$z_1^* < z_2^*$}

By symmetry, we have
\begin{equation}
z_2^* = \min\{c_2,1\}, \quad z_{12}^* = z_1^* = \max\{0,c_1 + c_{12}\},
\end{equation}
and our assumption $z_1^* < z_2^*$ can only be valid if $c_2 > c_1 + c_{12}$.

\item \framebox{$z_1^* = z_2^*$}

In this case, it is easy to verify that we must have $z_{12}^* = z_1^* = z_2^*$,
and we can rewrite our optimization problem in terms of one variable only (call it $z$).
The problem becomes that of minimizing $\frac{1}{2}(z-c_1)^2 + \frac{1}{2}(z-c_2)^2 -c_{12}z$,
which equals a constant plus $(z-\frac{c_1+c_2+c_{12}}{2})^2$,
subject to $z \in \mathbb{U} \triangleq [0,1]$.
Hence:
\begin{equation}
z_{12}^* = z_1^* = z_2^* = \left[\frac{c_1+c_2+c_{12}}{2}\right]_{\mathbb{U}}.
\end{equation}
\end{enumerate}

Putting all the pieces together, we obtain the solution displayed in Eq.~\ref{eq:ad3_slave_pairwise_01}.

It remains to address the case where $c_{12}<0$. 
By redefining $c_1'=c_1+c_{12}$, $c_2'=1-c_2$, $c_{12}'=-c_{12}$, $z'_2 = 1-z_2$,
and $z'_{12} = z_1-z_{12}$, 
we can reduce \eqref{eq:pairwise01_main} to the form in \eqref{eq:pairwise02_main}. 
Substituting back in Eq.~\ref{eq:ad3_slave_pairwise_01}, 
we obtain  the solution displayed in Eq.~\ref{eq:ad3_slave_pairwise_02}.



\subsection{Marginal Polytope of Hard Constraint Factors}
\label{sec:hard_constraint_factors_intro}

The following proposition 
establishes that the marginal polytope of a hard constraint factor 
is the convex hull of its acceptance set.
\begin{proposition}\label{prop:marginalpolytopehard}
Let $\alpha$ be a binary hard constraint factor with degree $K$,
and consider the set of all possible distributions $\Prob(\vectsymb{Y}_{\alpha})$
which satisfy
$\Prob(\vectsymb{Y}_{\alpha}=\vectsymb{y}_{\alpha}) = 0$
for every $\vectsymb{y}_{\alpha} \notin \sett{S}_{\alpha}$.
Then, the set of possible marginals realizable for some  distribution in that set is given by
\begin{eqnarray}\label{eq:marginalpolytope_hard}
\sett{Z}_{\alpha} &:=& \left\{(q_{1\alpha}(1),\ldots,q_{K\alpha}(1)) \,\,\bigg|\,\,
 \vectsymb{q}_{i\alpha}=\matr{M}_{i\alpha}\vectsymb{q}_{\alpha},\,\,
 \text{for some $\vectsymb{q}_{\alpha} \in \Delta^{|\Y_{\alpha}|}$ s.t.
 $q_{\alpha}(\vectsymb{y}_{\alpha}) = 0, \forall{\vectsymb{y}_{\alpha} \notin \sett{S}_{\alpha}}$}\right\}\nonumber\\
 &=& \conv \sett{S}_{\alpha}.
\end{eqnarray}
\end{proposition}
 \begin{proof}
 From the fact that we are constraining
 $q_{\alpha}(\vectsymb{y}_{\alpha}) = 0, \forall{\vectsymb{y}_{\alpha} \notin \sett{S}_{\alpha}}$, it follows:
 \begin{eqnarray}
 \sett{Z}_{\alpha} &=& \left\{\vectsymb{z} \ge 0 \,\,\Bigg|\,\, \exists \vectsymb{q}_{\alpha} \ge 0
 \,\,\text{s.t.}\,\, \forall i \in \sett{N}({\alpha}),
 z_i = \sum_{\substack{\vectsymb{y}_{\alpha} \in \sett{S}_{\alpha}\\\vectsymb{y}_i = 1}} q_{\alpha}(\vectsymb{y}_{\alpha}) =
 1- \sum_{\substack{\vectsymb{y}_{\alpha} \in \sett{S}_{\alpha}\\\vectsymb{y}_i = 0}} q_{\alpha}(\vectsymb{y}_{\alpha})
 \right\}\nonumber\\
 &=& \left\{\vectsymb{z} \ge 0 \,\,\Bigg|\,\, \exists \vectsymb{q}_{\alpha} \ge 0,
 \sum_{\vectsymb{y}_{\alpha} \in \sett{S}_{\alpha}} q_{\alpha}(\vectsymb{y}_{\alpha}) = 1
 \,\,\text{s.t.}\,\,
 \vectsymb{z} = \sum_{\vectsymb{y}_{\alpha} \in \sett{S}_{\alpha}} q_{\alpha}(\vectsymb{y}_{\alpha}) \vectsymb{y}_{\alpha} \right\} \nonumber\\
 &=& \conv \sett{S}_{\alpha}.
 \end{eqnarray}
 \end{proof}

For hard constraint factors,
the AD$^3$ subproblems take the following form (cf. Eq.~\ref{eq:ddadmm_quad}):
\begin{align}\label{eq:ddadmm_quad_hard}
\mathrm{minimize} \quad
&
\displaystyle
\frac{1}{2}
\sum_{i \in \sett{N}(\alpha)}
\|\vectsymb{q}_{i\alpha} - \vectsymb{a}_i\|^2
\nonumber\\
\text{with respect to}  \quad
& \vectsymb{q}_{\alpha} \in \Delta^{|\Y_{\alpha}|},\,\,  \vectsymb{q}_{i\alpha} \in \set{R}^{|\Y_i|},\,\,\forall i \in \sett{N}(\alpha)\nonumber\\
\text{subject to} \quad & \vectsymb{q}_{i\alpha} = \matr{M}_{i\alpha}\vectsymb{q}_{\alpha}\nonumber\\
&
q_{\alpha}(\vectsymb{y}_{\alpha}) = 0,\,\,
\forall \vectsymb{y}_{\alpha}\ne \sett{S}_{\alpha}.
\end{align}
From Proposition~\ref{prop:marginalpolytopehard},
and making use of a reduced parametrization,
noting that $\|\vectsymb{q}_{i\alpha} - \vectsymb{a}_i\|^2 =
(q_{i\alpha}(1) - a_i(1))^2 +
(1-q_{i\alpha}(1) - a_i(0))^2$,
which
equals a constant plus
$2(q_{i\alpha}(1) - (a_i(1) +1 - a_i(0))/2)^2$,
we have that this problem is equivalent to:
\begin{align}\label{eq:quadproj}
\mathrm{minimize} \quad
&
\frac{1}{2}\|\vectsymb{z} - \vectsymb{z}_0\|^2\nonumber\\
\text{with respect to}  \quad
&
\vectsymb{z} \in \sett{Z}_{\alpha},
\end{align}
where $z_{0i} := (a_i(1) +1 - a_i(0))/2$, for each $i \in \sett{N}(\alpha)$.

\subsection{XOR Factor}\label{sec:appendix_xor}

For the XOR factor,
the quadratic problem in Eq.~\ref{eq:ddadmm_quad} reduces to that of \emph{projecting onto the simplex}.
That problem is well-known in the optimization community (see, \emph{e.g.}, \citealt{Brucker1984,Michelot1986});
by writing the KKT conditions, it is simple to show that the solution $\vectsymb{z}^*$ is a soft-thresholding
of $\vectsymb{z}_0$, and therefore the problem can be reduced to that of finding the right threshold.
Algorithm~\ref{alg:projontosimplex} provides an efficient procedure;
it requires a sort operation, which renders its cost $O(K \log K)$. A proof of correctness appears in \citet{Duchi2008}.%
\footnote{A red-black tree can be used to reduce this cost to $O(K)$ \citep{Duchi2008}.
In later iterations of AD$^3$,
great speed-ups can be achieved in practice
since this procedure is repeatedly invoked with small changes to the coefficients.}

\begin{algorithm}[t]
   \caption{Projection onto simplex \citep{Duchi2008}}
\begin{algorithmic}\label{alg:projontosimplex}
   \STATE {\bfseries Input:} $\vectsymb{z}_0$
   \STATE Sort $\vectsymb{z}_0$ into $\vectsymb{y}_0$: $y_1 \ge \ldots \ge y_K$
   \STATE Find $\rho = \max \left\{j \in [K] \,\,|\,\, y_{0j} - \frac{1}{j} \left( (\sum_{r=1}^j y_{0r}) - 1\right) > 0\right\}$
   \STATE Define $\tau = \frac{1}{\rho} \left( \sum_{r=1}^{\rho} y_{0r} - 1 \right)$
   \STATE {\bfseries Output:} $\vectsymb{z}$ subject to $z_i = \max\{z_{0i} - \tau, 0\}$.
\end{algorithmic}
\end{algorithm}

\subsection{OR Factor}\label{sec:appendix_or}

The following procedure can be used for
computing a projection onto $\sett{Z}_{\mathrm{OR}}$:
\begin{enumerate}
\item Set $\tilde{\vectsymb{z}}$ as the projection of $\vectsymb{z}_0$ onto the unit cube.
This can be done by clipping each coordinate to the unit interval $\set{U} = [0,1]$,
i.e., by setting $\tilde{z}_i = [z_{0i}]_{\set{U}} = \min\{1, \max\{0, z_{0i}\}\}$.
If $\sum_{i=1}^K \tilde{z}_i \ge 1$, then return  $\tilde{\vectsymb{z}}$. Else go to step 2.
\item Return the projection of $\vectsymb{z}_0$ onto the simplex (use Algorithm~\ref{alg:projontosimplex}).
\end{enumerate}
The correctness of this procedure
is justified by the following lemma:

\begin{lemma}[Sifting Lemma.]\label{lemma:proj_twostage}
Consider a problem of the form
\begin{equation}\label{eq:opt1constr}
P: \qquad \min_{\vectsymb{x} \in \sett{X}} f(\vectsymb{x}) \quad \text{subject to}\quad g(\vectsymb{x}) \le 0,
\end{equation}
where $\sett{X}$ is nonempty convex subset of $\set{R}^D$ and $f:\sett{X} \rightarrow \set{R}$ and $g:\sett{X} \rightarrow \set{R}$
are convex functions.
Suppose that the problem \eqref{eq:opt1constr} is feasible and bounded below,
and let $\sett{A}$ be the set of solutions of the \emph{relaxed problem} $\min_{\vectsymb{x} \in \sett{X}} f(\vectsymb{x})$,
i.e. $\sett{A} = \{ \vectsymb{x} \in \sett{X} \,\,|\,\, f(\vectsymb{x}) \le f(\vectsymb{x}'),\,\, \forall \vectsymb{x}' \in \sett{X}\}$.
Then:
\begin{enumerate}
\item if for some $\tilde{\vectsymb{x}} \in \sett{A}$ we have $g(\tilde{\vectsymb{x}}) \le 0$,
then $\tilde{\vectsymb{x}}$ is also a solution of the original problem $P$;
\item otherwise (if for all $\tilde{\vectsymb{x}} \in \sett{A}$ we have $g(\tilde{\vectsymb{x}}) > 0$),
then the inequality constraint is necessarily active in $P$, i.e.,
problem $P$ is equivalent to $\min_{\vectsymb{x} \in \sett{X}} f(\vectsymb{x})$ subject to $g(\vectsymb{x}) = 0$.
\end{enumerate}
\end{lemma}
\begin{proof}
Let $f^*$ be the optimal value of $P$.
The first statement is obvious: since $\tilde{\vectsymb{x}}$ is a solution of a relaxed problem we have
$f(\tilde{\vectsymb{x}}) \le f^*$; hence if $\tilde{\vectsymb{x}}$ is feasible this becomes an equality.
For the second statement,
assume that $\exists \vectsymb{x} \in \sett{X}$ subject to $g(\vectsymb{x}) < 0$ (otherwise, the statement holds trivially).
The nonlinear Farkas' lemma \citep[Prop.~3.5.4, p.~204]{Bertsekas2003} implies that there exists some ${\lambda}^* \ge 0$
subject to $f(\vectsymb{x}) - f^* + {\lambda}^* g(\vectsymb{x}) \ge 0$ holds for all $\vectsymb{x} \in \sett{X}$.
In particular, this also holds for an optimal $\vectsymb{x}^*$ (i.e., such that $f^* = f(\vectsymb{x}^*)$), which implies that
${\lambda}^* g(\vectsymb{x}^*) \ge 0$. However, since ${\lambda}^* \ge 0$ and $g(\vectsymb{x}^*) \le 0$ (since
$\vectsymb{x}^*$ has to be feasible), we also have ${\lambda}^* g(\vectsymb{x}^*) \le 0$, \emph{i.e.},
${\lambda}^* g(\vectsymb{x}^*) = 0$. Now suppose that $\lambda^* = 0$. Then we have
$f(\vectsymb{x}) - f^* \ge 0$, $\forall \vectsymb{x} \in \sett{X}$, which implies that
$\vectsymb{x}^* \in \sett{A}$
and contradicts the assumption that $g(\tilde{\vectsymb{x}}) > 0, \forall \tilde{\vectsymb{x}} \in \sett{A}$.
Hence we must have  $g(\vectsymb{x}^*) = 0$.
\end{proof}

\bigskip

Let us see how the Sifting Lemma applies to the problem of projecting onto $\sett{Z}_{\mathrm{OR}}$. 
If the relaxed problem in the first step does not return a feasible point
then, from the Sifting Lemma, the constraint $\sum_{i=1}^K {z}_i \ge 1$ has to be active, \emph{i.e.},
we must have $\sum_{i=1}^K {z}_i= 1$.
This, in turn, implies that $\vectsymb{z} \le 1$, hence the problem becomes equivalent to the XOR case.
In sum, the worst-case runtime is $O(K\log K)$, although it is $O(K)$ if the first step succeeds.

\subsection{OR-with-output Factor}\label{sec:appendix_orout}

Solving the AD$^3$ subproblem for the OR-with-output factor is slightly more complicated than in the previous  cases; however,
we next see that it can also be addressed in $O(K \log K)$ time with a sort operation.

The polytope $\sett{Z}_{\mathrm{OR-out}}$ can be expressed as the intersection of the following three sets:%
\footnote{Actually, the set $\set{U}^{K+1}$ is redundant, since we have $\sett{A}_2 \subseteq \set{U}^{K+1}$ and
therefore $\sett{Z}_{\mathrm{OR-out}} = \sett{A}_1 \cap \sett{A}_2$. However it is
computationally advantageous to consider this redundancy, as we shall see.} %
\begin{eqnarray}
\set{U}^{K+1} &:=& [0,1]^{K+1}\\
\sett{A}_1 &:=& \{\vectsymb{z} \in \set{R}^{K+1}\,\,|\,\, z_k \le
z_{K+1}, \forall k=1,\ldots,K\}\\
\sett{A}_2 &:=& \left\{\vectsymb{z} \in [0,1]^{K+1}\,\,\bigg|\,\, \sum_{k=1}^K z_k \ge z_{K+1}\right\}.
\end{eqnarray}
We further define $\sett{A}_0 := [0,1]^{K+1} \cap \sett{A}_1$.
From the Sifting Lemma (Lemma~\ref{lemma:proj_twostage}), we have that the following procedure is correct:
\begin{enumerate}
\item Set $\tilde{\vectsymb{z}}$ as the projection of $\vectsymb{z}_0$ onto the unit cube.
If $\tilde{\vectsymb{z}} \in \sett{A}_1\cap \sett{A}_2$, then we are done: just return  $\tilde{\vectsymb{z}}$.
Else, if $\tilde{\vectsymb{z}} \in \sett{A}_1$ but $\tilde{\vectsymb{z}} \notin \sett{A}_2$,  go to step 3.
Otherwise, go to step 2.
\item Set $\tilde{\vectsymb{z}}$ as the projection of $\vectsymb{z}_0$ onto $\sett{A}_0$ (we will describe how to do this later).
If $\tilde{\vectsymb{z}} \in \sett{A}_2$, return $\tilde{\vectsymb{z}}$. Otherwise, go to step 3.
\item Return the projection of $\vectsymb{z}_0$ onto the set
$\{\vectsymb{z} \in [0,1]^{K+1}\,\,|\,\, \sum_{k=1}^K z_k = z_{K+1}\}$.
This set is precisely the marginal polytope of a 
XOR factor with the last output negated, 
hence the projection corresponds to the local subproblem for that factor, 
for which we can employ the procedure already described for the XOR factor 
(using Algorithm~\ref{alg:projontosimplex}).
\end{enumerate}
Note that the first step above can be omitted; however, it avoids performing step 2 (which requires a
sort) unless it is really necessary.
To completely specify the algorithm, we only need to explain how to compute the projection onto $\sett{A}_0$ (step 2):
\begin{procedure}\label{proc:projontocubedcone}
To project onto $\sett{A}_0 = [0,1]^{K+1} \cap \sett{A}_1$:
\begin{enumerate}
\item[2a.] Set $\tilde{\tilde{\vectsymb{z}}}$ as the projection of $\vectsymb{z}_0$ onto $\sett{A}_1$.
Algorithm~\ref{alg:projontocone} shows how to do this.
\item[2b.] Set $\tilde{\vectsymb{z}}$ as the projection of $\tilde{\tilde{\vectsymb{z}}}$ onto the unit cube (with the usual
clipping procedure).
\end{enumerate}
\end{procedure}
The proof that the composition of these two projections yields the desired projection onto $\sett{A}_0$
is a bit involved (we present it as Proposition~\ref{prop:projontocubedcone} below).%
\footnote{Note that in general,
the composition of individual projections is not equivalent to projecting onto the intersection.
In particular, commuting steps 2a and 2b would make our procedure incorrect.} %

We turn our attention 
to the problem of projecting onto $\sett{A}_1$ (step 2a). 
This can be written as the following problem:
\begin{equation}
\min_{\vectsymb{z}} \frac{1}{2}\|\vectsymb{z} - \vectsymb{z}_0\|^2 \quad \text{subject to} \quad z_k \le z_{K+1}, \,\, \forall k = 1,\ldots,K.
\end{equation}
It can be successively rewritten as:
\begin{eqnarray}
&&\min_{z_{K+1}} \frac{1}{2} (z_{K+1} - z_{0,K+1})^2 +
\sum_{k=1}^{K} \min_{z_k \le z_{K+1}} \frac{1}{2}(z_k - z_{0k})^2 \nonumber\\
&=& \min_{z_{K+1}} \frac{1}{2} (z_{K+1} - z_{0,K+1})^2 +  \sum_{k=1}^{K}\frac{1}{2}(\min\{z_{K+1}, z_{0k}\} -z_{0k})^2  \nonumber\\
&=& \min_{z_{K+1}} \frac{1}{2} (z_{K+1} - z_{0,K+1})^2 +  \frac{1}{2}\sum_{k\in \sett{I}(z_{K+1})}(z_{K+1}-z_{0k})^2. 
\end{eqnarray}
where $\sett{I}(z_{K+1}) \triangleq \{k\in[K] : z_{0k}\ge z_{K+1}\}$.
Assuming that the set $\sett{I}(z_{K+1})$ is given, the previous is a sum-of-squares problem whose solution is
\begin{equation}
z_{K+1}^* = \frac{z_{0,K+1} + \sum_{k \in \sett{I}(z_{K+1})} z_{0k}}{1 + |\sett{I}(z_{K+1})|}.
\end{equation}
The set $\sett{I}(z_{K+1})$ can be determined by inspection after sorting $z_{01},\ldots,z_{0K}$.
The procedure is shown in Algorithm~\ref{alg:projontocone}.

\begin{algorithm}[t]
   \caption{Projection onto $\sett{A}_1$}
\begin{algorithmic}\label{alg:projontocone}
   \STATE {\bfseries Input:} $\vectsymb{z}_0$
   \STATE Sort ${z}_{01},\ldots,{z}_{0K}$ into $y_1 \ge \ldots \ge y_{K}$
   \STATE Find $\rho = \min \left\{j \in [K+1] \,\,|\,\, \frac{1}{j} \left( z_{0,K+1} + \sum_{r=1}^{j-1} y_{r}\right) > y_{j}\right\}$
   \STATE Define $\tau = \frac{1}{\rho} \left( z_{0,K+1} + \sum_{r=1}^{\rho-1} y_{r} \right)$
   \STATE {\bfseries Output:} $\vectsymb{z}$ subject to $z_{K+1} = \tau$ and $z_i = \min\{z_{0i}, \tau\}$, $i=1,\ldots,K$.
\end{algorithmic}
\end{algorithm}



\begin{algorithm}[t]
   \caption{Dykstra's algorithm for projecting onto $\bigcap_{j=1}^J \sett{C}_j$}
\begin{algorithmic}\label{alg:dijkstraproj}
   \STATE {\bfseries Input:} Point $\vectsymb{x}_0 \in \set{R}^D$, convex sets $\sett{C}_1,\ldots, \sett{C}_J$
   \STATE Initialize $\vectsymb{x}^{(0)} = \vectsymb{x}_0$, $\vectsymb{u}_j^{(0)} = \vect{0}$ for all $j=1,\ldots,J$
   \STATE $t \leftarrow 1$
   \REPEAT
   \FOR{$j=1$ {\bfseries to} $J$}
   \STATE Set $s = j + (t-1)J$
   \STATE Set $\tilde{\vectsymb{x}}_0 = \vectsymb{x}^{(s - 1)} - \vectsymb{u}_j^{(t-1)}$
   \STATE Set $\vectsymb{x}^{(s)} = \mathrm{proj}_{\sett{C}_j} (\tilde{\vectsymb{x}}_0)$, and
   $\vectsymb{u}_j^{(t)} = \vectsymb{x}^{(s)} - \tilde{\vectsymb{x}}_0$
   \ENDFOR
   \STATE $t \leftarrow t+1$
   \UNTIL {convergence.}
   \STATE {\bfseries Output:} $\vectsymb{x}$
\end{algorithmic}
\end{algorithm}

\begin{proposition}\label{prop:projontocubedcone}
Procedure~\ref{proc:projontocubedcone} is correct.
\end{proposition}
\begin{proof}
The proof is divided into the following parts:
\begin{enumerate}
\item We show that
Procedure~\ref{proc:projontocubedcone} corresponds to the
first iteration of Dykstra's projection algorithm \citep{Boyle1986} applied to
sets $\sett{A}_1$ and $[0,1]^{K+1}$;
\item We show that Dykstra's converges in one iteration if a specific condition is met;
\item We show that with the two sets above that condition is met.
\end{enumerate}
The first part is trivial. Dykstra's algorithm is shown as Algorithm~\ref{alg:dijkstraproj};
when $J=2$, $\sett{C}_1 = \sett{A}_1$ and $\sett{C}_2 = [0,1]^{K+1}$,
and noting that $\vectsymb{u}_1^{(1)} = \vectsymb{u}_2^{(1)} = 0$,
its first iteration is precisely Procedure~\ref{proc:projontocubedcone}.

We turn to the second part, to show that, when $J=2$, the fact that $\vectsymb{x}^{(3)} = \vectsymb{x}^{(2)}$
implies that $\vectsymb{x}^{(s)} = \vectsymb{x}^{(2)}$, $\forall s > 3$. In words, if at the second iteration $t$
of Dykstra's,
the value of $\vectsymb{x}$ does not change after computing the first projection, then it will never change,
so the algorithm has converged and $\vectsymb{x}$ is the desired projection.
To see that, consider the moment in Algorithm~\ref{alg:dijkstraproj} when $t=2$ and $j=1$.
After the projection, we update $\vectsymb{u}_1^{(2)} = \vectsymb{x}^{(3)} - (\vectsymb{x}^{(2)} - \vectsymb{u}_1^{(1)})$,
which when $\vectsymb{x}^{(3)} = \vectsymb{x}^{(2)}$ equals $\vectsymb{u}_1^{(1)}$, \emph{i.e.},
$\vectsymb{u}_1$ keeps unchanged. Then, when $t=2$ and $j=2$, one first computes
$\tilde{\vectsymb{x}}_0 = \vectsymb{x}^{(3)} - \vectsymb{u}_2^{(1)} = \vectsymb{x}^{(3)} - (\vectsymb{x}^{(2)} - \vectsymb{x}_0) =
\vectsymb{x}_0$, \emph{i.e.}, the projection is the same as the one already computed at  $t=1$, $j=2$.
Hence the result is the same, \emph{i.e.}, $\vectsymb{x}^{(4)} = \vectsymb{x}^{(2)}$, and similarly
$\vectsymb{u}_2^{(2)} = \vectsymb{u}_2^{(1)}$.
Since neither $\vectsymb{x}$, $\vectsymb{u}_1$ and $\vectsymb{u}_2$ changed in the second iteration,
and subsequent iterations only depend on these values, we have that $\vectsymb{x}$ will never change afterwards.

Finally, we are going to see that, regardless of the choice of $\vectsymb{z}_0$ in Procedure~\ref{proc:projontocubedcone}
($\vectsymb{x}_0$ in Algorithm~\ref{alg:dijkstraproj}) we always have $\vectsymb{x}^{(3)} = \vectsymb{x}^{(2)}$.
Looking at Algorithm~\ref{alg:projontocone}, we see that after $t=1$:
\begin{align}
&x^{(1)}_k = \left\{
\begin{array}{ll}
\tau, & \text{if $k = K+1$ or $x_{0k} \ge \tau$}\\
x_{0k}, & \text{otherwise,}
\end{array}
\right. \qquad
u_{1k}^{(1)} = \left\{
\begin{array}{ll}
\tau - x_{0k}, & \text{if $k = K+1$ or $x_{0k} \ge \tau$}\\
0, & \text{otherwise,}
\end{array}
\right.  \nonumber\\&
x^{(2)}_k = [x^{(1)}_k]_{\set{U}} = \left\{
\begin{array}{ll}
[\tau]_{\set{U}}, & \text{if $k = K+1$ or $x_{0k} \ge \tau$}\\
\left[x_{0k}\right]_{\set{U}}, & \text{otherwise.}
\end{array}
\right.
\end{align}
Hence in the beginning of the second iteration ($t=2$, $j=1$), we have
\begin{eqnarray}
\tilde{x}_{0k} = x^{(2)}_k - u_{1k}^{(1)} = \left\{
\begin{array}{ll}
[\tau]_{\set{U}} - \tau + x_{0k}, & \text{if $k = K+1$ or $x_{0k} \ge \tau$}\\
\left[x_{0k}\right]_{\set{U}}, & \text{otherwise.}
\end{array}
\right.
\end{eqnarray}
Now two things should be noted about Algorithm~\ref{alg:projontocone}:
\begin{itemize}
\item If a constant is added to all entries in $\vectsymb{z}_0$, the set $\sett{I}(z_{K+1})$ remains the same, and
$\tau$ and $\vectsymb{z}$ are affected by the same constant;
\item Let $\vectsymb{z}_0'$ be such that $z_{0k}' = z_{0k}$ if $k = K+1$ or $z_{0k} \ge \tau$,
and $z_{0k}' \le \tau$ otherwise. Let $\vectsymb{z}'$ be the projected point when such
$\vectsymb{z}_0'$ is given as input.
Then $\sett{I}(z_{K+1}') = \sett{I}(z_{K+1})$, $\tau'=\tau$,
$z_{k}' = z_{k}$ if $k = K+1$ or $z_{0k} \ge \tau$, and $z_k' = z_{0k}'$ otherwise.
\end{itemize}
The two facts above allow to relate the projection of $\tilde{\vectsymb{x}}_{0}$
(in the second iteration)
with that of $\vectsymb{x}_{0}$
(in the first iteration).
Using $[\tau]_{\set{U}} - \tau$ as the constant, and
noting that, for $k \ne K+1$ and $x_{0k} < \tau$,
we have $\left[x_{0k}\right]_{\set{U}} - [\tau]_{\set{U}} + \tau \ge \tau$ if  $x_{0k} < \tau$,
the two facts imply that:
\begin{eqnarray}
x^{(3)}_k = \left\{
\begin{array}{ll}
x^{(1)}_k + [\tau]_{\set{U}} - \tau = [\tau]_{\set{U}}, & \text{if $k = K+1$ or $x_{0k} \ge \tau$}\\
\left[x_{0k}\right]_{\set{U}}, & \text{otherwise;}
\end{array}
\right.
\end{eqnarray}
hence $\vectsymb{x}^{(3)} = \vectsymb{x}^{(2)}$, which concludes the proof.
\end{proof}

\section{Proof of Proposition~\ref{prop:ad3_subproblem_primal_sparse}}\label{sec:proof_ad3_subproblem_primal_sparse}

We first show that the rank of the matrix $\matr{M}$  is at most $\sum_{i\in \sett{N}(\alpha)} |\sett{Y}_i| - \sett{N}(\alpha) + 1$.
For each $i \in \sett{N}(\alpha)$,
let us consider the $|\sett{Y}_i|$ rows of $\matr{M}$.
By definition of $\matr{M}$, the set of entries on these rows which have the value $1$
form a partition of $\sett{Y}_{\alpha}$, hence,
summing these rows yields the all-ones row vector,
and this happens for each $i \in \sett{N}(\alpha)$.
Hence we have at least $\sett{N}(\alpha)-1$ rows that are linearly dependent.
This shows that the rank of $\matr{M}$ is at most $\sum_{i\in \sett{N}(\alpha)} |\sett{Y}_i| - \sett{N}(\alpha) + 1$.
Let us now rewrite \eqref{eq:ad3_subproblem_primal} as
\begin{align}\label{eq:primal_sep1}
\text{minimize} \quad & \frac{1}{2}\|\vectsymb{u} - \vectsymb{a}\|^2 + g(\vectsymb{u}) \nonumber\\
\text{with respect to} \quad & \vectsymb{u} \in \set{R}^{\sum_i |\sett{Y}_i|},
\end{align}
where $g(\vectsymb{u})$ is the solution value of the following linear problem:
\begin{align}\label{eq:primal_sep2}
\text{minimize} \quad & - \vectsymb{b}^{\top} \vectsymb{v}\\
\text{with respect to} \quad &
\vectsymb{v} \in \set{R}^{|\sett{Y}_{\alpha}|}\nonumber\\
\text{subject to} \quad &
\left\{
\begin{array}{ll}
\matr{M} \vectsymb{v} = \vectsymb{u} \nonumber\\
\vect{1}^{\top}\vectsymb{v} = 1\nonumber\\
\vectsymb{v} \ge 0.
\end{array}
\right.
\end{align}
From the simplex constraints (last two lines), we have that problem \eqref{eq:primal_sep2} is bounded below
(i.e., $g(\vectsymb{u}) > -\infty$). Furthermore,
problem \eqref{eq:primal_sep2} is feasible (i.e., $g(\vectsymb{u}) < +\infty$) if and only if $\vectsymb{u}$
satisfies the non-negativity and normalization constraints for every $i \in \sett{N}(\alpha)$:
\begin{equation}\label{eq:normnonneg}
\sum_{y_i} u_i(y_i) = 1, \quad u_i(y_i) \ge 0, \,\, \forall y_i.
\end{equation}
Those constraints imply $\vect{1}^{\top}\vectsymb{v} = 1$. Hence we can
add the constraints \eqref{eq:normnonneg} to the problem in \eqref{eq:primal_sep1},
discard the constraint $\vect{1}^{\top}\vectsymb{v} = 1$ in \eqref{eq:primal_sep2},
and assume that the resulting problem (which we reproduce below) is feasible and bounded below:
\begin{align}\label{eq:primal_sep3}
\text{minimize} \quad & - \vectsymb{b}^{\top} \vectsymb{v}\\
\text{with respect to} \quad &
\vectsymb{v} \in \set{R}^{|\sett{Y}_{\alpha}|}\nonumber\\
\text{subject to} \quad &
\left\{
\begin{array}{ll}
\matr{M} \vectsymb{v} = \vectsymb{u} \nonumber\\
\vectsymb{v} \ge 0.
\end{array}
\right.
\end{align}
Problem \eqref{eq:primal_sep3} is a linear program in standard form.
Since it is feasible and bounded, it admits a solution at a vertex of the constraint set
\citep{Rockafellar1970}.
We have that a point $\widehat{\vectsymb{v}}$ is a vertex
if and only if the columns of $\matr{M}$ indexed by
$\{\vectsymb{y}_{\alpha} \,\,|\,\, v_{\alpha}(\vectsymb{y}_{\alpha}) \ne 0\}$
are linearly independent.
We cannot have more than $\sum_{i\in \sett{N}(\alpha)} |\sett{Y}_i| - \sett{N}(\alpha)+1$ of these columns,
since this is the rank of $\matr{M}$.
It follows that \eqref{eq:primal_sep3} (and hence \eqref{eq:ad3_subproblem_primal}) has a solution
$\vectsymb{v}^*$ with at most $\sum_{i\in \sett{N}(\alpha)} |\sett{Y}_i| - \sett{N}(\alpha) + 1$ nonzeros.


\vskip 0.2in
\setlength{\bibsep}{7pt}
\bibliography{jmlr2012}

\end{document}